\documentclass[twoside,11pt]{article}

\usepackage{jmlr2e}
\usepackage[utf8]{inputenc} 
\usepackage[T1]{fontenc}    
\usepackage{hyperref}       
\usepackage{url}            
\usepackage{booktabs}       
\usepackage{amsfonts}       
\usepackage{nicefrac}       
\usepackage{microtype}      
\usepackage{amsmath}
\usepackage{float}
\usepackage{graphicx}
\usepackage{subcaption}
\usepackage{mwe}
\usepackage{xcolor}
\usepackage{diagbox}
\usepackage{amsfonts}
\usepackage{amssymb}
\usepackage{times}
\usepackage{multirow}
\usepackage{multicol}
\usepackage{fancyhdr,graphicx,amsmath,amssymb}
\usepackage[ruled,vlined]{algorithm2e}
\include{pythonlisting}
\SetKwInput{KwResult}{Input}
\SetKwInput{KwResults}{Output}
\usepackage[toc,page]{appendix}
\usepackage{comment}
\usepackage{ulem}
\usepackage{cancel}
\usepackage{bbm}
\usepackage{graphicx}

\newcommand{\thickhline}{\noalign{\hrule height 1.0pt}}

\newcommand{\mat}[1]{\mathbf{#1}}

\newcommand{\zz}[1]{\textcolor{red}{#1}}

\newcommand{\zm}[1]{\textcolor{blue}{#1}}

\ShortHeadings{Quantum-Inspired Hamiltonian Monte Carlo}{LIU AND ZHANG}
\firstpageno{1}

\begin{document}

\setcounter{tocdepth}{4} \setcounter{secnumdepth}{4}

\title{Quantum-Inspired Hamiltonian Monte Carlo for Bayesian Sampling}

\author{\name Ziming Liu\thanks{This work was done when the first author was visiting University of California, Santa Barbara.} \email zmliu@mit.edu \\
       \addr Department of Physics\\
       Massachusetts Institute of Technology\\
       Cambridge, MA, USA
       \AND
       \name Zheng Zhang \email zhengzhang@ece.ucsb.edu \\
       \addr Department of Electrical \& Computer Engineering\\
       University of California, Santa Barbara\\
       Santa Barbara, CA, USA}

\editor{}

\maketitle

\begin{abstract}
Hamiltonian Monte Carlo (HMC) is an efficient Bayesian sampling method that can make distant proposals in the parameter space by simulating a Hamiltonian dynamical system. Despite its popularity in machine learning and data science, HMC is inefficient to sample from spiky and multimodal distributions. Motivated by the energy-time uncertainty relation from quantum mechanics, we propose a Quantum-Inspired Hamiltonian Monte Carlo algorithm (QHMC). This algorithm allows a particle to have a random mass matrix with a probability distribution rather than a fixed mass. We prove the convergence property of QHMC and further show why such a random mass can improve the performance when we sample a broad class of distributions. In order to handle the big training data sets in large-scale machine learning, we develop a stochastic gradient version of QHMC using Nos{\'e}-Hoover thermostat called QSGNHT, and we also provide theoretical justifications about its steady-state distributions. Finally in the experiments, we demonstrate the effectiveness of QHMC and QSGNHT on synthetic examples, bridge regression, image denoising and neural network pruning. The proposed QHMC and QSGNHT can indeed achieve much more stable and accurate sampling results on the test cases.
\end{abstract}

\begin{keywords}
  Hamiltonian Monte Carlo, Quantum-Inspired Methods, Sparse Modeling
\end{keywords}

\section{Introduction}

Hamiltonian Monte Carlo (HMC)~\citep{Duane:1987de,neal2011mcmc} improves the traditional Markov Chain Monte Carlo (MCMC) method ~\citep{hastings1970monte} by introducing an extra momentum vector $\mat{q}$ conjugate to a state vector $\mat{x}$. In MCMC, a particle is usually only allowed to randomly walk in the state space. However, in HMC, a particle can flow quickly along the velocity direction and make distant proposals in the state space based on Hamiltonian dynamics. The energy-conservation property in continuous Hamiltonian dynamics can largely increase the acceptance rate, resulting in much faster mixing rate than standard MCMC. In numerical simulations where continuous time is quantized by discrete steps, the Metropolis-Hastings (MH) correction step can guarantee the correctness of the invariant distribution. HMC plays an increasingly important role in machine learning~\citep{NIPS2016_6117,myshkov2016posterior}.

HMC requires computation of the gradients of the energy function, which is the negative logarithm of the posterior probability in a Bayesian model. Therefore, HMC is born for smooth and differentiable functions but not immediate for non-smooth functions. In the case where functions have discontinuities, extra reflection and refraction steps~\citep{NIPS2015_5801,nishimura2017discontinuous} need to be involved in order to enhance the sampling efficiency. In the case characterized by $\ell_p$ prior where $1\leq p<2$, proximity operator methods~\citep{chaari2016hamiltonian,chaari17nshmc} can be used to increase the accuracy for simulating the Hamiltonian dynamics. 

This paper aims to develop a new type of HMC method in order to sample efficiently from a possibly spiky or multimodal posterior distribution. A representative example is Bayesian model of sparse or low-rank modeling using an $\ell_p$ ($0<p\leq1$) norm as the penalty. It is shown that sparser models can be obtained because $\ell_p$ $(0<p<1)$ prior~\citep{zhao2014p,xu20101} is a better, although non-convex, approximation for $\ell_0$ norm~\citep{Polson2017BayesianL0} than $\ell_1$ norm that is widely used in compressed sensing~\citep{donoho2006compressed,eldar2012compressed} and model reductions~\citep{hou2018structural}.

Leveraging the energy-time uncertainty relation in quantum mechanics, this paper proposes a quantum-inspired HMC (QHMC) method. In quantum mechanics, a particle can have a mass distribution other than a fixed mass value. Although being a pedagogical argument from physics, we will show from both theories and experiments that this principle can actually help achieve better sampling efficiency for non-smooth functions than standard HMC with fixed mass. The main idea of QHMC is visualized in Fig.~\ref{fig:qhmc_illus} with a one-dimensional harmonic oscillator. Assume that we have a ball with mass $m$ attached to a spring at the origin. The restoring force exerted on the ball pulls the ball back to the origin and the magnitude of the force is proportional to the displacement $x$, i.e., $F=-kx$. The ball oscillates around the origin with the time period $T=2\pi\sqrt{\frac{m}{k}}$. In standard HMC, $m$ is fixed and usually set as 1. In contrast, QHMC allows $m$ to be time-varying, meaning that the particle is sometimes moving fast and sometimes moving slowly. This is equivalent to employing a varying time scale. Consequently, different distribution landscapes can be explored more effectively with different time scales. In a flat but broad region, QHMC can quickly scan the whole region with a small time period $T$ (or small $m$). In a spiky region, we need to carefully explore every corner of the landscape, therefore a large $T$ (or large $m$) is preferred. In standard HMC, it is hard to choose a fixed time scale or mass to work well for both cases. This physical intuition is similar to the key idea of randomized HMC~\citep{Bou_Rabee_2017}, but this work consider more general cases where the mass can be a positive definite matrix $\mat{M}$.

\begin{figure}[t]
\centering
    \includegraphics[width=0.8\linewidth]{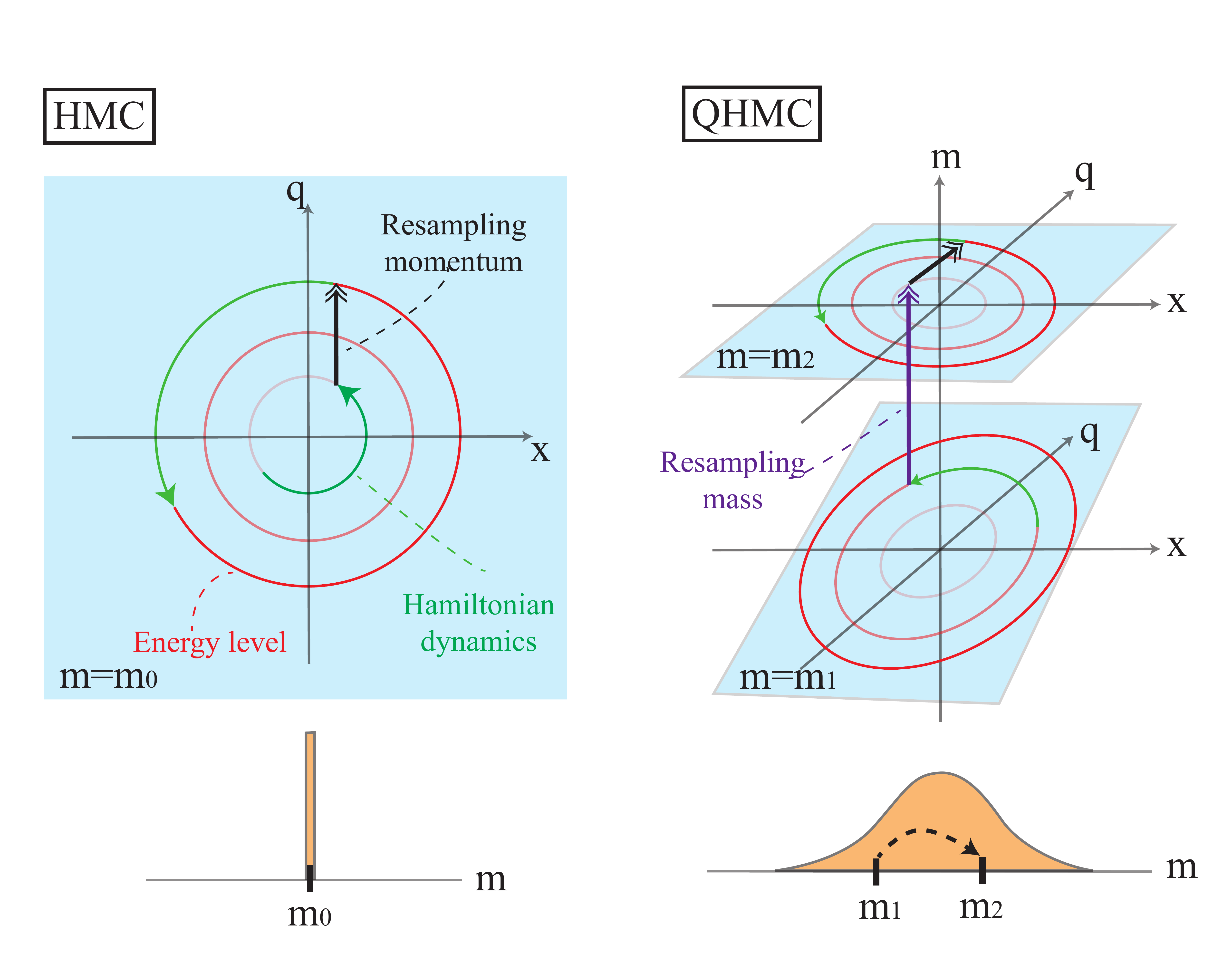}
    \caption{Illustration of Quantum-Inspired Hamiltonian Monte Carlo (QHMC) with the simple case of a one-dimensional harmonic oscillator. For HMC, the trajectories (energy levels) are confined in $(x,q)$ space with fixed mass. For QHMC, the mass is the third dimension that needs sampling. Here we use a scalar mass for illustration, however in general our framework allows the mass to be a positive definite matrix.}
    \label{fig:qhmc_illus}
    \vspace{-10pt}
\end{figure}

This paper is organized as follows. In Section
\ref{sec:HMC}, we review the standard HMC and summarize several HMC variants that are related to the physics literature. In Section \ref{sec:QHMC}, we propose the novel QHMC algorithm that treats mass as a random (positive definite) matrix rather than a fixed real positive scalar. We investigate the convergence properties of QHMC and demonstrate the advantages of QHMC with some toy examples on which HMC fails to work. Section~\ref{sec:mass_adapt} explains why treating the mass as a random variable improves the sampling performance for a large class of distributions. Section \ref{sec:SQHMC} proposes quantum stochastic gradient Nos{\'e}-Hoover thermostat (QSGNHT) to implement QHMC with massive training data, and proves its convergence based on stochastic differential equations (SDE). In Section \ref{sec:exp}, we use both synthetic and realistic examples to demonstrate the effectiveness of QHMC and QSGNHT. They can avoid parameter tunings and achieve superior accuracy for a wide range of distributions, especially for spiky and multi-modal ones that are common in Bayesian learning.

\begin{algorithm}[t]
\SetAlgoLined
\KwResult{starting point $\mat{x_0}$, step size $\epsilon$, simulation steps $L$, mass $\mat{M}=m\mat{I}$.}
 \For{$t=1,2,\cdots$}{
    Resample $\mat{q}\sim {\cal N}(0,\mat{M})$\;
    ($\mat{x}_0,\mat{q}_0$)=($\mat{x}^{(t)},\mat{q}^{(t)}$)\;
    Simulate dynamics based on Eq.~(\ref{eq:hd})\;
    $\mat{q}_0\gets \mat{q}_0-\frac{\epsilon}{2}\nabla U(\mat{x}_0)$\;
    \For{$i=1,\cdots,L-1$}{
    $\mat{x}_i\gets \mat{x}_{i-1}+\epsilon \mat{M}^{-1}\mat{q}_{i-1}$\;
    $\mat{q}_i\gets \mat{q}_{i-1}-\epsilon\nabla U(\mat{x}_i)$
    }
    $\mat{x}_L\gets \mat{x}_{L-1}+\epsilon \mat{M}^{-1}\mat{q}_{L-1}$\;
    $\mat{q}_L\gets \mat{q}_{L-1}-\frac{\epsilon}{2}\nabla U(\mat{x}_L)$\;
    ($\hat{\mat{x}},\hat{\mat{q}}$)=($\mat{x}_L,\mat{q}_L$)\;
    M-H step: $u\sim\mathrm{Uniform[0,1]}$\;
    $\rho = e^{-H(\hat{\mat{x}},\hat{\mat{q}})+H(\mat{x}^{(t)},\mat{q}^{(t)})}$\;
    \eIf{$u<\mathrm{min}(1,\rho)$}{$(\mat{x}^{(t+1)},\mat{q}^{(t+1)})=(\hat{\mat{x}},\hat{\mat{q}})$}{$(\mat{x}^{(t+1)},\mat{q}^{(t+1)})=(\mat{x}^{(t)},\mat{q}^{(t)})$}
    }
    \KwResults{$\{\mat{x}^{(1)},\mat{x}^{(2)},\cdots\}$}
 \caption{Hamiltonian Monte Carlo}
 \label{alg:hmc}
\end{algorithm}

\section{Hamiltonian Monte Carlo}\label{sec:HMC}
In Bayesian inference, we want to sample the parameters $\mat{x}$ from the posterior distribution $p(\mat{x}|{\cal D})$ given the dataset ${\cal D}$. Instead of directly sampling $\mat{x}$, HMC introduces an extra momentum variable $\mat{q}$, and samples in the joint space of $(\mat{x},\mat{q})$. By defining a potential function $U(\mat{x})=-\mathrm{log}\ p(\mat{x}|{\cal D})$ and a kinetic energy $K(\mat{q})=\frac{1}{2}\mat{q}^T\mat{M}_0^{-1}\mat{q}$ where $\mat{M}_0$ is a time-independent positive-definite mass matrix, HMC~\citep{Duane:1987de,neal2011mcmc} samples from the joint density of $(\mat{x},\mat{q})$ by simulating the following Hamiltonian dynamics:
\begin{equation}\label{eq:hd}
    d   \begin{pmatrix}
   \mat{x}\\\mat{q}
   \end{pmatrix}
   =dt
   \begin{pmatrix}
   \mat{M}_0^{-1}\mat{q}\\-\nabla U(\mat{x})
   \end{pmatrix}
\end{equation}
The resulting steady-state distribution is \begin{equation}
    \pi(\mat{x},\mat{q})\propto {\rm exp}\left( -\beta U\left( \mat{x} \right)-\beta K\left( \mat{q} \right)\right).
\end{equation}
where $\beta=1/k_BT$~\footnote{In statistical physics, $k_B$ is the Boltzmann constant, $T$ is the temperature.} and is set as 1 in standard HMC. Because $\mat{x}$ and $\mat{q}$ are independent with each other, one can easily marginalize the joint density over $\mat{q}$ to obtain the invariant distribution in the parameter space $\pi(\mat{x})\propto {\rm exp}(- U(\mat{x}))=p(\mat{x}|{\cal D})$. 

The algorithm flow of a standard HMC is summarized in Alg.~\ref{alg:hmc}. The HMC sampler switches between two steps: 1) travel on a constant energy surface according to Hamiltonian dynamics in Eq.~(\ref{eq:hd}) with step size $\epsilon$ and number of simulation steps $L$, and 2) maintain state $\mat{x}$ but resample momentum $\mat{q}$ to transit to another energy surface. In order to guarantee volume preservation and accuracy of the simulations~\citep{neal2011mcmc}, the ``leapfrog" integrator is commonly used in HMC. In practice, one needs to choose a step size $\epsilon$ to discretize the continuous time $t$, and the number of simulation steps $L$ to decide how many steps the dynamics runs before resampling the momentum.

However, the performance of a standard HMC can degrade due to the following limitations: 

\begin{itemize}
    \item Ill-conditioned distribution: the isotropic mass in HMC assumes the target distribution has the same scale in all directions, so HMC is poor at exploring ill-conditioned distributions, e.g., a Gaussian distribution with a covariance matrix whose eigenvalues have a wide spread.
    \item Multimodal distribution:
    although the momentum variable can help the particle reach higher energy levels than MCMC, it is still hard for HMC to explore multimodal distributions because of the fixed temperature ($k_BT=1$).
    \item Stochastic implementation: standard HMC involves full gradient computations over the whole data set. This can be very expensive when dealing with big-data problems. 
    \item Non-smooth distribution: HMC needs computing gradients of the energy function, therefore it can behave badly or even fail for non-smooth (discontinuous or spiky) functions.
\end{itemize}

It is interesting to note that some physical principles have been employed to overcome some aforementioned limitations of HMC. Specifically, magnetic HMC~\citep{tripuraneni2017magnetic} can efficiently explore multimodal distributions by introducing magnetic fields; Stochastic gradient HMC improves the efficiency of handling large data set by simulating the second-order Langevin dynamics~\citep{chen2014stochastic}; The relativistic HMC~\citep{lu2016relativistic} achieves a better convergence rate due to the analogy between the speed of light and gradient clipping; The reflection and refraction HMC~\citep{NIPS2015_5801} leverages optics theory to efficiently sample from discontinuous distributions; The wormhole HMC~\citep{lan2014wormhole}, motivated by Einstein's general relativity, builds shortcuts between two modes; Also inspired by general relativity, the Riemannian HMC~\citep{girolami2011riemann} can adapt a mass matrix pre-conditioner according to the geometry of a Riemannian manifold, which can better handle ill-conditioned distributions; Randomized HMC~\citep{Bou_Rabee_2017}, a special case of our work, models the lifetime of a path as an exponential distribution which is a common assumption in thermodynamics. Table \ref{tab:hmc_variant} has summarized some representative HMC algorithms and their corresponding physical models. 

Inspired by quantum mechanics, this work constructs a stochastic process with a time-varying mass matrix, aiming to sample efficiently from a possibly spiky or multi-modal distribution. Our method can be combined with other methods due to its efficiency (nearly no extra cost) and flexibility.

\begin{table}[]
    \centering
    \caption{Summary of various HMC methods.} \small
    \begin{tabular}{|l|l|l|}\hline
        Challenges to address & HMC variants & Physic theory \\\thickhline
        Ill-conditioned distribution& RMHMC~\citep{girolami2011riemann} & General relativity \\\hline
        \multirow{5}{*}{Multimodal dist.} & Magnetic HMC~\citep{tripuraneni2017magnetic} & Electromagnetism \\\cline{2-3}
        & Wormhole HMC~\citep{lan2014wormhole} & General relativity \\\cline{2-3}
        & Tempered HMC~\citep{graham2017continuously} & Thermodynamics 
        \\\cline{2-3}
        & RHMC~\citep{Bou_Rabee_2017} & Thermodynamics
        \\\cline{2-3}
        & Fractional HMC~\citep{ye2018stochastic} & L{\'e}vy Process
        \\\hline
        \multirow{3}{*}{Large training data set}
        & Stochastic Gradient HMC~\citep{chen2014stochastic} & Langevin dynamics \\\cline{2-3}
        & Thermostat HMC~\citep{ding2014bayesian} & Thermodynamics \\\cline{2-3}
        & Relativistic HMC~\citep{lu2016relativistic} & Special relativity \\\hline
        Discontinuous dist.& Optics HMC~\citep{NIPS2015_5801} & Optics\\\hline
        {\bf Spiky \& multimodal dist.}& {\bf Quantum-inspired HMC (this work) }& {\bf Quantum mechanics}\\\hline
    \end{tabular} \normalsize
    \label{tab:hmc_variant}
\end{table}

\section{Quantum-Inspired Hamiltonian Monte Carlo}\label{sec:QHMC}
In this section, we propose the physical model and numerical implementation of a quantum-inspired Hamiltonian Monte Carlo (QHMC). In Section \ref{sec:qhmc-steady} and \ref{sec:time-average} we prove that its resulting steady-state distribution and time-averaged distribution are indeed the targeted posterior distribution $p(\mat{x}| {\cal D})$. In Section \ref{sec:discuss}, we discuss a few possible extensions of QHMC.

\subsection{QHMC Model and Algorithm}

Different from HMC that simulates the Hamiltonian dynamics of a constant mass $\mat{M}_0$, our QHMC allows $\mat{M}(t)$ to be time-varying and random. Specifically, we construct a stochastic process and let each realization be $\mat{M}(t)$. 
In general at each time $t$, $\mat{M}(t)$ is sampled independently from a (time-independent) distribution $P_{\mat{M}}(\mat{M})$. In other words, $\{\mat{M}(t_1),\mat{M}(t_2),\cdots,\mat{M}(t_n)\}$ are i.i.d. positive definite matrices. Section \ref{sec:mass_adapt} will provide some heuristics on how to choose $P_\mat{M}(\mat{M})$ properly. After obtaining a realization of time-varying mass $\mat{M}(t)$, we simulate the following dynamical system:
\begin{equation}\label{eq:hmc_mass}
    d   \begin{pmatrix}
   \mat{x}\\\mat{q}
   \end{pmatrix}
   =dt
   \begin{pmatrix}
   \mat{M}(t)^{-1}\mat{q}\\-\nabla U(\mat{x})
   \end{pmatrix}
\end{equation}
where we assume that $U(\mat{x})$ a differentiable potential energy function with a well-defined gradient $\nabla U(\mat{x})$ everywhere, except for some points with zero measure~\footnote{A set of points on the $x$-axis is said to have measure zero if the sum of the lengths of intervals enclosing all the points can be made arbitrarily small. The extension to higher dimensions is straightforward.}. The discretized version of Eq.~(\ref{eq:hmc_mass}) is
\begin{equation}\label{eq:qhmc_update}
\left\{
\begin{aligned}
&\mat{q}_{i+\frac{1}{2}}\gets\ \mat{q}_i -\frac{\epsilon}{2}\nabla U(\mat{x}_i).\\
&\mat{x}_{i+1}\gets\ \mat{x}_i+\epsilon \mat{M}_t^{-1}\mat{q}_{i+\frac{1}{2}}\quad \\
&\mat{q}_{i+1}\gets\ \mat{q}_{i+\frac{1}{2}} -\frac{\epsilon}{2}\nabla U(\mat{x}_{i+1}).\\
\end{aligned}
\right.
\end{equation}
Where $t$ denotes the index of paths, $\mat{M}_t$ is the mass matrix used for the $t$-th path and is sampled from $P_\mat{M}(\mat{M})$ at the beginning of the path; $i=1,2,\cdots,L$ denotes the index of steps in each path. The resulting numerical algorithm flow is shown in Alg.~\ref{alg:qhmc}~\footnote{Some special yet useful cases of QHMC, including S-QHMC/D-QHMC/M-QHMC, will be discussed in Section \ref{sec:discuss}}. The implementation of QHMC is rather simple: compared with HMC, we only need an extra step of resampling the mass matrix. Note that HMC can be regarded as a special case of QHMC: QHMC becomes a standard HMC when $P_\mat{M}(\mat{M})$ is the Dirac delta function $\delta(\mat{M}-\mat{M}_0)$.

{\bf Physical Intuition:} In standard HMC~\citep{Duane:1987de,neal2011mcmc}  described by Eq.~(\ref{eq:hd}), the mass matrix $\mat{M}$ is usually chosen as diagonal such that $\mat{M}=m\mat{I}$ and further the scalar mass $m$ is commonly set as 1. Such a choice corresponds to the classical physics where the mass is scalar and deterministic. In the quantum physics, however, the energy and time obeys the following uncertainty relationship:
\begin{equation}\label{eq:energy-time}
    \Delta E\Delta t\approx \hbar.
\end{equation}
Here $\Delta E$ is the uncertainty of the static energy for a particle, and $\Delta t$ is seen as length of time the particle can survive (lifetime), and $\hbar\sim 1.05\times 10^{-34}$ is the Planck constant. Moreover, based on the well-known mass-energy relation discovered by Albert Einstein
\begin{equation}\label{eq:energy-mass}
    E=mc^2,
\end{equation}
we know that the mass of a particle is proportional to its static energy. Combining Eq. (\ref{eq:energy-time}) and (\ref{eq:energy-mass}), we conclude $\Delta m\approx \hbar/(c^2\Delta t)$, indicating that a quantum particle can have a random mass obeying the distribution $P_m(m)$ (more generally $P_\mat{M}(\mat{M})$), or equivalently, each realization of the stochastic process can produce a time-varying mass $\mat{M}=\mat{M}(t)$~\footnote{A classical particle differs from a quantum particle in many aspects. For instance, a classical particle has an infinite lifetime, deterministic position and momentum, and continuous energy; in contrast, a quantum particle has a finite lifetime, uncertain position and momentum, and discrete energy levels (for bound states). In this paper we only focus on the nature of mass uncertainty.}. In quantum mechanics, the mass distribution $P_\mat{M}(\mat{M})$ is only dependent on the types of the quantum particles, therefore we assume that $P_\mat{M}(\mat{M})$ is independent of $\mat{x}$ and $\mat{q}$ in this manuscript. Without this assumption, a wrong distribution may be produced, which will be shown in Section \ref{sec:time-average}.

To avoid possible confusions, we would like to state explicitly that the dynamics described by Eq.~(\ref{eq:hmc_mass}) is not a completely quantum model because both $\mat{x}$ and $\mat{q}$ are deterministic once an initial condition and a specific realization of $\mat{M}(t)$ are given. Indeed  a classical computer will probably not solve a quantum system with a lower computational cost than its classical counterpart.

\begin{algorithm}[t]
\SetAlgoLined
\KwResult{starting point $\mat{x}_0$, step size $\epsilon$, simulation steps $L$, and mass distribution $P_\mat{M}(\mat{M})$ \\ (For D-QHMC and S-QHMC, $P_\mat{M}(\mat{M})$ requires parameters $\vec{\mu}_m$/ $\vec{\sigma}_m$ and $\mu_m$/$\sigma_m$).}
 \For{$t=1,2,\cdots$}{
    Resample $\mat{M}_t\sim P_{\mat{M}}(\mat{M})$\;
    \quad [D-QHMC:  $\omega_k\sim\mathcal{N}(\mu_m^{(k)},\sigma_m^{(k)2}), m_{kk}=10^{\omega_k}, \mat{M}=\text{diag}(m_{11},\cdots,m_{dd})$]\;
    \quad [S-QHMC:  $\omega\sim\mathcal{N}(\mu_m,\sigma_m^2), m_{kk}=10^{\omega}, \mat{M}=m\mat{I}$]\;
    
    Resample $q\sim {\cal N}(0,\mat{M})$\;
    ($\mat{x}_0,\mat{q}_0$)=($\mat{x}^{(t)}, \mat{q}^{(t)}$)\;
    Simulate dynamics based on Eq. (2)\;
    $\mat{q}_0\gets \mat{q}_0-\frac{\epsilon}{2}\nabla U(\mat{x}_0)$\;
    \For{$i=1,\cdots,L-1$}{
    $\mat{x}_i\gets \mat{x}_{i-1}+\epsilon \mat{M}_t^{-1}\mat{q}_{i-1}$\;
    $\mat{q}_i\gets \mat{q}_{i-1}-\epsilon\nabla U(\mat{x}_i)$
    }
    $\mat{x}_L\gets \mat{x}_{L-1}+\epsilon \mat{M}_t^{-1}\mat{q}_{L-1}$\;
    $\mat{q}_L\gets \mat{q}_{L-1}-\frac{\epsilon}{2}\nabla U(\mat{x}_L)$\;
    ($\hat{\mat{x}},\hat{\mat{q}}$)=($\mat{x}_L,\mat{q}_L$)\;
    M-H step: $u\sim\mathrm{Uniform[0,1]}$\;
    $\rho = e^{-H(\hat{\mat{x}},\hat{\mat{q}})+H(\mat{x}^{(t)},\mat{q}^{(t)})}$\;
    \eIf{$u<\mathrm{min}(1,\rho)$}{$(\mat{x}^{(t+1)},\mat{q}^{(t+1)})=(\hat{\mat{x}},\hat{\mat{q}})$}{$(\mat{x}^{(t+1)},\mat{q}^{(t+1)})=(\mat{x}^{(t)},\mat{q}^{(t)})$}
    }
    \KwResults{ $\{\mat{x}^{(1)},\mat{x}^{(2)},\cdots\}$}
 \caption{Quantum-inspired Hamiltonian Monte Carlo (QHMC)}
 \label{alg:qhmc}
\end{algorithm}

\subsection{Steady-State Distribution of QHMC (Space Domain)}\label{sec:qhmc-steady}
Now we show that the continuous-time stochastic process in Eq. (\ref{eq:hmc_mass}) and the discrete version in Eq. (\ref{eq:qhmc_update}) and Alg.~\ref{alg:qhmc} can produce a correct steady distribution that describes the desired posterior density $p(\mat{x}|{\cal D}) \sim {\rm exp}(-U(\mat{x}))$.

\begin{theorem}
\label{thm:qhmc-sd}
Consider a continuous-time Hamiltonian dynamics with a deterministic time-varying positive-definite matrix $\mat{M}(t)$ in Eq. (\ref{eq:hmc_mass}). The time-dependent distribution $p(\mat{x}, \mat{q}, t)\propto {\rm exp}(-U(\mat{x})-\frac{1}{2}\mat{q}^T\mat{M}(t)^{-1}\mat{q})$ satisfies the Fokker-Planck equation of Eq. (\ref{eq:hmc_mass}). Furthermore, the marginal density $p_s(\mat{x})\propto{\rm exp}(-U(\mat{x}))$ is a unique steady-state distribution in the $\mat{x}$ space if momentum resampling steps $p(\mat{q})\propto{\rm exp}(-\frac{1}{2}\mat{q}^T\mat{M}(t)^{-1}\mat{q})$ are included.
\end{theorem}

Theorem \ref{thm:qhmc-sd} is the key result of this work, and we would like to prove it from two perspectives. In Section \ref{physics-proof}, we treat $\mat{M}(t)$ as a deterministic time scale control parameter. This view admits a nice physical interpretation of a re-scaled Hamiltonian dynamics. Furthermore, the relation between QHMC and randomzied HMC~\citep{Bou_Rabee_2017} is natural to see in this setting. In Section~\ref{math-proof}, we treat $\mat{M}{(t)}$ similarly with state variables $\mat{x}$ and $\mat{q}$. This perspective can facilitate the proof with the Bayes rule.

\subsubsection{Proof of Theorem \ref{thm:qhmc-sd} from the Perspective of Time Scaling}
\label{physics-proof}

In Lemma \ref{lemma:time-dependent-hmc}, we first prove that any time-modifier $\mat{A}(t)$ (the meaning of which will be clear later) in Eq. (\ref{eq:time-depend-hmc}) preserves the correct steady distribution.

\begin{lemma}
\label{lemma:time-dependent-hmc}
Consider a time-dependent continuous Markov process
\begin{equation}\label{eq:time-depend-hmc}
d   \begin{pmatrix}
   \mat{x}\\\mat{q}
   \end{pmatrix}
   =dt
   \begin{pmatrix}
   \mat{A}(t)&0 \\
   0 & \mat{A}(t)
   \end{pmatrix}
   \begin{pmatrix}
   \mat{q}\\-\nabla U(\mat{x})
   \end{pmatrix}
\end{equation}
where $\mat{A}(t)$ is a deterministic time-varying symmetric matrix. For general symmetric $\mat{A}(t)$, we have a steady distribution for Eq. (\ref{eq:time-depend-hmc}) as $p_s(\mat{x},\mat{q})\propto {\rm exp}(-U(\mat{x})-\mat{q}^T\mat{q}/2)$. Further if a jump process exists, i.e. the momentum is resampled every time interval $t_0$ as $p(\mat{q})\propto {\rm exp}(-\mat{q}^T\mat{q}/2)$, then the steady distribution is unique.
\end{lemma}

\begin{proof}
The evolution of probability density $p(\mat{x}, \mat{q})$ for the particles $(\mat{x}, \mat{q})$ in Eq. (\ref{eq:time-depend-hmc}) can be described by the following Fokker-Planck equation~\footnote{The Fokker-Planck equation is a generic tool to analyze the time evolution of the density function for the unknown state variables in a stochastic differential equations (SDE). Although Eq.~(\ref{eq:time-depend-hmc}) is a deterministic process instead of a stochastic one, it falls into the framework of SDE as a special case with a zero diffusion term. In statistical mechanics, the special deterministic case is referred as the Liouville's theorem~\citep{ma14ller2013basics}.}:
\begin{equation}
    \partial_tp+\frac{d\mat{x}}{dt}\cdot\nabla_\mat{x} p+\frac{d\mat{q}}{dt}\cdot\nabla_\mat{q} p=0.
\end{equation}
We consider the specific density function $p_s(\mat{x},\mat{q})\propto {\rm exp}(-U(\mat{x})-\mat{q}^T\mat{q}/2)$. By setting $d\mat{x}/dt=\mat{A}(t)\mat{q}$, $d\mat{q}/dt=-\mat{A}(t)\nabla U(\mat{x})$ according to Eq. (\ref{eq:time-depend-hmc}) and noticing that $\nabla_\mat{x} p_s=-p_s\nabla U(\mat{x})$ and $\nabla_\mat{q} p_s = -p_s\mat{q}$, we have
\begin{equation}
    \partial_t p_s = 0.
\end{equation}
Therefore $p_s(\mat{x},\mat{q})\propto {\rm exp}(-U(\mat{x})-\mat{q}^T\mat{q}/2)$ is a stationary distribution of the time-dependent process described by Eq. (\ref{eq:time-depend-hmc}). If the momentum is resampled for a certain time interval, the steady distribution will not change and the Markov process will be guaranteed as ergodic, because the sufficient conditions for ergodicity~\citep{borovkov1998ergodicity} are (1) irreducibility (2) aperiodicity and (3) positive recurrence, the first two of which are guaranteed by the resampling step as diffusion noise. Condition (3) is holds when $U(\mat{x})\to\infty$ for $|\mat{x}|\to\infty$ which is a reasonable assumption for energy functions. Once the Markov process is ergodic, the steady distribution should be unique~\citep{borovkov1998ergodicity}.
\end{proof}

The physical interpretation of $\mat{A}(t)$ can be understood by looking at a special case $\mat{A}(t)=a(t)\mat{I}$. In this case $a(t)$ can be absorbed into $dt$ as $dt'=a(t)dt$, and $a(t)$ increases or shrink the time scale (hence the name ``time-modifier''). The steady distribution is by definition independent of time, thus $a(t)$ does not change the steady distribution. Note that Eq. (\ref{eq:time-depend-hmc}) becomes a standard time-independent Hamiltonian dynamics with $\mat{M}=\mat{I}$ when $\mat{A}(t)=\mat{I}$.

Now we are ready to prove theorem \ref{thm:qhmc-sd}. We show that employing a (deterministic) time-varying mass matrix $\mat{M}(t)$ is equivalent to employing a (deterministic) time-modifier $\mat{A}(t)$.

\begin{proof}
We change variables from $(\mat{x},\mat{q})$ to $(\mat{x}',\mat{q}')$ with the transformation:
\begin{equation}\label{eq:change-variable-2}
\left\{
             \begin{array}{ll}
             \mat{x}'=&\mat{x}\\
             \mat{q}'=&\mat{M}(t)^{-1/2}\mat{q}
             \end{array}
\right.
\end{equation}
where $\mat{M}^{-1/2}$ is defined as $\mat{V}\mat{D}^{-1/2}\mat{V}^T$ with the diagonalization of $\mat{M}$ as $\mat{M}=\mat{V}\mat{D}\mat{V}^T$. After changing variables, Eq. (\ref{eq:hmc_mass}) is transformed to
\begin{equation}
    d   \begin{pmatrix}
   \mat{x}'\\\mat{q}'
   \end{pmatrix}
   =dt
   \begin{pmatrix}
   \mat{M}(t)^{-1/2}&0 \\
   0 & \mat{M}(t)^{-1/2}
   \end{pmatrix}
   \begin{pmatrix}
   \mat{q}'\\-\nabla U(\mat{x}')
   \end{pmatrix}
\end{equation}
Because $\mat{M}(t)^{-1/2}$ is a symmetric matrix, according to Lemma \ref{lemma:time-dependent-hmc}, we have a unique steady distribution for $(\mat{x}',\mat{q}')$ such that $p_s'(\mat{x}',\mat{q}')\propto {\rm exp}(-U(\mat{x}')+\mat{q}'^T\mat{q}'/2)$. After transforming $(\mat{x}',\mat{q}')$ back to the $(\mat{x},\mat{q})$ space, the distribution $p_s(\mat{x},\mat{q},t)\propto ({\rm det}\mat{M}(t))^{-\frac{1}{2}} {\rm exp}(-U(\mat{x})+\mat{q}^T\mat{M}(t)^{-1}\mat{q}/2)$ is dependent on $t$, where ${\rm det}(\mat{M}(t))$ arises from the Jacobian of variable transformation:
\begin{equation}
\begin{aligned}
    p_s(\mat{x},\mat{q},t)d\mat{x}d\mat{q}&=p'_s(\mat{x}',\mat{q}',t)d\mat{x}'d\mat{q}'\\
    p_s(\mat{x},\mat{q},t)=p_s'(\mat{x}',\mat{q}',t){\rm det}
    \begin{pmatrix}
    \frac{\partial\mat{x}'}{\partial\mat{x}} & \frac{\partial\mat{x}'}{\partial\mat{q}}\\
    \frac{\partial\mat{q}'}{\partial\mat{x}} & \frac{\partial\mat{q}'}{\partial\mat{q}'}
    \end{pmatrix}
    &\propto ({\rm det}\mat{M}(t))^{-\frac{1}{2}}{\rm exp}(-U(\mat{x})+\mat{q}^T\mat{M}(t)^{-1}\mat{q}/2).
\end{aligned}
\end{equation}
After marginalization over the momentum $\mat{q}$ we obtain the marginal probability density $p_s(\mat{x})\propto {\rm exp}(-U(\mat{x}))$~\footnote{Here ${\rm det}(\mat{M}(t))$ is a scaling factor independent of $\mat{x}$ and $\mat{q}$, and it vanishes after normalization.}, which is again independent of $t$ hence a steady distribution.
\end{proof}

The intuition of Theorem \ref{thm:qhmc-sd} is as follows. The time-varying matrix $\mat{M}(t)$ has an effect of increasing or shrinking the time scale, but it does not change the steady distribution in the $\mat{x}$ space. As a corollary of Theorem \ref{thm:qhmc-sd}, we further show that the steady distribution of our discretized QHMC algorithm indeed is the desired posterior distribution.

\begin{corollary}\label{cor:qhmc_sd}
Consider a piecewise $\mat{M}(t)$. The continuous time is divided into pieces as $\cdots< t_{n-1}<t_n<t_{n+1}<\cdots$ and on each time interval, $\mat{M}$ is a constant matrix i.e. $\mat{M}(t)=\mat{M}_n$ if $t_n\leq t<t_{n+1}$. By sampling $\mat{M}_n\sim P_\mat{M}(\mat{M})$ as done in Alg.~\ref{alg:qhmc}, we have the correct teady distribution $p_s(\mat{x})\propto{\rm exp}(-U(\mat{x}))$ for each interval.
\end{corollary}

If we assume $\mat{M}(t)=m(t)\mat{I}$ and interpret $\mat{A}(t)=\epsilon(t)\mat{I}$ as the continuous matrix version of step size $\epsilon$, then changing the mass is equivalent to changing the step size. From the proof of Theorem \ref{thm:qhmc-sd} we know $\mat{M}(t)\sim\mat{A}(t)^{-2}$ or $m\sim\epsilon^{-2}$. The equivalence $m\sim\epsilon^{-2}$ can be intuitively understood as the result of a ``scaling property" of  Eq.~(\ref{eq:hd}): given a trajectory $(\mat{x}(t),\mat{q}(t))$ for a particle with mass $m$, the rescaled trajectory $t\to \alpha t$, $\mat{q}\to \alpha\mat{q}, \mat{x}\to\mat{x}$ for a particle with mass $m\to \alpha^2 m$ is just the original trajectory but with a different time scale which also obeys the Hamiltonian dynamics Eq.~(\ref{eq:hd}).

Although QHMC with a scalar mass is equivalent to HMC with a randomized step~\citep{Bou_Rabee_2017,neal2011mcmc}, QHMC is a more general formulation because $\mat{M}$ can be any positive definite matrix other than proportional to identity. We elaborate the benefits of QHMC over randomized HMC in Section~\ref{sec:discuss} and in our numerical experiments.

\subsubsection{Proof of Theorem \ref{thm:qhmc-sd} via the Bayes Rule}\label{math-proof}
Alternatively, we can treat $\mat{M}$ as additional random variables just like $\mat{x}$ and $\mat{q}$. Then we can consider the joint distribution $(\mat{x},\mat{q},\mat{M})$ and use the Bayes rule to prove Theorem~\ref{thm:qhmc-sd}.

\begin{proof}
We consider the joint steady distribution $p_s(\mat{x},\mat{q},\mat{M})$. Here we have dropped the explicit dependence of $\mat{M}$ on $t$ because $\mat{M}(t)$ obeys the mass distribution $P_\mat{M}(\mat{M})$ for all $t$. Leveraging Bayesian theorem, we have
\begin{equation}
    p_s(\mat{x},\mat{q},\mat{M})=p_s(\mat{x},\mat{q}|\mat{M})P_\mat{M}(\mat{M})
\end{equation}
Recall that $p_s(\mat{x},\mat{q}|\mat{M})\propto {\rm exp}(-U(\mat{x})){\rm exp(-\frac{1}{2}\mat{p}^T\mat{M}^{-1}\mat{p}})$, it immediately implies
\begin{equation}
    p_s(\mat{x})=\int_{\mat{q}}\int_{\mat{M}}d\mat{q}\ d\mat{M}\ p_s(\mat{x},\mat{q},\mat{M})\propto {\rm exp}(-U(\mat{x}))
\end{equation}
which means the marginal steady distribution $p_s(\mat{x})$ is the correct posterior distribution.
\end{proof}

The proof relies on two facts: (1) $\mat{M}$ itself has a steady distribution in time (i.i.d. for QHMC); (2) $\mat{M}$ should not depend explicitly on $\mat{x}$, otherwise the samples from the simulated dynamics will produce the wrong posterior distribution, as we will show below.

\subsection{Time-Averaged Distribution (Time Domain)}\label{sec:time-average}

We still need to show that the time-averaged distribution of QHMC is our desired result. The reason is as follows: rather than simulating a set of particles simultaneously, HMC-type methods simulate one particle after another, and collect a set of samples at a sequence of time points. Therefore, the time-averaged distribution $p_t(\mat{x})$, rather than the space-averaged distribution (steady distribution $p_s(\mat{x})$), is our final obtained sample distribution. Here the time-averaged distribution is defined as:

\begin{equation}\label{eq:def_pt}
p_t(\mat{x})=
\left\{
             \begin{array}{lc}
             \underset{T\to\infty}{{\rm lim}}\frac{1}{T}\sum_{t=0}^T \mathbbm{1}(\mat{x}(t)=\mat{x})\quad &({\rm discrete})\\
             \underset{T\to\infty}{{\rm lim}}\frac{1}{T}\int_0^T \delta(\mat{x}(t)-\mat{x})dt\quad &({\rm continuous})\\
             \end{array}
\right.
\end{equation}

Where $\mathbbm{1}(\mat{x}=\mat{x})$ is the indicator function (1 if the argument in the bracket is true and 0 otherwise), and $\delta(\mat{x}(t)-\mat{x})$ is the Dirac delta function ($+\infty$ if the number in the bracket is 0 and 0 otherwise, plus a normalization criteria as $\int_{\mat{x}}\delta(\mat{x}(t)-\mat{x})d\mat{x}=1$)~\footnote{One can show that $p_t(\mat{x})$ is a probabilistic density function because $\sum_\mat{x}p_t(\mat{x})=1$ or $\int_\mat{x}p_t(\mat{x})d\mat{x}=1$.}. The physical meaning of $p_t(\mat{x})$ corresponds to drawing a histogram of samples obtained in the time sequence in HMC and $T$ is the number of simulation paths. In most HMC methods where the step size or mass is fixed, the equivalence between $p_s(\mat{x})$ and $p_t(\mat{x})$ can be justified with ergodicity theory from the mathematical perspective~\citep{gray1988probability,rosenthal1995convergence,bakry2008rate} or ensemble theory from the physical perspective~\citep{PhysRev.91.784,oliveira2007ergodic}. Also Eq.~(\ref{eq:hd}) has the steady distribution $p_s(\mat{x})=p(\mat{x}|{\cal D})$. As a result, both $p_t(\mat{x})$ and $p_s(\mat{x})$ can produce the correct posterior distribution $p(\mat{x}|{\cal D})$. However for QHMC, the step size or mass is effectively modified by $A(t)$ (c.f. Theorem \ref{thm:qhmc-sd} and Lemma \ref{lemma:time-dependent-hmc}), or more generally $A(\mat{x},\mat{q},t)$. As a result, it is non-trivial to obtain the equivalence between $p_t(\mat{x})$ and $p(\mat{x}|{\cal D})$ in QHMC.

{\bf Observations.} We first illustrate the above issue by considering the piecewise energy function
\begin{equation}\label{eq:asymmetric-well}
U(x)=
\left\{
             \begin{array}{lc}
             -x\quad &(x<0)\\
             3x\quad &(x\geq 0).\\
             \end{array}
\right.
\end{equation}
Because $U(x)$ has larger gradients in $x\geq 0$ than $x<0$, one may want to change the mass (and thus the step size) in the QHMC simulation. We consider two different schemes:
\begin{itemize}
    \item {\bf Explicit mass adaptation.} One may use a small mass $m=0.1$ (or equivalently a large step size) for $x<0$ and $m=1$ for $x>0$. However, this explicit adaptation can lead to a wrong time-averaged distribution. As shown in Fig.~\ref{fig:adaptation} (a), the region $x>0$ has more samples than required by the true posterior distribution. This is because a small step size results in more frequent proposals [the number of proposals per unit time is $1/(\epsilon L)$], and more proposals in the region $x>0$ ends up with more accepted samples than the ground truth~\footnote{In this example, one is able to produce the true posterior distribution with an explicit step size adaption as $\epsilon(x)\sim|\nabla U(x)|^{-1}$ by fixing the time interval. However, this method is generally less efficient than QHMC: QHMC makes a proposal every $L$ steps, whereas this explicit mass-adaption method does not have a proper upper bound for the number of steps. For instance, in very spiky regions ($|\nabla U(x)|$ very large), it takes $O(|\nabla U(x)|)$ steps to make a new proposal.}.
    
    \item {\bf Implicit mass adaption.} The mass is a random variable with a Bernoulli distribution: we have an equal probability to choose either $m=0.1$ or $m=1$. This choice is independent of $x$, and it produces the correct time-averaged distribution as shown in Fig.~\ref{fig:adaptation} (b). 
\end{itemize}

\begin{figure}
    \centering
    \begin{subfigure}[b]{0.475\textwidth}
            \centering
            \includegraphics[trim = 35mm 86mm 30mm 80mm, clip,width=\textwidth]{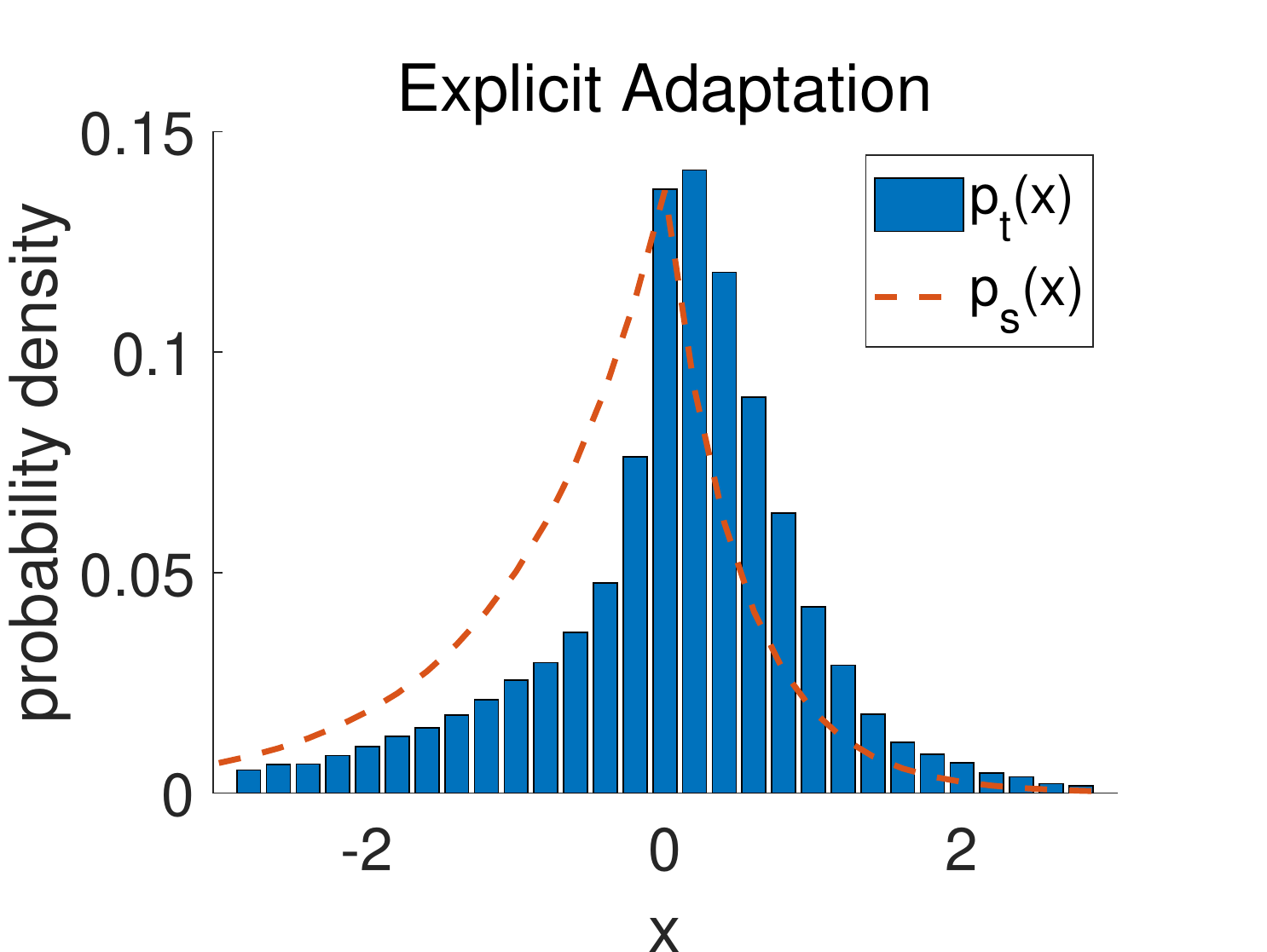}
            {{\small }}    
        \end{subfigure}
        \hfill
        \begin{subfigure}[b]{0.475\textwidth}  
            \centering 
            \includegraphics[trim = 35mm 86mm 30mm 80mm, clip,width=\textwidth]{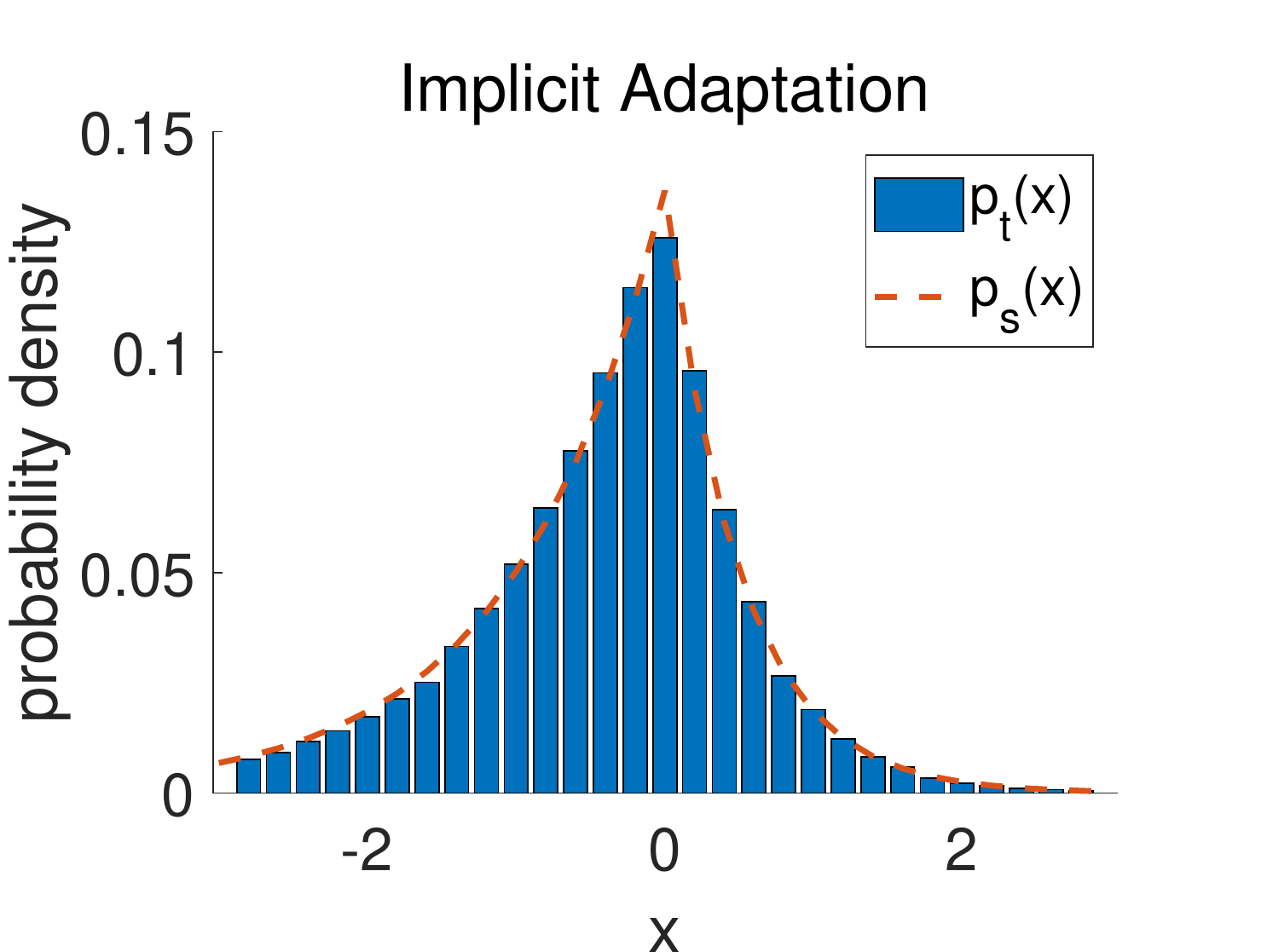}
            {{\small }} 
        \end{subfigure}
    \caption{(a) Explicit mass adaptation can lead to wrong time averaged distribution; (b) Implicit mass adaptation produces the correct time averaged distribution.}
    \label{fig:adaptation}
    \vspace{-10pt}
\end{figure}

Our QHMC method employs the implicit mass adaptation. The benefits of this implicit method will be theoretically justified in Section \ref{sec:mass_adapt}. Here we first make the aforementioned intuitions more rigorous by proposing Theorem \ref{thm:time-average}.

\begin{theorem}\label{thm:time-average}
Denote $\mat{z}=(\mat{x},\mat{q})$, and consider the trajectory $\mat{z}(t)$ described by the dynamics (\ref{eq:time-depend-hmc-2}):

\begin{equation}\label{eq:time-depend-hmc-2}
d   \begin{pmatrix}
   \mat{x}\\\mat{q}
   \end{pmatrix}
   =dt
   \begin{pmatrix}
   \mat{A}(\mat{z},t)&0 \\
   0 & \mat{A}(\mat{z},t)
   \end{pmatrix}
   \begin{pmatrix}
   \mat{q}\\-\nabla U(\mat{x})
   \end{pmatrix}
\end{equation}

The only difference between Eq.~(\ref{eq:time-depend-hmc-2}) and Eq.~(\ref{eq:time-depend-hmc}) is that $A(\mat{z},t)$ replaces $A(t)$. The time-averaged distribution of the particle can be defined as
\begin{equation}\label{eq:time-average}
    p_t(\mat{z})=\underset{T\to\infty}{{\rm lim}}\frac{1}{T}\int_0^T \delta(\mat{z}(t)-\mat{z})dt
\end{equation}
We restrict our discussions to the case $\mat{A}(\mat{z},t)=a(\mat{z},t)\mat{I}$. We denote $p(\mat{z}|{\cal D})=p(\mat{x}|{\cal D})p(\mat{q})=p(\mat{x}|{\cal D}){\rm exp}(-\mat{q}^T\mat{q}/2)$. For the special case (implicit mass adaptation) $a(\mat{z},t)=a(t)$, we have $p_t(\mat{z})=p(\mat{z}|{\cal D})$. However, $p_t(\mat{z})\neq p(\mat{z}|{\cal D})$ for the general case (in an explicit mass adaptation) $a(\mat{z},t)\neq a(t)$. 
\end{theorem}

\begin{proof} 
Our goal is to calculate $p_t({\mat{z}})$ and compare it with $p(\mat{z}|{\cal D})$. The key idea of proof includes two steps:

\begin{itemize}
    \item Step 1: Change the time variable from $t$ to $t'$ and define $p_{t'}(\mat{z})$. We find $p_{t'}(\mat{z})=p(\mat{z}|D)$.
    \item Step 2: Establish the relation between $p_t(\mat{z})$ and $p_{t'}(\mat{z})$. We find $p_t(\mat{z})\propto E_t(\frac{1}{a(\mat{z},t)})p_{t'}(\mat{z})$.
\end{itemize}

{\bf Step 1}: In Eq. (\ref{eq:time-depend-hmc-2}) we can define a new time variable $t'$ such that $dt'=a(\mat{z},t)dt$, and rewrite the trajectory $\mat{z}(t)$ as $\mat{z}'(t')$. With the new time variable $t'$, the original time-varying Eq.~(\ref{eq:time-depend-hmc-2}) becomes the standard HMC dynamics in (\ref{eq:hd}). Therefore, we have $p_{t'}(\mat{z})=p(\mat{z}|{\cal D})$ where $p_{t'}(\mat{z})$ is defined as

\begin{equation}\label{eq:ptprime}
    p_{t'}(\mat{z})=\underset{T'\to\infty}{{\rm lim}}\frac{1}{T'}\int_0^{T'} \delta(\mat{z}'(t')-\mat{z})dt'.
\end{equation}

{\bf Step 2}: First we note  
\begin{equation}\label{eq:delta-transform}
    \int_0^T \delta(\mat{z}(t)-\mat{z})dt=\sum_{i=1}^{N_T(\mat{z})}\frac{1}{|\frac{d\mat{z}}{dt}(t_i)|}\bigg|_{\mat{z}(t_i)=\mat{z}}
\end{equation}
where $N_T(\mat{z})$ counts the number of $t_i\leq T$ where $\mat{z}(t_i)=\mat{z}$. Consequently, both $p_t(\mat{z})$ and $p_{t'}(\mat{z})$ can be rewritten as 
\begin{equation}\label{eq:14}
    p_t(\mat{z})=\underset{T\to\infty}{{\rm lim}}\frac{1}{T}\sum_{i=1}^{N_T(\mat{z})} \frac{1}{|\frac{d\mat{z}}{dt}(t_i)|}\bigg|_{\mat{z}(t_i)=\mat{z}},
\end{equation}
\begin{equation}\label{eq:15}
    p_{t'}(\mat{z})=\underset{T'\to\infty}{{\rm lim}}\frac{1}{T'}\sum_{i=1}^{N'_{T'}(\mat{z})} \frac{1}{|\frac{d\mat{z}'}{dt'}(t'_i)|}\bigg|_{\mat{z}'(t'_i)=\mat{z}}.
\end{equation}

We further note 
\begin{equation}
\label{eq:16}
    \bigg|\frac{d\mat{z}}{dt}(t_i)\bigg|=\frac{dt'_i}{dt_i}\bigg|\frac{d\mat{z}'}{dt'}(t'_i)\bigg|=a(\mat{z}(t_i),t_i)\bigg|\frac{d\mat{z}'}{dt'}(t'_i)\bigg|.
\end{equation}
Combining Eq.~(\ref{eq:14}), \eqref{eq:15} and \eqref{eq:16}, we have
\begin{equation}\label{eq:17}
    p_t(\mat{z})=\underset{T'\to\infty}{{\rm lim}}\frac{1}{T'}\sum_{i=1}^{N'_{T'}(\mat{z})} \frac{1}{|\frac{d\mat{z}'}{dt'}(t'_i)|}\times \frac{1}{a(\mat{z},t_i)}\bigg|_{\mat{z}'(t'_i)=\mat{z}}\propto E_t(\frac{1}{a(\mat{z},t)})p_{t'}(\mat{z})=E_t(\frac{1}{a(\mat{z},t)})p(\mat{z}|{\cal D}) 
\end{equation}
Where the last equality holds due to $p_{t'}(\mat{z})=p(\mat{z}|{\cal D})$. When $a(\mat{z},t)=a(t)$, $E_t(1/a(\mat{z},t))$ is independent of $\mat{z}$,  and $p_t(\mat{z})=p(\mat{z}|{\cal D})$. Otherwise, $p_t(\mat{z})\neq p(\mat{z}|{\cal D})$ in general when $a(\mat{z},t)$ explicitly depends on $\mat{z}$.
\end{proof}

In the QHMC algorithm (Alg.~\ref{alg:qhmc}), $P_\mat{M}(\mat{M}) $ is independent of $\mat{x}$ or $\mat{q}$, therefore the equivalence between $p_t(\mat{x})$ and $p(\mat{x}|{\cal D})$ is justified. In contrast, explicit mass adaptation has $p_t(\mat{x})\neq p(\mat{x}|{\cal D})$ in general. In practical simulations, only a finite number of samples are obtained. The asymptotic distribution error for a general ergodic Markov process has an exponential decay rate (i.e. linear convergence rate)~\citep{bakry2008rate}, such that $|p_t(\mat{x})-p(\mat{x}|{\cal D})|_\infty\leq C\rho^t$ where $p_t(\mat{x})$ is the finite-time sample distribution, $\rho<1$ and $C$ are constant. Therefore, the obtained sample distribution $p_t(\mat{x})$ can be arbitrarily close to the true posterior distribution $p(\mat{x}|{\cal D})$ as long as the stochastic process is simulated for long enough time.

In the current QHMC formulation $\mat{M}$ is independent of $\mat{x}$. However, it is possible to include a state-dependent mass matrix $\mat{M}(\mat{x})$ as done in Riemannian HMC~\citep{girolami2011riemann} by introduction some correction terms~\citep{ma2015complete}. This paper focuses on the benefits of QHMC caused by its time-varying mass, therefore we ignore the possible employment of a state-dependent mass matrix.

\subsection{Some Remarks}
\label{sec:discuss}

We have some remarks about the general framework of QHMC.

\begin{itemize}
    \item {\bf Some simple cases of QHMC.} Although the choice of the mass distribution $P_{\mat{M}}(\mat{M})$ can be quite general and flexible, a few simple cases turn out to be quite useful in practice: multi-modal, diagonal or scalar mass matrices (corresponding to M-QHMC, D-QHMC or S-QHMC, respectively). Their mass density functions are given below:
\begin{align}
    &\text{\bf M-QHMC}: \mat{M}\sim \sum \limits_{i=1}^k \alpha_i \delta(\mat{M}-\mat{M}_i),\; \text{with} \; \alpha_i \geq0\; \text{and} \; \sum_{i=1}^k \alpha_i=1 \nonumber \\
    &\text{\bf D-QHMC}: {\rm log}\; m_{ii}\sim \mathcal{N}(\mu_{m}^{(i)},\sigma_m^{(i)2}),  \mat{M}=\text{diag}(m_{11},\cdots,m_{dd}) \nonumber \\
    &\text{\bf S-QHMC}: {\rm log}\; m\sim \mathcal{N}(\mu_m,\sigma_m^2), \mat{M}=m\mat{I}. \nonumber
\end{align}
These simple choices often produce excellent results in practice (see Section~\ref{sec:exp}). The D-QHMC can be regarded as a ``quantum" version of preconditioned HMC which has a preconditioner for the mass matrix, and it can improve the sampling performance from an ill-conditioned distribution. The M-QHMC can facilitate sampling from a multi-modal distribution.  In both D-QHMC and S-QHMC, the nonzero entries of the mass matrix are sampled from log-normal distributions, which will be justified in Section \ref{sec:mass_adapt}. In our implementation, we often fix $\sigma_m=0$ firstly to find the optimal $\mu_m$, then we fix $\mu_m$ and try different values of $\sigma_m$. We find that $\sigma_m \in [0.5,3]$ is often an excellent choice. It is worth investigating how to choose $\mu_m$ and $\sigma_m$ in a more rigorous and automatic way in the future.

\item {\bf Differences between QHMC and randomized HMC~\citep{Bou_Rabee_2017}.} Randomized HMC is equivalent to S-QHMC which is a special case of QHMC. Our QHMC has a natural theoretical analysis based on the continuous Fokker-Planck equation (as in Theorem \ref{thm:qhmc-sd}). In contrast, the randomized HMC allows a continuous description only when the step size approaches zero. Meanwhile, the parameterizations of randomized HMC and our QHMC are entirely different. Randomized HMC intuitively chooses the exponential form $t\sim \mathrm{exp}(-t/\tau)$ inspired from thermodynamics. This exponential form contains the characteristic time $\tau$, indicating that the sampled time can barely differ from $\tau$ in magnitudes. In contrast, the log-normal parameterization of QHMC allows the mass samples to spread in wide ranges (across multiple magnitudes). This novel parameterization leads to the superior performance of QHMC on the spiky examples in Section \ref{sec:exp}. 

\item {\bf Possible combinations of QHMC with other advanced techniques.} The QHMC method has almost zero extra costs compared with the standard HMC: one only needs to re-sample the mass matrix $\mat{M}$ from distribution $P_{\mat{M}}(\mat{M})$ before performing each standard HMC simulation path. Due to the ease of implementation, QHMC may be easily combined with other techniques, such as the Riemannian manifold method~\citep{girolami2011riemann}, the continuous tempering technique~\citep{graham2017continuously} and/or the ``no U-Turn'' sampler~\citep{hoffman2014no}. The M-QHMC implementation may be further improved by combining it with the techniques in continuous tempering HMC~\citep{graham2017continuously}, magnetic HMC~\citep{tripuraneni2017magnetic} and/or wormhole HMC~\citep{lan2014wormhole}.

\end{itemize}

\section{Implicit Mass Adaption in QHMC}
\label{sec:mass_adapt}

The mass matrix is set as a positive definite random matrix in QHMC. In this section, we explain how this treatment can benefit the Bayesian sampling in various scenarios. 

\subsection{Light Particles for Handling Smooth Energy Functions}
\label{sec:small m}

While using heavy particles (and small step size) can produce more accurate simulation results and higher acceptance rates of a Hamiltonian dynamics, the simulation suffers from random walking behavior and low mixing rates. Here we show that the mass should be appropriately small in order to sample a distribution in a broad region when the associated energy function is very smooth. 

We consider a continuous energy function $U(\mat{x})$ in $\mat{x}\in\mathbb{R}^d$ and with $\beta$-smoothness:
\begin{equation}
    ||\nabla U(\mat{x})-\nabla U(\mat{y})||_2\leq \beta||\mat{x}-\mat{y}||_2
\end{equation}
and an HMC implementation with the leapfrog scheme~\footnote{Here we have ignored the half-step momentum updates at the beginning and end of each simulation path.}:
\begin{equation}\label{eq:eq78}
\begin{aligned}
    \mat{x}_{n+1}&=\mat{x}_n+\epsilon \mat{M}^{-1}\mat{q}_n\\
    \mat{q}_{n+1}&=\mat{q}_n-\epsilon \nabla U(\mat{x}_{n+1}).
\end{aligned}
\end{equation}

In the following, we qualitatively measure the random walking effects with different choices of mass parameters. We first investigate a quadratic energy function in Lemma \ref{lemma:3}.

\begin{lemma}\label{lemma:3}
For the quadratic energy function $U(\mat{x})=\frac{1}{2}\mat{x}^T\mat{A}\mat{x}$ (with $\mat{A}$ being symmetric), the discrete dynamical system in Eq.~(\ref{eq:eq78}) has bounded trajectories if and only if $\mat{M}\succ\frac{\epsilon^2}{4}\mat{A}\succ\mat{0}$. Here for two symmetric matrices $\mat{A},\mat{B}$ the inequality $\mat{A}\succ\mat{B}$ means $\mat{A}-\mat{B}$ is positive definite. 
\end{lemma}

\begin{proof}
By eliminating $\mat{q}_n$ in (\ref{eq:eq78}) we have the second-order recurrence relation:
\begin{equation}\label{eq:general_k>2}
    \mat{x}_{n+2}-2\mat{x}_{n+1}+\mat{x}_n+\epsilon^2\mat{M}^{-1}\nabla U(\mat{x}_{n+1})=0.
\end{equation}
Because $U(\mat{x})=\frac{1}{2}\mat{x}^T\mat{A}\mat{x}$, the Eq. (\ref{eq:general_k>2}) becomes 
\begin{equation}\label{eq:eq31}
    \mat{x}_{n+2}+(\epsilon^2\mat{M}^{-1}\mat{A}-2\mat{I})\mat{x}_{n+1}+\mat{x}_n=0
\end{equation}

According to Sylvester's law of inertia, the spectrum of $\mat{M}^{-1}\mat{A}$ is identical to the spectrum of $\mat{M}^{-\frac{1}{2}}\mat{A}\mat{M}^{-\frac{1}{2}}$, which is symmetric and hence contain only real eigenvalues. Eq. (\ref{eq:eq31}) is bounded if and only if
\begin{equation}
    2\mat{I}\succ \epsilon^2\mat{M}^{-1}\mat{A}-2\mat{I}\succ -2\mat{I}
\end{equation}
which is equivalent to $\mat{M}\succ\frac{\epsilon^2}{4}\mat{A}\succ\mat{0}$.
\end{proof}

We provide a few remarks here. Firstly, the inequality $\frac{\epsilon^2}{4}\mat{A}\succ\mat{0}$ implies that $\mat{A}$ should be positive definite, otherwise the quadratic potential function does not have a lower bound and cannot ``trap" a particles in a bounded region. Secondly, the first inequality implies that  the mass matrix should have large eigenvalues to make the trajectories bounded, which corresponds to a heavy particle moving slowly in $\mat{x}$ space.

Next we consider a more general smooth and continuous energy function.

\begin{theorem}\label{theorem:5}
Assume $U(\mat{x})$ is a differentiable and continuous energy function with $\beta$-smoothness. Denote the smallest eigenvalue of $\mat{M}$ as $m_{min}>0$ and assume it satisfies $m_{min}>\frac{\beta\epsilon^2}{6}$. Without loss of generality we assume $\nabla U(\mat{x}^*)=0$ for $\mat{x}^*=\mat{0}$, and we denote $\mat{C}=(\frac{\beta\epsilon^2}{2}\mat{M}^{-1}-1)$. Lemma \ref{lemma:3} implies that the sequence $\{\mat{x}_n^*\}$ generated by $\mat{x}_{n+2}^*+\mat{C}\mat{x}_{n+1}^*+\mat{x}_n^*=0$ is bounded, i.e. $||\mat{x}_n^*||_2<\Delta_0$. There exist constants $A$ and $\Delta$, such that $||\mat{x}_n||_2<\Delta$ for all $\mat{x}_n (n<N_0)$ generated by Eq.~(\ref{eq:general_k>2}). Here $A$ only depends on the initial condition $\mat{x}_0$, and $N_0$ is given by
\begin{equation}\label{eq:N0}
    N_0=\mathrm{log}_\lambda(\frac{\Delta}{A\Delta_0}), \lambda=4+\sqrt{17}
\end{equation}
\end{theorem}

\begin{proof}

Given the $\beta$-smoothness condition, we have $||\nabla U(\mat{x})-\nabla U(\mat{x}^*)||_2\leq\beta||\mat{x}-\mat{x}^*||_2$ thus $||\nabla U(\mat{x})||\leq \beta ||\mat{x}||_2$. The sequences $\{\mat{x}_n^*\}$ and $\{\mat{x}_n\}$ are obtained as
\begin{equation}\label{eq:25}
    \begin{split}
    \mat{x}_{n+2}^*&=-\mat{C}\mat{x}_{n+1}^*-\mat{x}_n^*,\\
    \mat{x}_{n+2}&=-(\epsilon^2\mat{M}^{-1}\nabla U(\mat{x}_{n+1})-2\mat{x}_{n+1})-\mat{x}_n.
    \end{split}
\end{equation}

Taking the difference between the above two equations we get 
\begin{equation}\label{eq:26}
\begin{split}
    \mat{x}_{n+2}&=(\mat{C}\mat{x}_{n+1}^*+\mat{x}_n^*+\mat{x}_{n+2}^*)-(\epsilon^2\mat{M}^{-1}\nabla U(\mat{x}_{n+1})-2\mat{x}_{n+1})-\mat{x}_n.
\end{split}
\end{equation}
Denote $\Delta_n=||\mat{x}_n||_2$ we have
\begin{equation}\label{eq:27}
\begin{split}
    \Delta_{n+2}=||\mat{x}_{n+2}||_2&\leq ||\mat{C}\mat{x}_{n+1}^*+\mat{x}_n^*+\mat{x}_{n+2}^*||_2+\epsilon^2||\mat{M}^{-1}||_2\cdot||\nabla U(\mat{x}_{n+1})||_2+2||\mat{x}_{n+1}||_2+||\mat{x}_n||_2\\
    &\leq (\lambda_{max}(\mat{C})+2)\Delta_0+(\frac{\epsilon^2\beta}{m_{min}}+2)\Delta_{n+1}+\Delta_{n}
    < 4\Delta_0+8\Delta_{n+1}+\Delta_{n}.
\end{split}
\end{equation}

In the extreme case, $\Delta_n$ grows exponentially as $\lambda^n$ where $\lambda=4+\sqrt{17}$. Therefore, there should exist a constant $A$ which only depends on the initial conditions, satisfying that 
\begin{equation}
    \Delta_n<A\Delta_0\lambda^n
\end{equation}
By enforcing $A\Delta_0\lambda^n<\Delta$ for all $n<N_0$, we immediately have Eq.~(\ref{eq:N0}).
\end{proof}

Following Theorem \ref{theorem:5}, now we provide a necessary condition for efficiently sampling the posterior density associated with a smooth energy function. 
\begin{corollary}
\label{corrollary:smooth}
A sampler can explore the smooth energy function efficiently, if it can find a path from $||\mat{x}_0||_2\leq\Delta_0$ to $||\mat{x}_i||_2\geq\Delta$ with no more than $N_0$ steps. A necessary
condition is 
\begin{equation}
    m_{min}\leq\frac{\beta\epsilon^2}{6}
\end{equation}

This implies that the sampler cannot explore the energy function efficiently if the mass is above a certain bound, which is linear with respect to $\beta$.
\end{corollary}

\subsection{Time-Varying Mass for Handling Spiky Distribution}\label{sec:safe-sampling}

Next we explain how our proposed choice of mass (the log-normal parameterization) can help sample a spiky distribution. Here ``spiky" means a distribution whose energy function $U(\mat{x})$ has a very large (or even infinite) gradient around some points. Fig.~\ref{fig:lp_illus} left shows an energy function $U(x)$ with a large gradient at $x=0$.  A representative family of spiky distributions is ${\rm exp} \left(-\lambda \| \mat{x}\|_p^p\right)$ with $p \in (0,1]$, which is widely used as a prior density for sparse modeling and model selection~\citep{armagan2009variational,huang2008asymptotic,donoho2006compressed,xu20101,zhao2014p,huang2008asymptotic}. Here the $\ell_p$ norm is defined in a loose sense as $||\mat{x}||_p=(\sum_{i=1}^d |x_i|^{p})^{\frac{1}{p}}$. As shown in Fig. \ref{fig:lp_illus} middle, $\| \mat{x}\|_p^p$ is non-convex and has divergent gradients around $x=0$ when $0<p<1$, causing troubles in traditional HMC samplers. Sampling from such a spiky distribution is a challenging task in Bayesian learning~\citep{chaari2016hamiltonian,chaari17nshmc}. 

The Hamiltonian dynamical system (\ref{eq:hd}) becomes stiff if the energy function $U(\mat{x})$ has a large gradient at some points. This issue can cause unstable numerical simulations. A naive idea is to adapt the step size based on the local gradient $\nabla U(\mat{x})$. Unfortunately, this explicit step-size tuning is equivalent to using a state-dependent time modifier $a(\mat{z}, t)$, and will produce a wrong time-averaged distribution $p_t(\mat{x})\neq p(\mat{x}|{\cal D})$, as proved in Theorem~\ref{thm:time-average}. In contrast, the implicit and stochastic mass adaptation in our QHMC is equivalent to using a state-independent time modifier $a(t)$, and it ensures the produced time-averaged distribution converging to the desired posterior density $p(\mat{x} | {\cal D})$. 

\begin{figure}[t]
    \centering
    \begin{subfigure}[b]{0.64\textwidth}
            \centering
    \includegraphics[trim = 10mm 95mm 10mm 102mm, clip,width=1.0\textwidth]{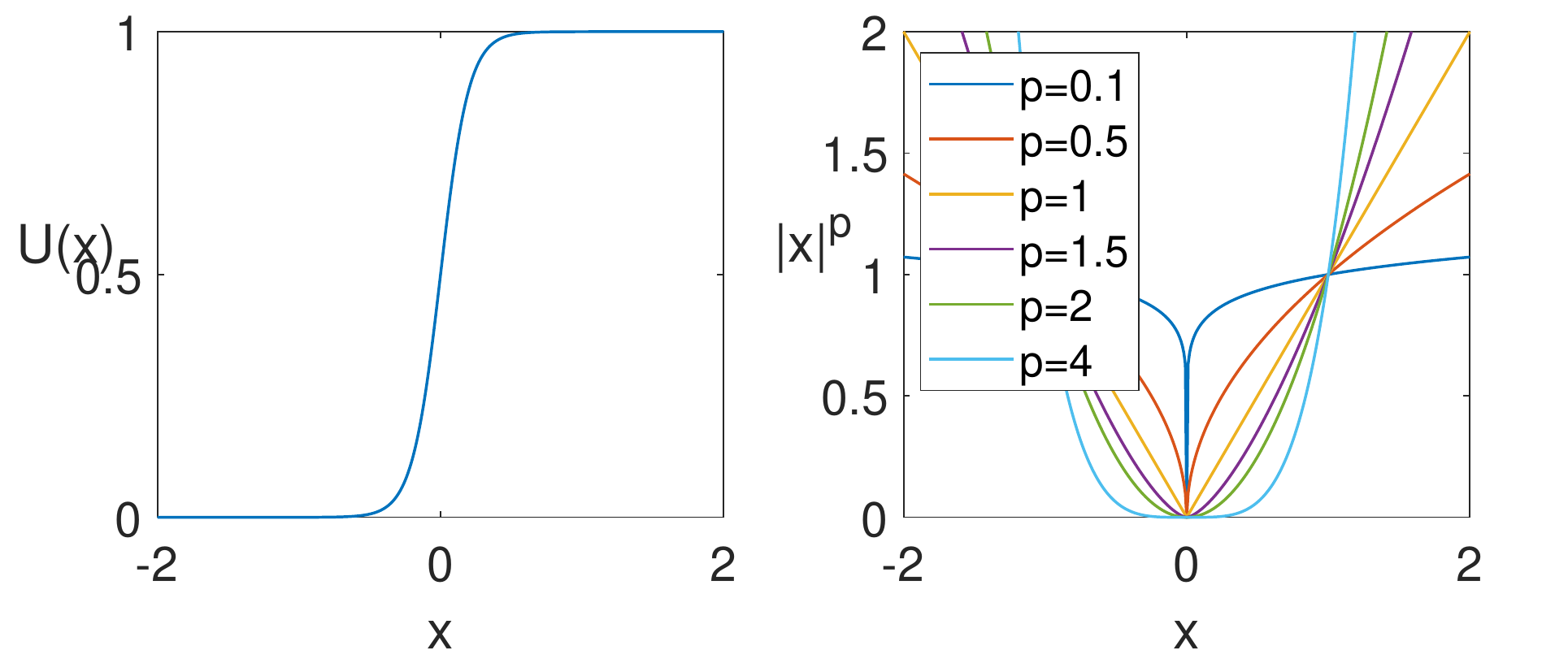}
    \end{subfigure}
    \begin{subfigure}[b]{0.32\textwidth}
            \centering
    \includegraphics[trim = 35mm 85mm 40mm 80mm, clip,width=1.0\textwidth]{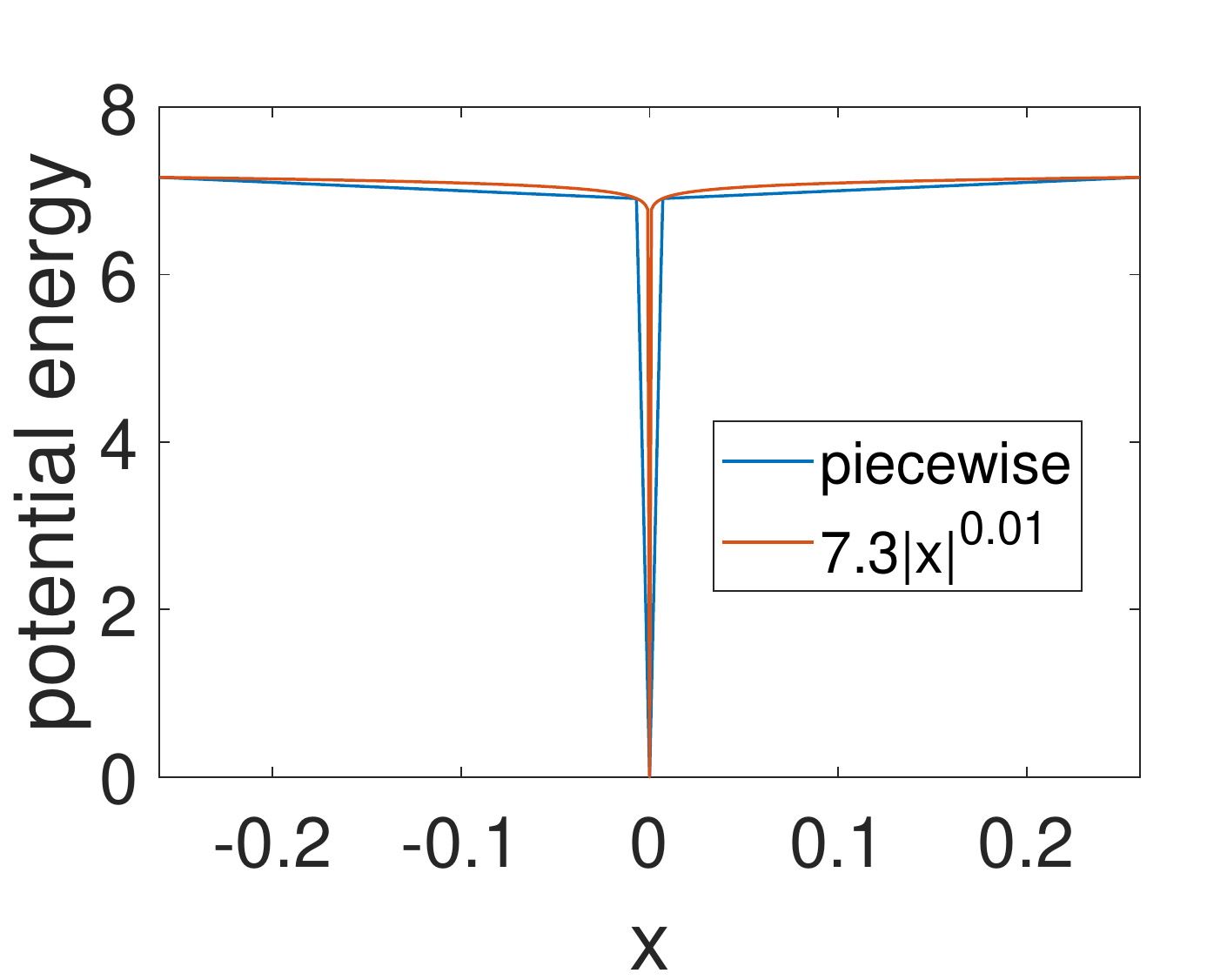}
    \end{subfigure}
    
    \caption{Left: a spiky potential energy function with a very large gradient around $x=0$. Middle: $\|x\|_p^p$ for various values of $p$. When $p$ tends to zero, $\|x\|_p$ becomes an indicator function, which is 0 for $x=0$ and 1 elsewhere. The gradients tend to infinity around $x=0$ for $0<p<1$. Right: the spiky potential function $7.3|x|^{0.01}$ can be approximated by the piecewise linear function in Eq.~(\ref{eq:x-and-1000x}).}
    \label{fig:lp_illus}
    \vspace{-10pt}
\end{figure}

\begin{figure}[t]
    \centering
    \begin{subfigure}[b]{0.475\textwidth}
            \centering
            \includegraphics[trim = 35mm 85mm 30mm 80mm, clip,width=\textwidth]{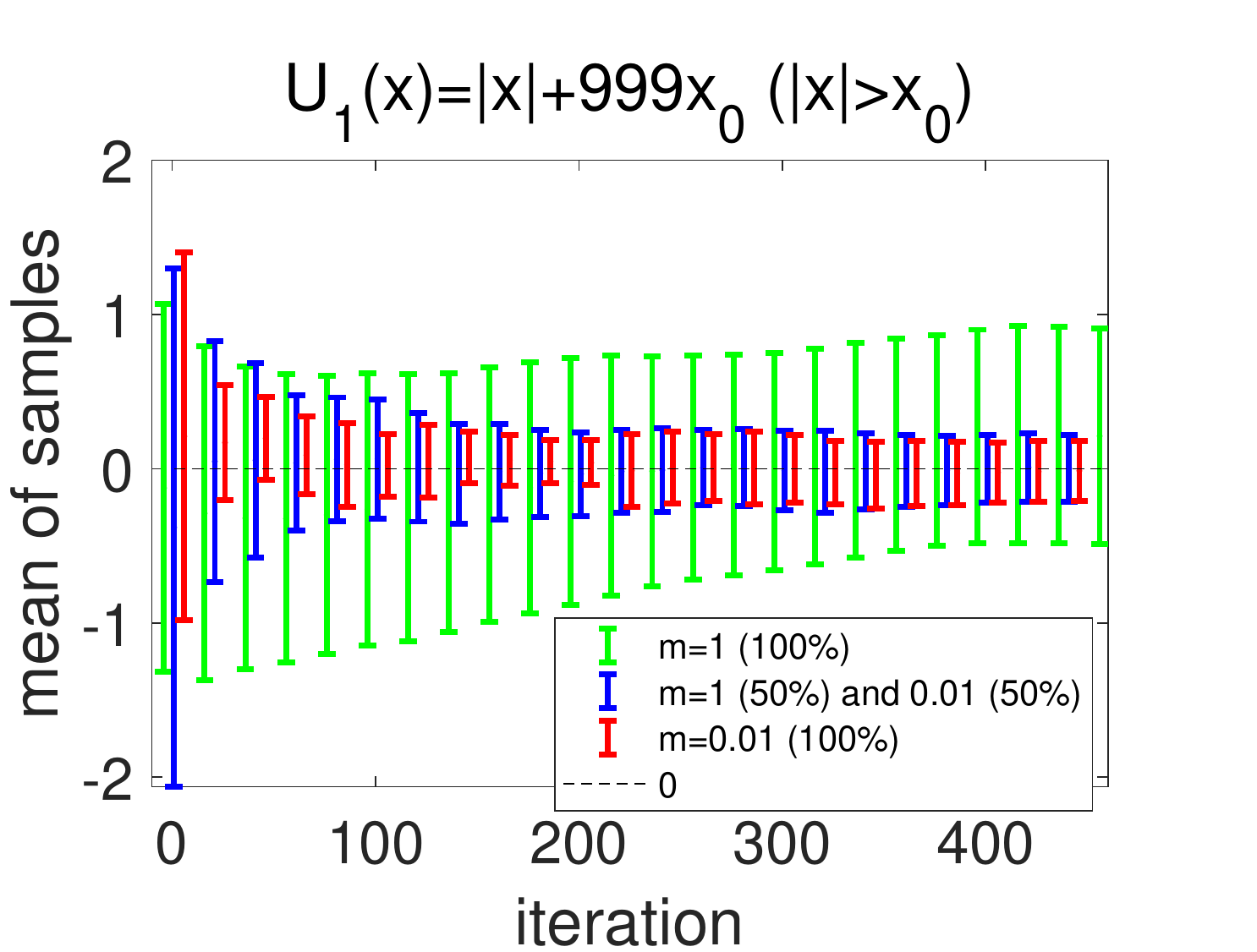}
            {{\small }}
        \end{subfigure}
        \hfill
        \begin{subfigure}[b]{0.475\textwidth}  
            \centering 
            \includegraphics[trim = 35mm 85mm 30mm 80mm, clip,width=\textwidth]{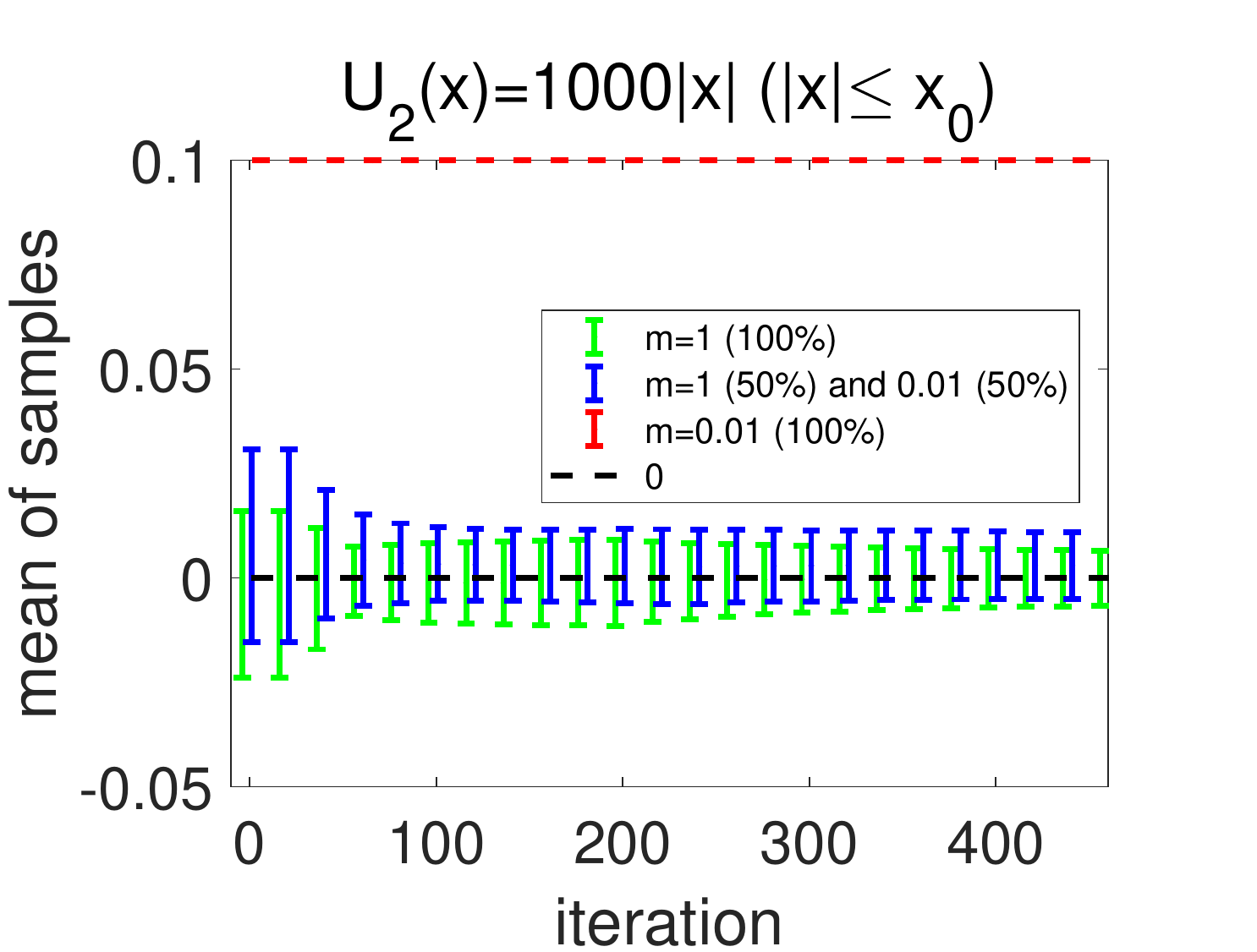}
            {{\small }}   
        \end{subfigure}
    \caption{The performance of HMC and QHMC. QHMC uses a random mass with distribution $p_m(m)=\frac{1}{2}(\delta(m-1)+\delta(m-0.01))$. (a) Result for the smooth function $U_1(x)=|x|+999 x_0$ where QHMC behaves similarly to HMC with $m=1$ and significantly outperforms HMC with $m=0.01$. (b) Results for the spiky function $U_2(x)=1000|x|$, where QHMC behaves similarly to HMC with $m=0.01$ and significantly outperforms HMC with $m=1$.}
    \label{fig:implicit-adaptation-benefit}
    \vspace{-5pt}
\end{figure}

The advantage of the implicit adaptation strategy in QHMC can be illustrated with a toy spiky distribution with the following potential energy function $U(x)$:
\begin{equation}\label{eq:x-and-1000x}
U(x)=
\left\{
             \begin{array}{lc}
             U_1(x)=|x|+999x_0 \quad & {\rm if}\; |x|>x_0,\\
             U_2(x)=1000|x| \quad & {\rm if} \; |x|\leq x_0.
             
             \end{array}
\right.
\end{equation}
Here $U(x)$ can be a simple yet good approximation to an $\ell_p$ function, as shown on the right of Fig.~\ref{fig:lp_illus}. The spiky and smooth region of $U(x)$ is separated by $-x_0$ and $x_0$. We care if both regions can be sampled accurately. When $x_0={\rm log}(1001)/1000$, the probability of $x$ being located in either region is 0.5, therefore $U_1(x)$ and $U_2(x)$ are equally important for sampling and we study them independently. Because $U_1(x)$ and $U_2(x)$ are both symmetric, their associated mean values should be zero, and we measure the numerical performance of a sampler by estimating the error bars obtained from 20 independent experiments. Suppose that we have two mass choices in HMC: $m=1$ or $m=0.01$. In QHMC, we adapt the mass implicitly by allowing $m=1$ and $m=0.01$ with an equal probability. Because $U_1(x)=|x|+999x_0\ (|x|>x_0)$ is very smooth, a small mass $m=0.01$ is preferred, as shown in Fig.~\ref{fig:implicit-adaptation-benefit} (a). Although QHMC slightly underperforms the HMC implemented with $m=0.01$, it significantly outperforms the HMC implemented with $m=1$. On the other hand, $U_2(x)=1000|x|$ has a very large gradient, therefore a relatively large mass $m=1$ is preferred, as shown in Fig.~\ref{fig:implicit-adaptation-benefit} (b). Similarly, QHMC has slightly worse performance than the HMC implemented with $m=1$, but it performs much better than the HMC implemented with $m=0.01$. In summary, the HMC cannot explore efficiently $U_1$ and $U_2$ simultaneously with a fixed mass, however our QHMC with an implicit mass adaptation can have excellent performance in both regions.

\textbf{Remark: Adapting Mass in an Exponential Scale.} Corollary~\ref{corrollary:smooth} implies that one needs to choose a small mass $m\leq \frac{\beta\epsilon^2}{6}$ to efficiently explore a potential energy of $\beta$-smoothness. Although $\ell_p (0<p<1)$ function has no global $\beta$-smoothness , one can define $\beta$ locally. In the example above, for $x>x_0$ or $x<-x_0$, $\beta=1$; for $-x_0<x<x_0$, $\beta=1000$. The wide spread of $\beta$ from 1 to 1000 justifies our choice of a widespread mass distribution as ${\rm log} m\sim {\cal N}(\mu_m,\sigma_m^2)$ in Alg.~\ref{alg:qhmc}. In order to adapt to widespread $\beta$, we should adapt $m$ in an exponential scale rather than in a linear scale.

\subsection{Explore Multimodal Energy Functions: Quantum Tunneling Effects}\label{sec:quantum-tunneling}
Finally, we show that the random mass distribution can help sample a multimodal distribution because it can manifest the ``quantum tunneling" effect. 

\begin{figure*}[t]
            \centering        \includegraphics[width=1\textwidth]{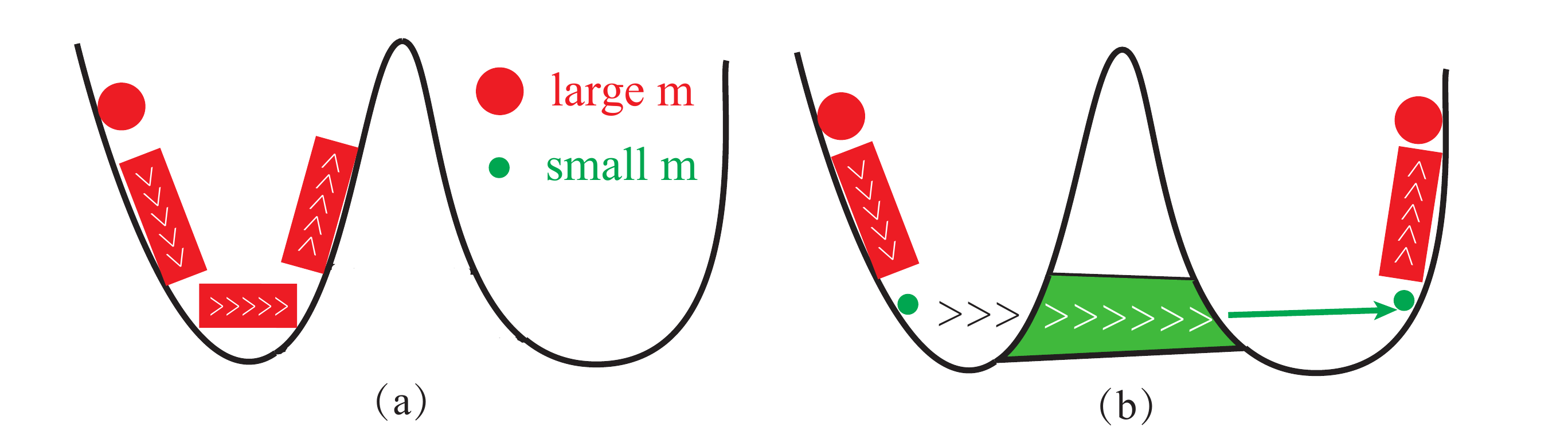}
            \caption[Network2]%
            {{\small The ``quantum tunneling" effects for QHMC. (a) In HMC, a relatively large mass (equivalently, small step size) is chosen to guarantee accurate Hamiltonian simulations. However, the particle can be trapped within one well. (b) In QHMC, the mass is random, and there is chance that one uses a small mass in one path. The light particle has a higher chance to jump to another well through ``tunneling".}}    
        \label{fig: small mass}
        \vspace{-15pt}
    \end{figure*}

The ``quantum tunneling" effect describes that a microscopic particle may climb over a potential peak even if its total energy is low, which is very different from the case in classical mechanics. In quantum mechanics, a particle should be treated as wave permeable in the whole space rather than a localized object. As a result, the particle (the wave) always has a non-zero probability to climb over the peak. The larger step the particle takes, the more likely a quantum tunneling will happen. 

Now we analyze the quantum tunneling effect of our QHMC method in a semi-quantitative way. When a particle has momentum uncertainty $\Delta q$, it should also have a real space uncertainty $\Delta x$ in quantum mechanics such at $\Delta x\Delta q\sim \hbar$, where $\hbar \approx 1.05\times 10^{-34}$ Js is the Plank's constant and represents the unit of time in the quantum world. In QHMC, the step size $\epsilon$ plays the role of time unit, so we instead have $\Delta x\Delta q\sim\epsilon$. The momentum variable $q$ has a distribution $\propto {\rm exp}(-q^2/2m)$ at the thermal equilibrium point, therefore we have momentum uncertainty $\Delta q\sim \sqrt{m}$ and position uncertainty $\Delta x\sim \epsilon/\Delta q=\epsilon/\sqrt{m}$. When $m$ is small the particle is less likely to be trapped in a single well. Fig.~\ref{fig: small mass} (b) shows the intuition of the quantum tunnelling effect.

\section{Stochastic-Gradient Implementation}\label{sec:SQHMC}

\subsection{Quantum Stochastic Gradient Nos{\'e}-Hoover Thermostat (QSGNHT)}

Similar to HMC, the proposed QHMC method suffers from a high computational cost when the training data size is huge. Consider a machine learning problem with $N$ training samples, the loss (energy) function $U(\mat{x})$ is commonly defined as the average loss over all training samples:  $U(\mat{x})=\frac{1}{N}\sum_{i=1}^NU_i(\mat{x})$, where $U_i (\mat{x})$ depends only on the $i$-th training sample. Calculating the full gradient $\nabla U(\mat{x})$ needs computation over every training sample. Instead, one may replace the true loss with a stochastic estimation $\tilde{U}(\mat{x})$, and the stochastic gradient is computed efficiently with only a small batch of samples $\nabla \tilde{U}(\mat{x})=\frac{1}{b}\sum_{i=1}^b \nabla U_i(\mat{x})$. Here $b$ is called the batch size. However, the mini-batch estimation of gradients will introduce extra noise. According to the central limit theorem, we can approximate the stochastic gradient as the true gradient plus a Gaussian noise with covariance $V(\mat{x})$: $\nabla\tilde{U}(\mat{x})\approx \nabla U(\mat{x})+{\cal N}(0,\mat{V}(\mat{x}))$, if $b$ is much smaller than $N$ but still relatively large.

However, such a naive stochastic gradient implementation can result in incorrect steady distribution, and one can add a friction term to compensate for the extra noise in a stochastic-gradient HMC~\citep{chen2014stochastic}. Different from the friction formulation in~\citep{chen2014stochastic}, we utilize the thermostat technique~\citep{ding2014bayesian,Leimkuhler2009AMA} to correct the steady distribution. Specifically, we treat the gradient uncertainty term $\mat{V}(\mat{x})$ as a noise with an unknown magnitude, and use the Nos{\'e}-Hoover thermostat to avoid the explicit estimation of this gradient noise term.  The resulting update rule for $(\mat{x},\mat{q},\xi)$ is shown in Eq.~(\ref{eq:thermostat}) with step size $\epsilon$:
\begin{equation}\label{eq:thermostat}
\left\{
\begin{aligned}
&\mat{x}_{i+1}\gets\ \mat{x}_i+\epsilon \mat{M}_t^{-1}\mat{q}_i\quad \\
&\mat{q}_{i+1}\gets\ \mat{q}_i -\epsilon\nabla \tilde{U}(\mat{x}_{i+1})-\epsilon\xi_i \mat{q}_i+\sqrt{2A}{\cal N}(0,\epsilon\mat{I})\\
&\xi_{i+1}\gets\ \xi_i+\frac{\epsilon}{m_\mu}(\mat{q}_{i+1}^T\mat{M}_t^{-1}\mat{q}_{i+1}-{\rm Tr}(\mat{M}_t^{-1}))
\end{aligned}
\right.
\end{equation}
where $\mat{M}_t$ refers to the mass matrix used for the $t$-th path, $i=1,2,\cdots,L$ refers to the index of steps in each path. $A$ indicates the magnitude of injected noise, $m_\mu$ is the thermal mass term, $d_0$ is the dimension of $\mat{x}$, and $T$ is in the definition of the energy function $U(x)=-T{\rm log}(p(\mat{x}|{\cal D}))$. We set $m_\mu=1$ (like in ~\citep{ma2015complete}), $A=1$ and $T=1$. We refer to the proposed method with thermostat as quantum stochastic gradient Nos{\'e}-Hoover thermostat (QSGNHT). The algorithm flow of QSGNHT is shown in Alg.~\ref{alg:sqhmc}.

\begin{algorithm}[t]
\SetAlgoLined
\KwResult{starting point $x_0$, step size $\epsilon$, simulation steps $L$, mass distribution parameters $\mu_m$ and $\sigma_m$, thermal mass $m_\mu$, batch size $b$, 
dimension of $\mat{x}$ is $d$, temperature $T$, diffusion strength $A$.}
 \For{$t=1,2,\cdots$}{
    Randomly select $b$ samples out of the $N$ training samples and compute $\tilde{U}(\mat{x})=\frac{1}{b}\sum_{i=1}^b U_i(\mat{x})$\; 
    Resample $\mat{M}_t\sim P_{\mat{M}}(\mat{M})$ , and resample $\mat{q}\sim {\cal N}(0,\mat{M}_t)$\;
    ($\mat{x}_0,\mat{q}_0$)=($\mat{x}^{(t)},\mat{q}^{(t)}$)\;
    Simulate dynamics based on Eq.~(\ref{eq:thermostat})\;
    $\mat{q}_0\gets \mat{q}_0-\frac{\epsilon}{2}\nabla \tilde{U}(\mat{x}_0)$\;
    \For{$i=1,\cdots,L-1$}{
    $\mat{x}_{i+1}\gets\ \mat{x}_i+\epsilon \mat{M}_t^{-1}\mat{q}_i$\;
    $\mat{q}_{i+1}\gets\ \mat{q}_i -\epsilon\nabla \tilde{U}(\mat{x}_{i+1})-\epsilon\xi_i \mat{q}_i+\sqrt{2A}{\cal N}(0,\epsilon\mat{I})$\;
    $\xi_{i+1}\gets\ \xi_i+\frac{\epsilon}{m_\mu}(\mat{q}_{i+1}^T\mat{M}_t^{-1}\mat{q}_{i+1}-{\rm Tr}(\mat{M}_t^{-1}))$\;}
    $\mat{q}_L\gets \mat{q}_L-\frac{\epsilon}{2}\nabla \tilde{U}(\mat{x}_L)$\;
    ($\hat{\mat{x}},\hat{\mat{q}}$)=($\mat{x}_L,\mat{q}_L$)\;
    M-H step: $u\sim\mathrm{Uniform[0,1]}$\;
    Define $\tilde{H}(\mat{x},\mat{q})=\tilde{U}(\mat{x})+\frac{1}{2}\mat{q}^T\mat{M}_i^{-1}\mat{q}$\;
    $\rho = e^{-\tilde{H}(\hat{\mat{x}},\hat{\mat{q}})+\tilde{H}(\mat{x}^{(t)},\mat{q}^{(t)})}$\;
    \eIf{$u<\mathrm{min}(1,\rho)$}{$(\mat{x}^{(t+1)},\mat{q}^{(t+1)})=(\hat{\mat{x}},\hat{\mat{q}})$}{$(\mat{x}^{(t+1)},\mat{q}^{(t+1)})=(\mat{x}^{(t)},\mat{q}^{(t)})$}
    }
    \KwResults{ $\{\mat{x}^{(1)},\mat{x}^{(2)},\cdots\}$}
 \caption{Quantum Stochastic Gradient Nos{\'e}-Hoover Thermostat (QSGNHT)}
 \label{alg:sqhmc}
\end{algorithm}

\subsection{Theoretical Analysis based on Stochastic Differential Equation (SDE)}

In this subsection, we prove that the QSGNHT implementation in Alg.~\ref{alg:sqhmc} indeed produces the desired posterior density $p(\mat{x} | {\cal D})$. Setting $\mat{M}(t)=m(t)\mat{I}$, we first provide the continuous-time stochastic differential equation (SDE) for QSGNHT:
\begin{equation}\label{sde_qsgnht}
    d\begin{pmatrix}
    \mat{x}\\\mat{q}\\\xi
    \end{pmatrix}=dt
    \begin{pmatrix}
    \mat{M}(t)^{-1}\mat{q}\\
    -\nabla U(\mat{x})-\xi\mat{q}+{\cal N}(0,\mat{V}(\mat{x}))\epsilon\\
    \frac{1}{m_\mu }(\mat{q}^T\mat{M}(t)^{-1}\mat{q}-{\rm Tr}(\mat{M}(t)^{-1}))
    \end{pmatrix}+
    \begin{pmatrix}
    0\\
    \sqrt{2A}d\mat{W}(t)\\
    0
    \end{pmatrix}
\end{equation}
where $\epsilon$ is the step size in the discretized dynamics. Before presenting our result in Theorem \ref{thm:qsgnht_steady_state}, we review the result for a general continuous-time Markov process in Lemma \ref{lemma:ma-paper}.

\begin{lemma}
\label{lemma:ma-paper}
(Theorem 1 in~\citep{ma2015complete}) 
A general continuous-time Markov process can be written as a stochastic differential equation (SDE) in this form:

\begin{equation}\label{eq:general_sde}
    d\mat{z} = \mat{f}(\mat{z})dt+\sqrt{2\mat{D}(\mat{z})}d\mat{W}(t)
\end{equation}
where $\mat{z}$ can be a general vector and $\mat{D}(\mat{z})$ represents the magnitude of the Wiener diffusion process. Then, $p_s(\mat{z})\propto{\rm exp}(-H(\mat{z}))$ is a steady distribution of the above SDE if $f(\mat{z})$ can be written as:
\begin{equation}\label{eq:f(z)}
    \mat{f}(\mat{z})=-(\mat{D}(\mat{z})+\mat{Q}(\mat{z}))\nabla H(\mat{z})+\mat{\Gamma} (\mat{z}),\quad \mat{\Gamma}_i(\mat{z})=\sum_{i=1}^d\frac{\partial}{\partial \mat{z}_j}(\mat{D}_{ij}(\mat{z})+\mat{Q}_{ij}(\mat{z}))
\end{equation}
where $H(\mat{z})=U(\mat{x})+g(\mat{x},\mat{q})$ is the Hamiltonian of the system, $U(\mat{x})=-{\rm log} p(\mat{x}|{\cal D})$ is the potential energy, $g(\mat{x},\mat{q})=\frac{1}{2}\mat{q}^T \mat{M}(t)^{-1}\mat{q}/2$ is the kinetic energy, $\mat{Q}(z)$ determines the deterministic transverse dynamics, $\mat{D}(\mat{z})$ is positive semidefinite, and $\mat{Q}(\mat{z})$ skew-symmetric. The steady distribution $p_s(\mat{z})$ is unique if $\mat{D}(\mat{z})$ is positive definite, or if ergodicity~\footnote{The ergodicity of a Markov process requires the coexistence of irreducibility, aperiodicity and positive recurrence. Intuitively, ergodicity means that every point in the state space can be hit within finite time with probability one.} can be shown.
\end{lemma}
\begin{proof}
The proof is based on~\citep{ma2015complete}. According to Eq.~(\ref{eq:general_sde}) we have a corresponding Fokker-Planck equation to describe the evolution of the probability density:
\begin{equation}\label{eq:fokker-planck}
    \partial_t p(\mat{z},t)=-\sum_i\frac{\partial}{\partial \mat{z}_i}(\mat{f}_i(\mat{z})p(\mat{z},t))+\sum_{i,j} \frac{\partial^2}{\partial \mat{z}_i\partial \mat{z}_j}(\mat{D}_{ij}(\mat{z})p(\mat{z},t))
\end{equation}
Eq. (\ref{eq:fokker-planck}) can be further written in a more compact form:
\begin{equation}
    \partial_t p(\mat{z},t)=\nabla^T\cdot([\mat{D}(\mat{z})+\mat{Q}(\mat{z})][p(\mat{z},t)\nabla H(\mat{z})+\nabla p(\mat{z},t)])
\end{equation}
We are able to verify $p_s(\mat{x},\mat{q},\xi)\propto{\rm exp}(-H(\mat{x},\mat{q},\xi))$ is invariant under Eq. (\ref{eq:fokker-planck}) by calculating $[e^{-H(\mat{z})}\nabla H(\mat{z})+\nabla e^{-H(\mat{z})}]=0$. If the process if ergodic, then the stationary distribution is unique. The ergodicity of the Markov process requires three conditions: (a) irreducibility (b) aperiodicity (c) positive recurrence. Irreducibility and aperiodicity can be guaranteed by non-zero diffusion noises, while positive recurrence is satisfied if $U(\mat{x})\to\infty$ when $|\mat{x}|\to\infty$.
\end{proof}

\begin{lemma}\label{lemma:qsgnht-M0}
Consider a system described by Eq.~(\ref{sde_qsgnht}) but with deterministic constant mass $\mat{M}(t)=\mat{M}_0$, then the steady distribution is unique and proportional to ${\rm exp}(-U(\mat{x}))$.
\end{lemma}

\begin{proof}
 Compare Eq.~(\ref{sde_qsgnht}) with Eq.~(\ref{eq:general_sde}) and (\ref{eq:f(z)}) we can determine the corresponding Hamiltonian and coefficient matrices $\mat{D}$ and $\mat{Q}$:
\begin{equation}
    \begin{aligned}
    H(\mat{x},\mat{q},\xi)=&U(\mat{x})+\frac{1}{2}\mat{q}^T\mat{M}_0^{-1}\mat{q}+\frac{m_\mu}{2}(\xi-A-\frac{V(\mat{x})\epsilon}{2})^2\\
    \mat{D}(\mat{x},\mat{q},\xi)=
    \begin{pmatrix}
    0 & 0 & 0\\
    0 & A+\frac{V(\mat{x})\epsilon}{2} & 0\\
    0 & 0 & 0\\
    \end{pmatrix}, & \quad
    \mat{Q}(\mat{x},\mat{q},\xi)=
    \begin{pmatrix}
    0 & -\mat{I} & 0\\
    \mat{I} & 0 & \mat{M}_0^{-1}\mat{q}/m_\mu\\
    0 & -\mat{q}^T\mat{M}_0^{-1}/m_\mu & 0\\
    \end{pmatrix}.
    \end{aligned}
\end{equation}
Obviously $\mat{D}(\mat{x},\mat{q},\xi)$ is positive semi-definite and $\mat{Q}(\mat{x},\mat{q},\xi)$ is skew-symmetric. By denoting $\mat{z}=(\mat{x},\mat{q},\xi)$, we know that there exists a steady distribution $p_s(\mat{z})\propto{\rm exp}(-H(\mat{z}))={\rm exp}(-U(\mat{x})-\frac{1}{2}\mat{q}^T\mat{M}_0^{-1}\mat{q}-\frac{1}{2}m_\mu(\xi-A-\frac{V(\mat{x})\epsilon}{2})^2)$. Due to non-zero diffusion error $A>0$ in the system, $p_s(\mat{x},\mat{q},\xi)$ is the unique steady distribution. Marginalizing over $\mat{q}$ and $\xi$, one obtains $p_s(\mat{x})\propto {\rm exp}(-U(\mat{x}))$.
\end{proof}

\begin{figure}[t]
        \centering
        \begin{subfigure}[b]{0.32\textwidth}
            \centering
            \includegraphics[trim = 55mm 102mm 60mm 102mm, clip,width=\textwidth]{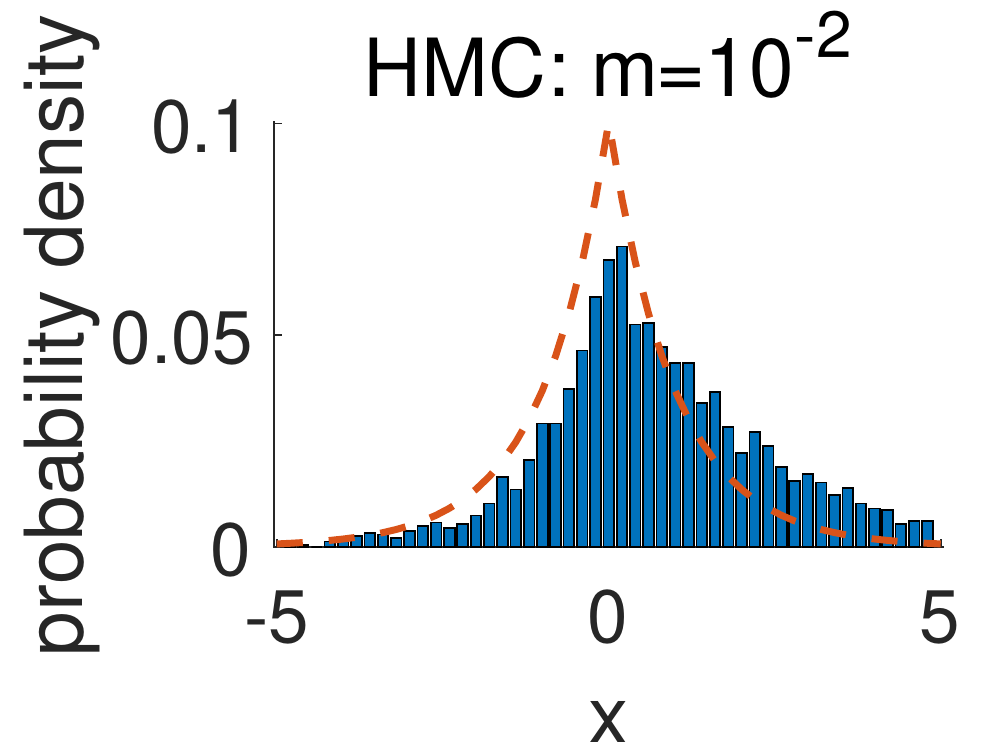}
             \caption[]%
            {{\small}}    
        \end{subfigure}
        \hfill
        \begin{subfigure}[b]{0.32\textwidth}  
            \centering 
            \includegraphics[trim = 55mm 102mm 60mm 102mm, clip,width=\textwidth]{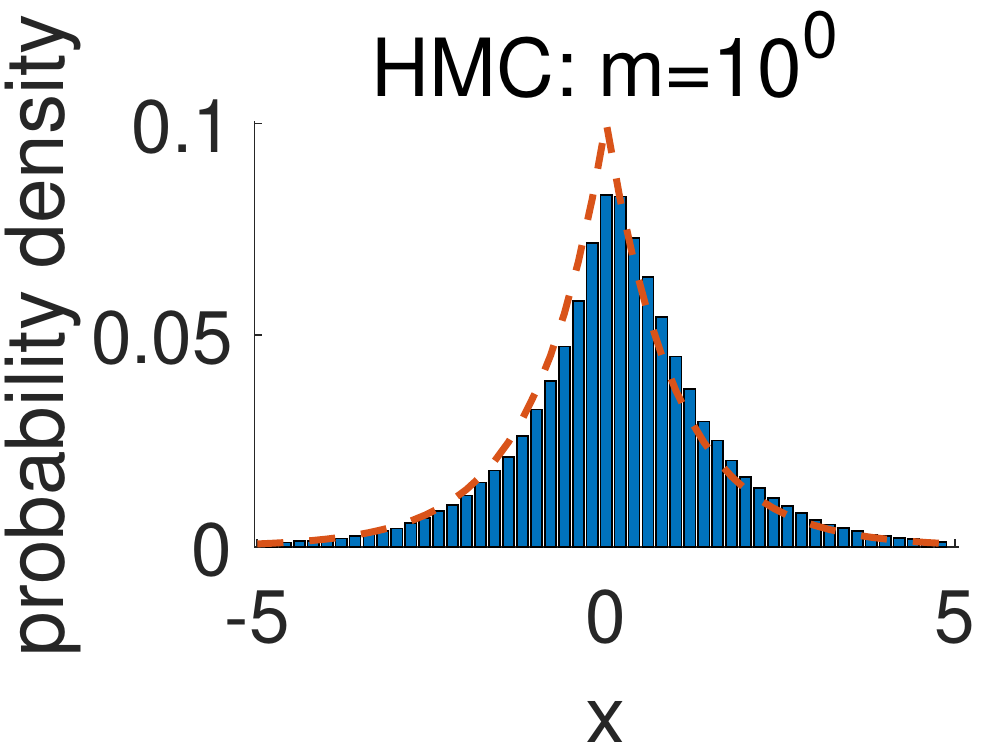}
             \caption[]%
            {{\small}}    
        \end{subfigure}
        \hfill
        \begin{subfigure}[b]{0.32\textwidth}  
            \centering 
            \includegraphics[trim = 55mm 102mm 60mm 102mm, clip,width=\textwidth]{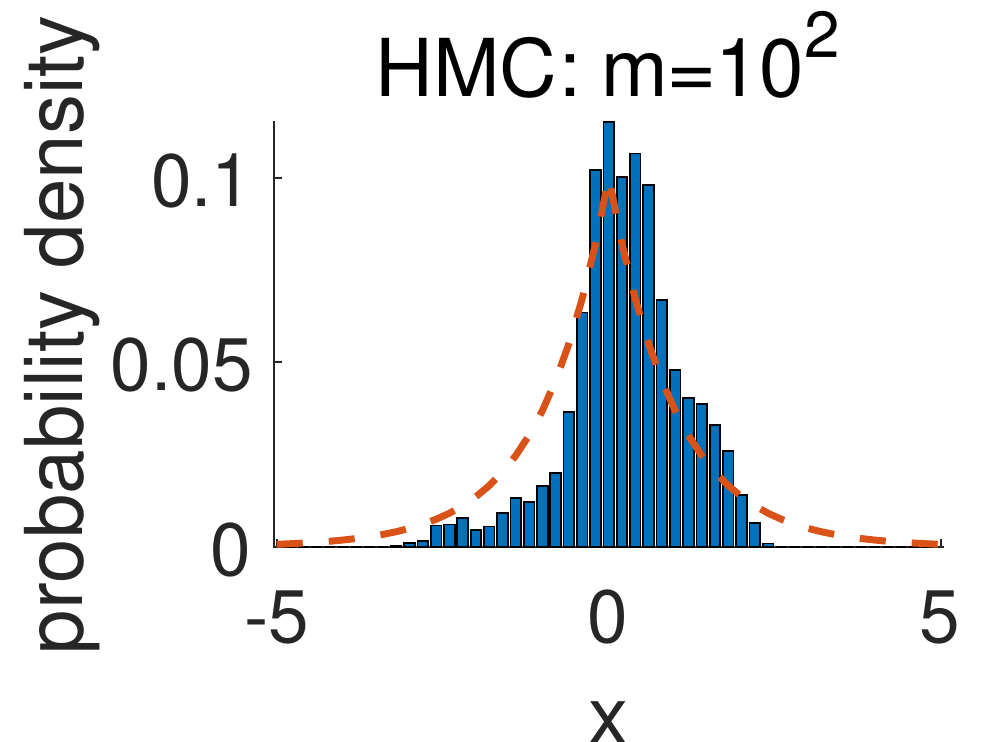}
             \caption[]%
            {{\small}}    
        \end{subfigure}
        \vskip\baselineskip
        \begin{subfigure}[b]{0.32\textwidth}   
            \centering 
            \includegraphics[trim = 55mm 102mm 60mm 102mm, clip,width=\textwidth]{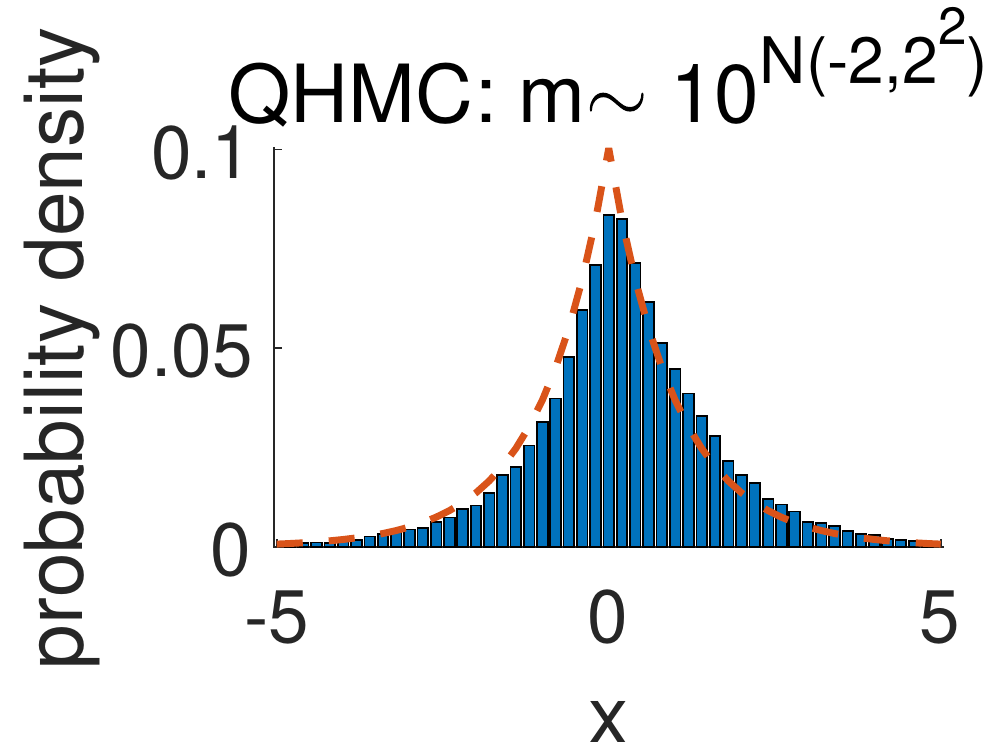}
             \caption[]%
            {{\small}}    
        \end{subfigure}
        \hfill
        \begin{subfigure}[b]{0.32\textwidth}   
            \centering 
            \includegraphics[trim = 55mm 102mm 60mm 102mm, clip,width=\textwidth]{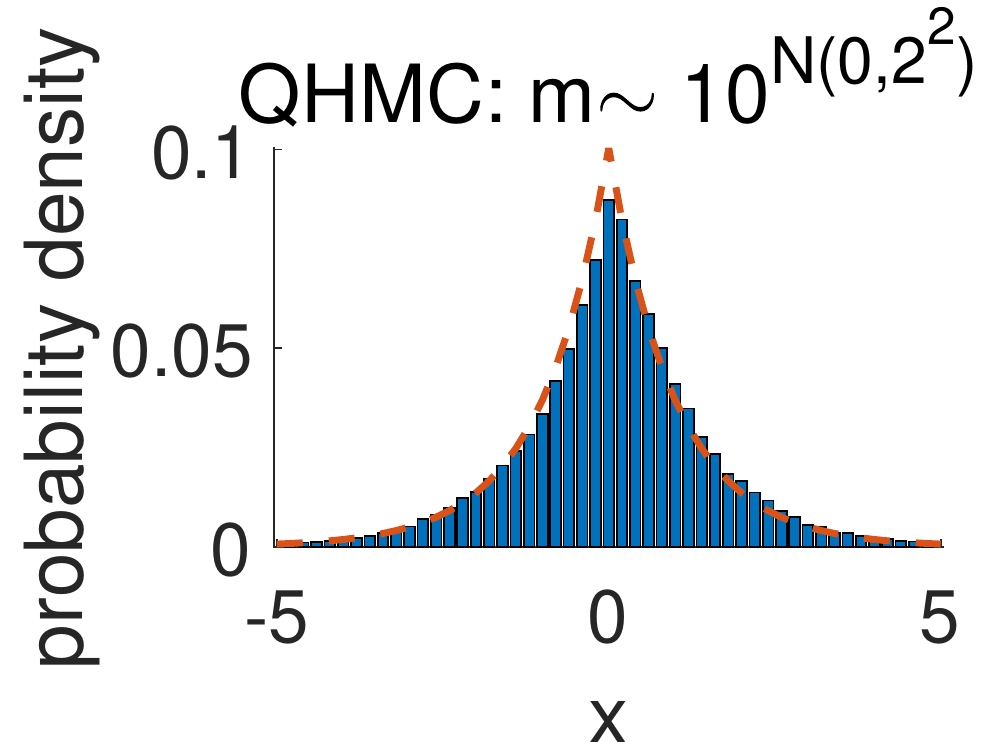}
             \caption[]%
            {{\small}}    
        \end{subfigure}
        \hfill
        \begin{subfigure}[b]{0.32\textwidth}   
            \centering 
            \includegraphics[trim = 55mm 102mm 60mm 102mm, clip,width=\textwidth]{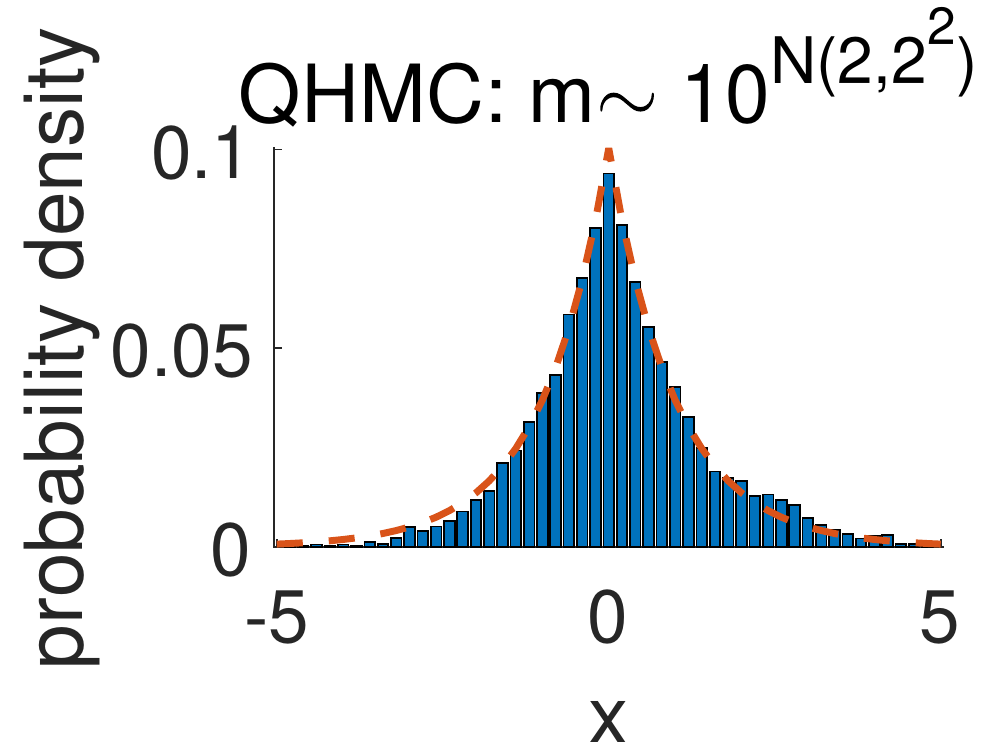}
             \caption[]%
            {{\small}}    
        \end{subfigure}
        \caption[ The average and standard deviation of critical parameters ]
        {\small The results of HMC and QHMC for the energy function $U(x)=|x|$. The red dashed line: true distribution; blue histogram: the probability obtained with 200000 simulation samples. (a)-(c) HMC is sensitive to the choice of mass; (d)-(f) QHMC works well for a large range of mass parameters.}
        \label{fig: lp_1}
        \vspace{-10pt}
    \end{figure}
    
Similar to the second proof of Theorem \ref{thm:qhmc-sd}, we have Theorem \ref{thm:qsgnht_steady_state} stated below.

\begin{theorem}\label{thm:qsgnht_steady_state}
The Markov process in Eq.~(\ref{sde_qsgnht}) has a unique marginal steady distribution $p_s(\mat{x})\propto {\rm exp}(-U(\mat{x}))$.
\end{theorem}

\begin{proof}
From lemma \ref{lemma:qsgnht-M0} we know that given a constant mass matrix $\mat{M}$, we have
\begin{equation}
    p_s(\mat{x},\mat{q},\xi|\mat{M})\propto {\rm exp}(-U(\mat{x})){\rm exp}(-\frac{1}{2}\mat{p}^T\mat{M}^{-1}\mat{p}){\rm exp}(-\frac{1}{2}m_\mu(\xi-A-\frac{V(\mat{x})\epsilon}{2})^2).
\end{equation}
Employing the Bayes rule $p_s(\mat{x},\mat{q},\xi,\mat{M})=p_s(\mat{x},\mat{q},\xi|\mat{M})P_\mat{M}(\mat{M})$, and marginalizing the joint distribution over $\mat{q},\xi,\mat{M}$ we have
\begin{equation}
    p_s(\mat{x})=\int_{\mat{q}}\int_{\xi}\int_{\mat{M}}\ d\mat{q}\ d\xi\ d{\mat{M}}\ p_s(\mat{x},\mat{q},\xi,\mat{M})\propto {\rm exp}(-U(\mat{x})).
\end{equation}
\end{proof}

\section{Numerical Experiments and Applications}\label{sec:exp}

This section verifies our proposed methods by several synthetic examples and some machine learning tasks such as sparse bridge regression, image denoising, and neural network pruning. Our default implementation of QHMC is the scalar QHMC (i.e., S-QHMC in Section~\ref{sec:discuss}) unless stated explicitly otherwise. The number of simulation steps and the step size are set as $L=5$ and $\epsilon=0.03$ in QHMC, if not stated explicitly otherwise~\footnote{We use a small $L$ and a large $\epsilon$ due to the efficiency consideration. Changing mass in QHMC is equivalent to implicitly changing $\epsilon$ and $L$ in HMC, therefore $\epsilon$ and $L$ can be chosen with lots of freedom.}. Our codes are implemented in {\tt MATLAB} and {\tt Python}, and all experiments are run in a computer with 4-core 2.40 GHz CPU and 8.0G memory. All codes are available at \url{https://github.com/KindXiaoming/QHMC}.

\subsection{Synthetic Examples}
\subsubsection{One-Dimensional $\ell_p$ Norm}

\begin{figure}[t]
        \centering
        \begin{subfigure}[b]{0.32\textwidth}
            \centering
            \includegraphics[trim = 35mm 85mm 30mm 85mm, clip,width=\textwidth]{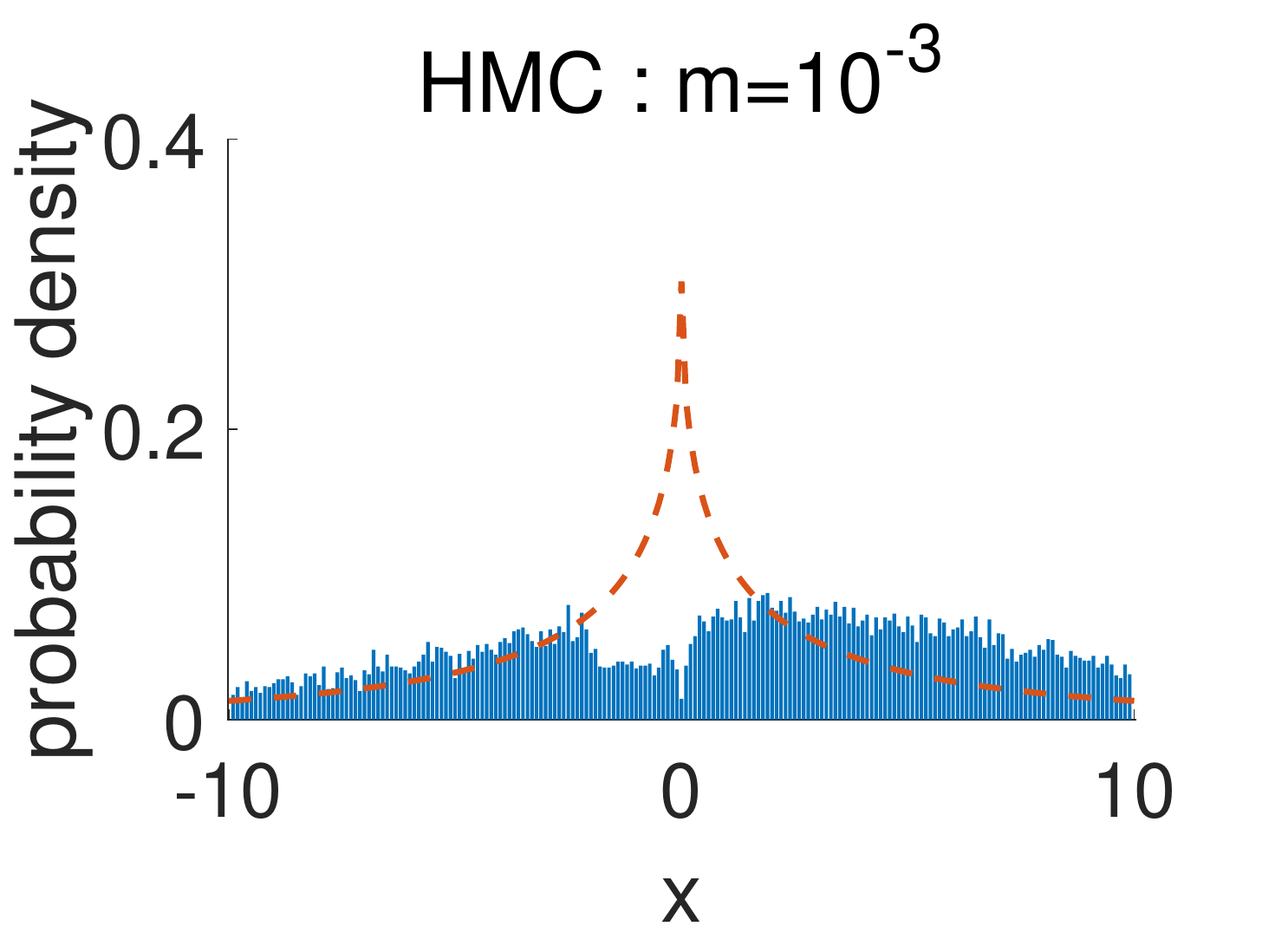}
            \caption[Network2]%
            {{\small}}    
        \end{subfigure}
        \hfill
        \begin{subfigure}[b]{0.32\textwidth}  
            \centering 
            \includegraphics[trim = 35mm 85mm 30mm 85mm, clip,width=\textwidth]{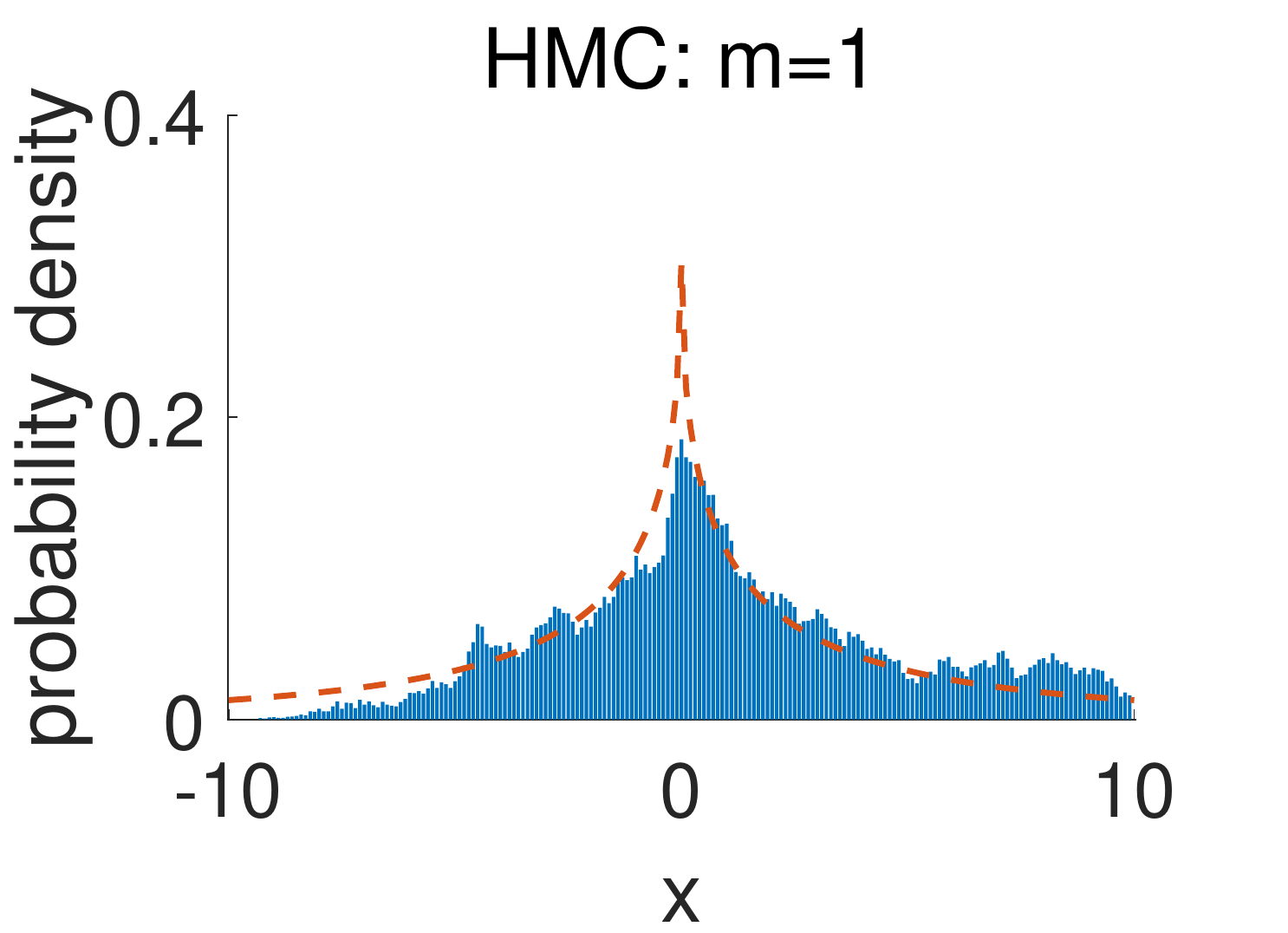}
            \caption[]%
            {{\small}}    
        \end{subfigure}
        \hfill
        \begin{subfigure}[b]{0.32\textwidth}  
            \centering 
            \includegraphics[trim = 35mm 85mm 30mm 85mm, clip,width=\textwidth]{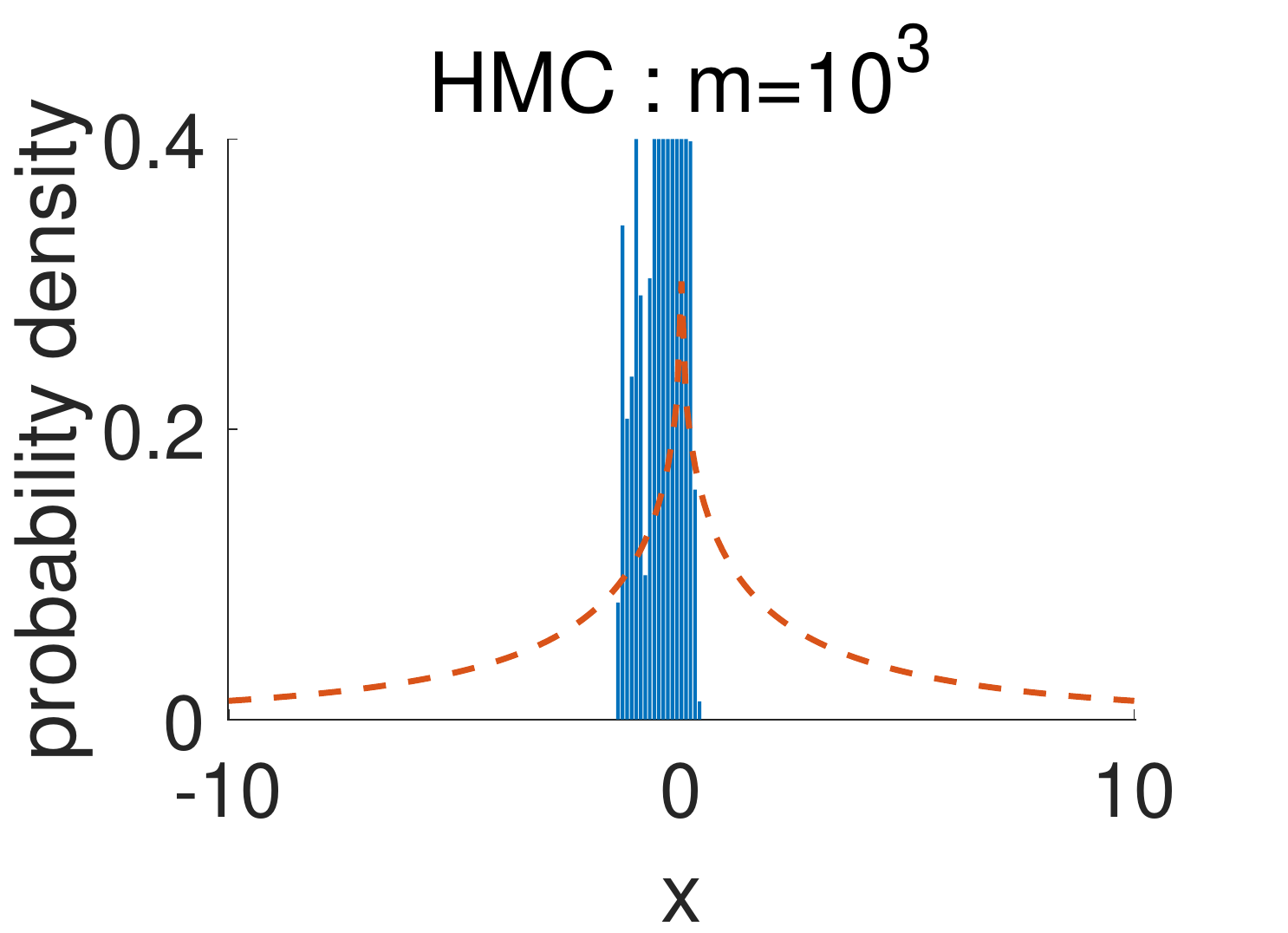}
            \caption[]%
            {{\small}}    
        \end{subfigure}
        \vskip\baselineskip
        \begin{subfigure}[b]{0.32\textwidth}   
            \centering 
            \includegraphics[trim = 35mm 85mm 30mm 85mm, clip,width=\textwidth]{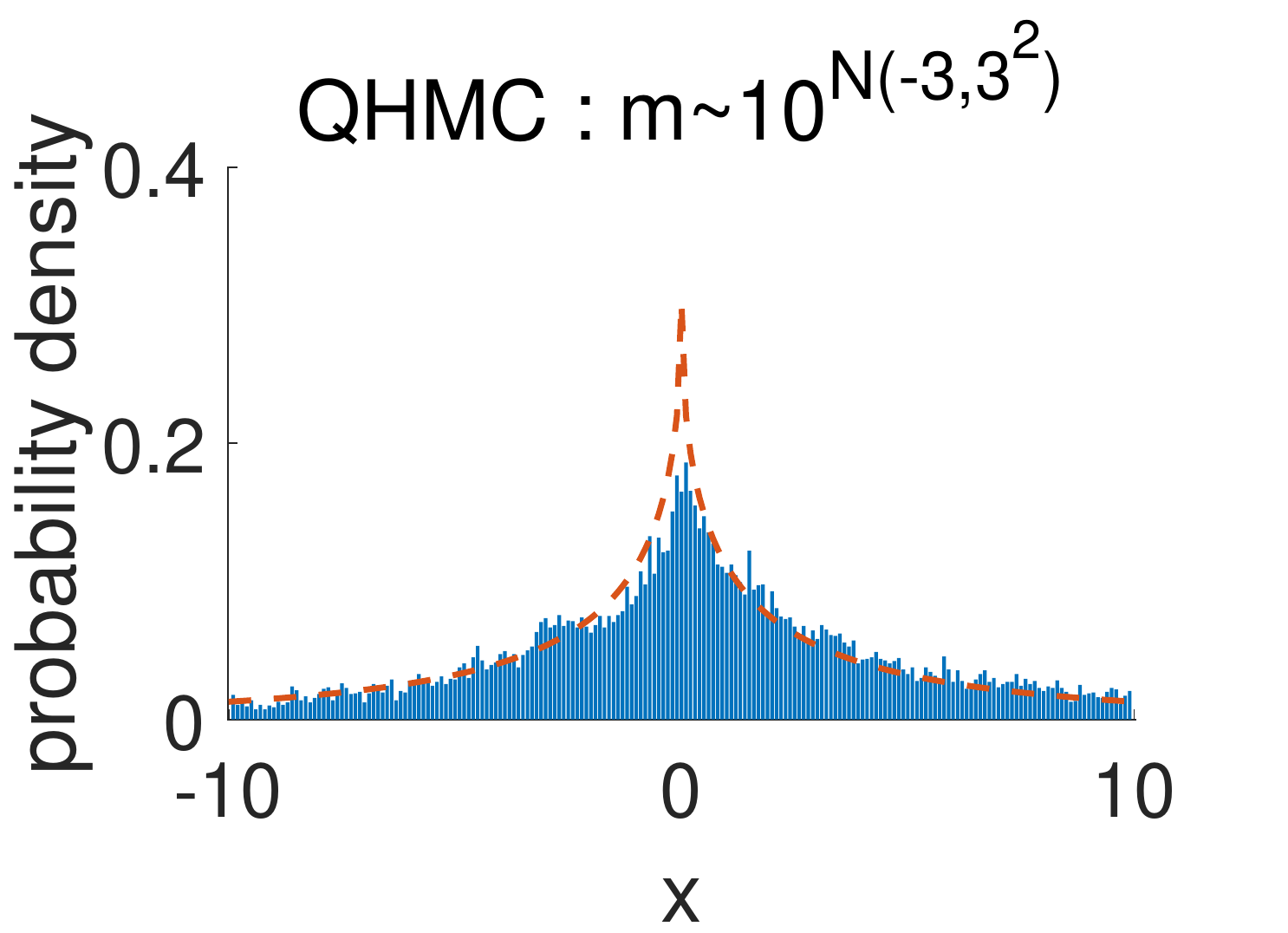}
            \caption[]%
            {{\small}}    
        \end{subfigure}
        \hfill
        \begin{subfigure}[b]{0.32\textwidth}   
            \centering 
            \includegraphics[trim = 35mm 85mm 30mm 85mm, clip,width=\textwidth]{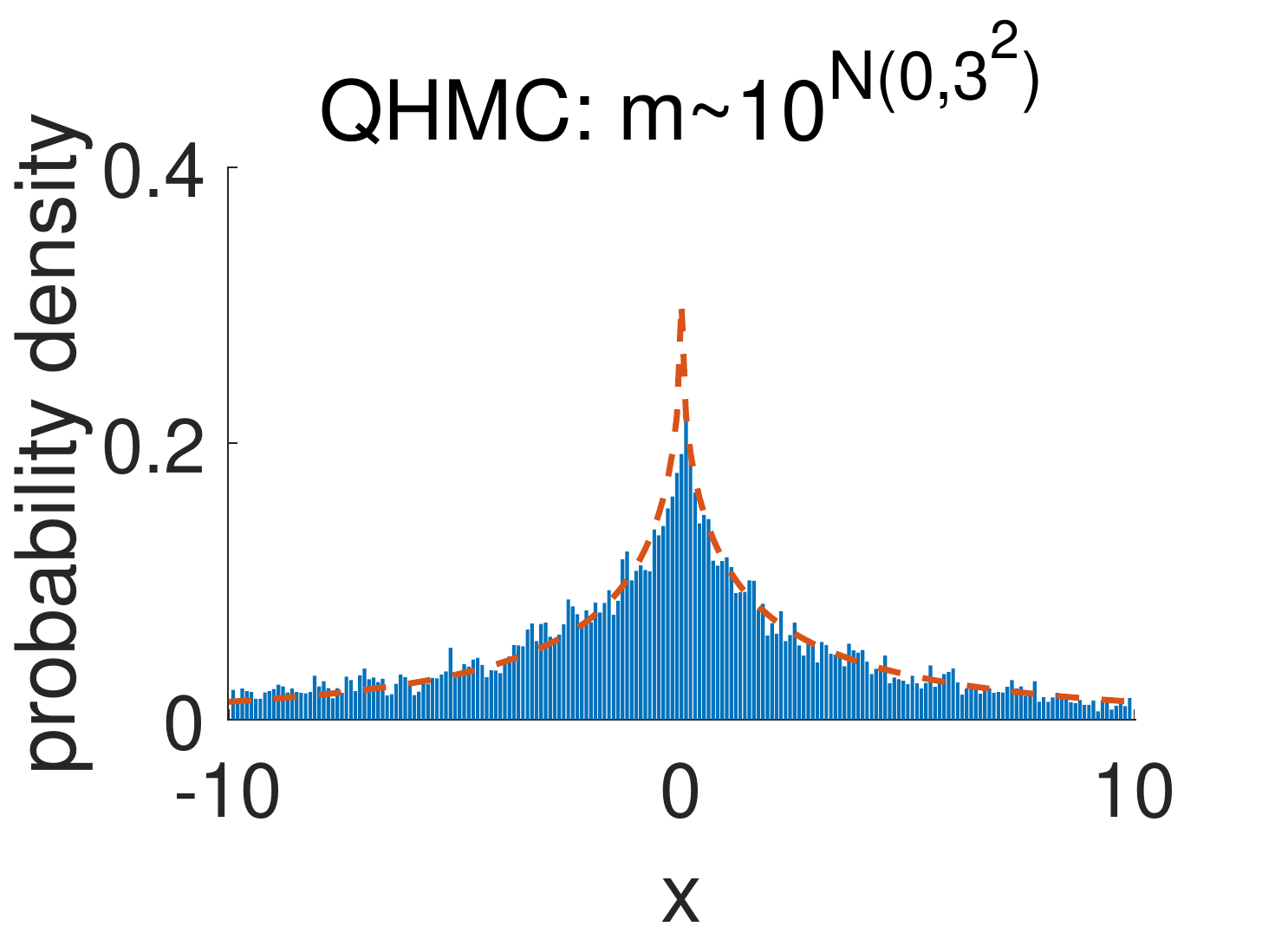}
            \caption[]%
            {{\small}}    
        \end{subfigure}
        \hfill
        \begin{subfigure}[b]{0.32\textwidth}   
            \centering 
            \includegraphics[trim = 35mm 85mm 30mm 85mm, clip,width=\textwidth]{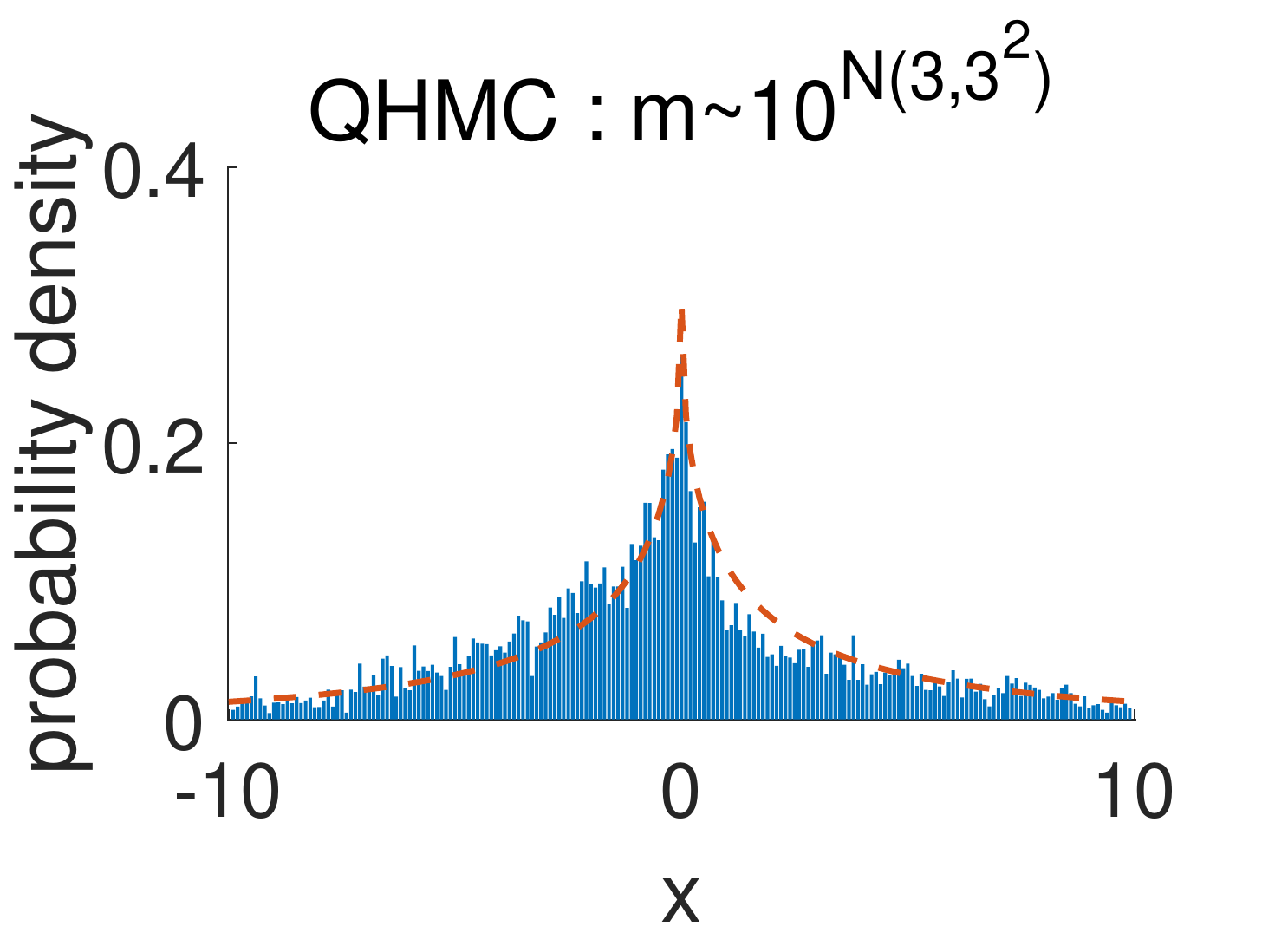}
            \caption[]%
            {{\small}}    
        \end{subfigure}
        \caption[ The average and standard deviation of critical parameters ]
        {\small The performance of HMC and QHMC in sampling ${\rm exp}(-|x|^{1/2})$. Red dashed line: true distribution; blue histogram: simulation result with 50000 samples. HMC is sensitive to the mass parameters, whereas QHMC has excellent performance even if the mass parameters change in a wide range.}
        \label{fig: lp_0d5}
        \vspace{-10pt}
    \end{figure}
    \begin{figure}[t]
        \centering
        \begin{subfigure}[b]{0.32\textwidth}
            \centering
            \includegraphics[trim = 35mm 85mm 30mm 85mm, clip,width=\textwidth]{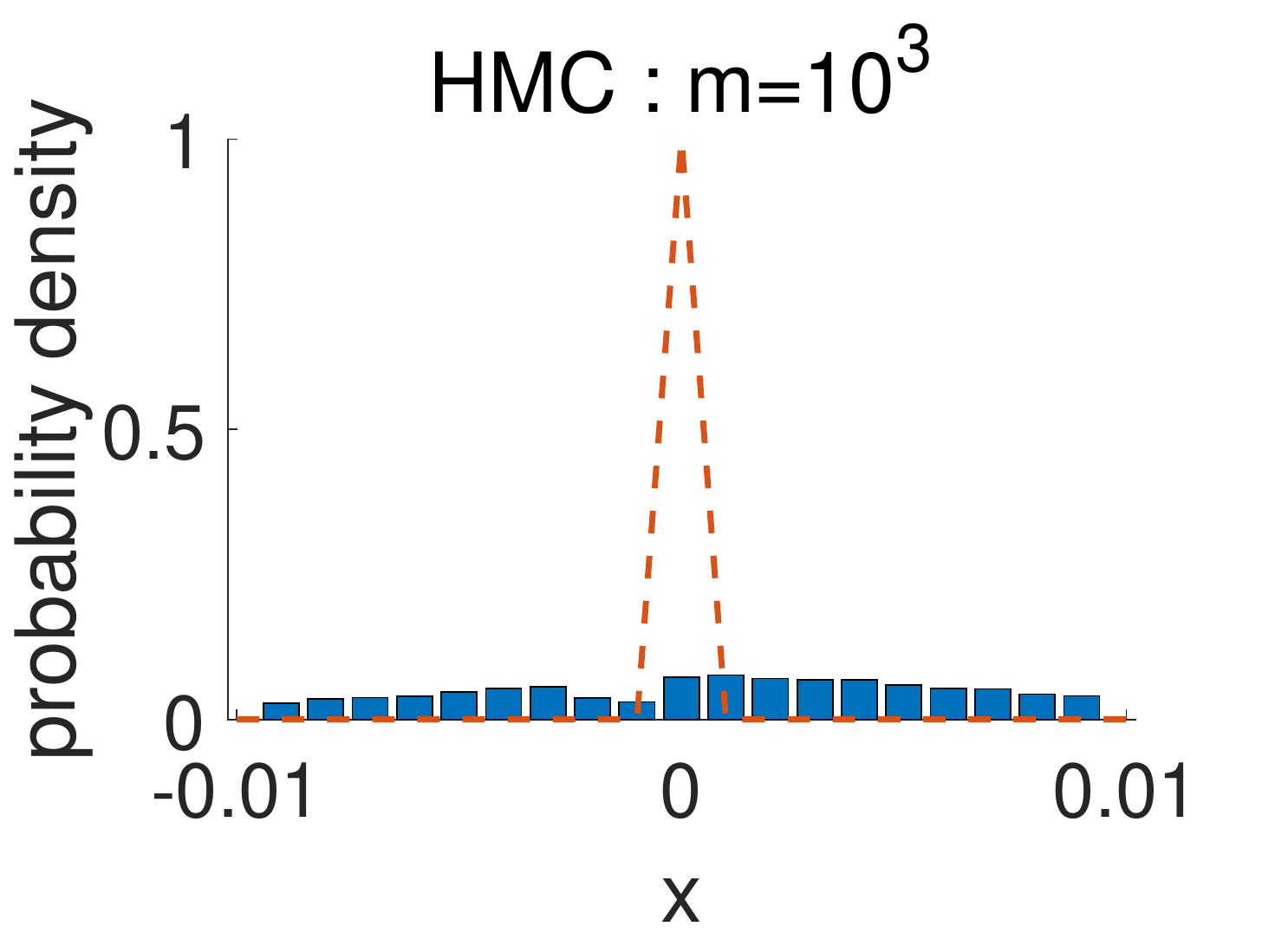}
            \caption[Network2]%
            {{\small}}    
        \end{subfigure}
        \hfill
        \begin{subfigure}[b]{0.32\textwidth}  
            \centering 
            \includegraphics[trim = 35mm 85mm 30mm 85mm, clip,width=\textwidth]{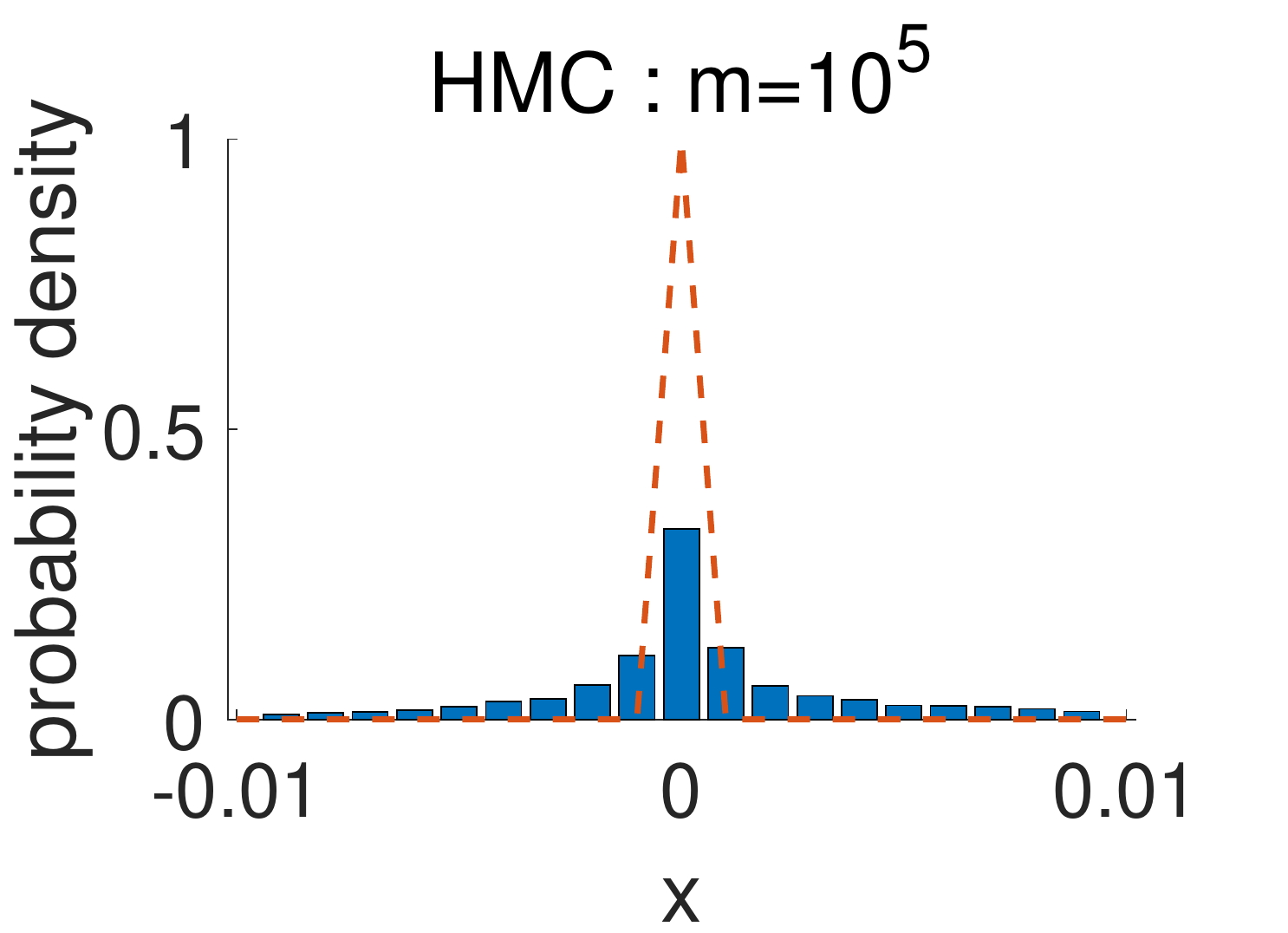}
            \caption[]%
            {{\small}}    
        \end{subfigure}
        \hfill
        \begin{subfigure}[b]{0.32\textwidth}  
            \centering 
            \includegraphics[trim = 35mm 85mm 30mm 85mm, clip,width=\textwidth]{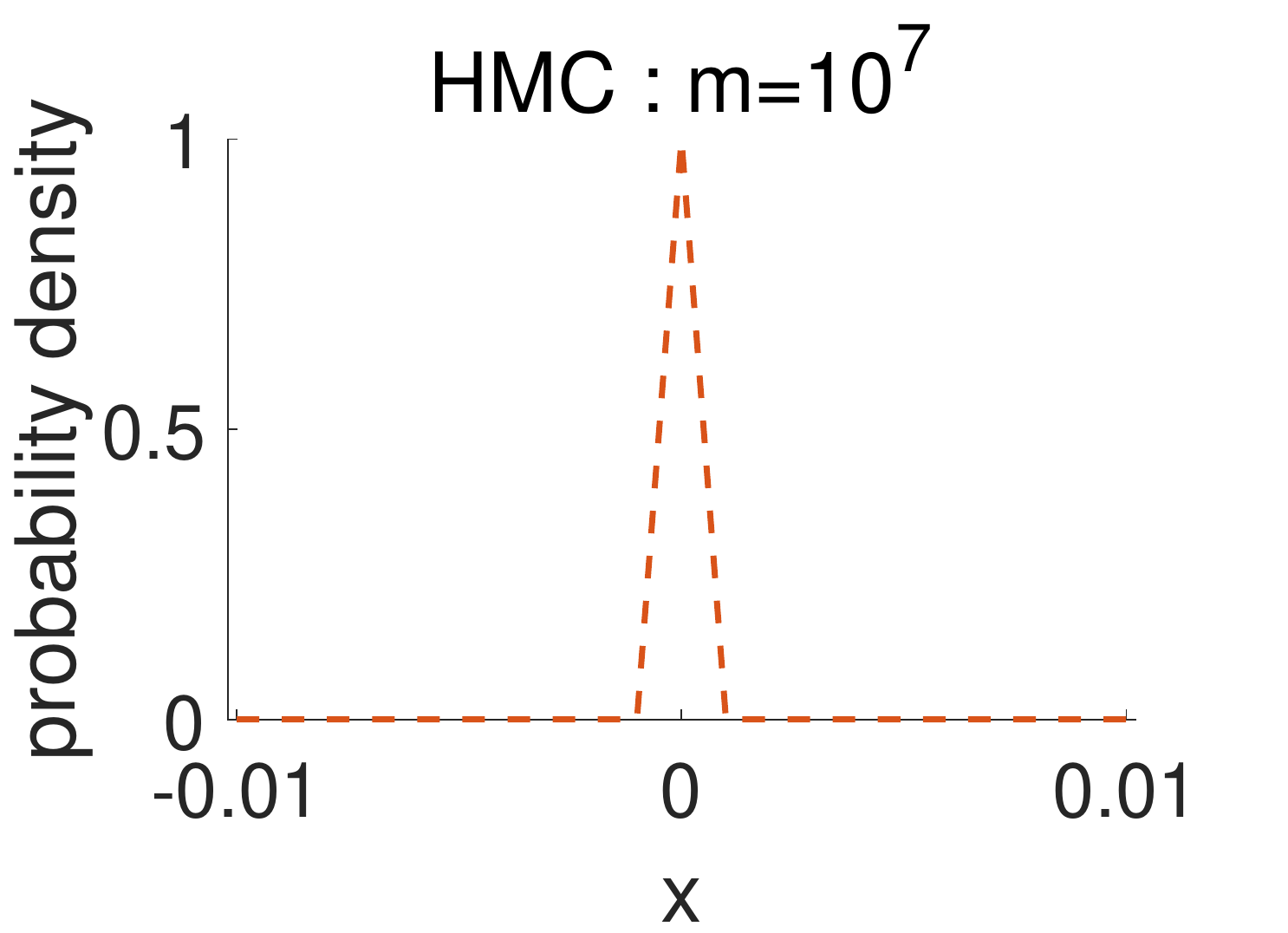}
            \caption[]%
            {{\small}}    
        \end{subfigure}
        \vskip\baselineskip
        \begin{subfigure}[b]{0.32\textwidth}   
            \centering 
            \includegraphics[trim = 35mm 85mm 30mm 85mm, clip,width=\textwidth]{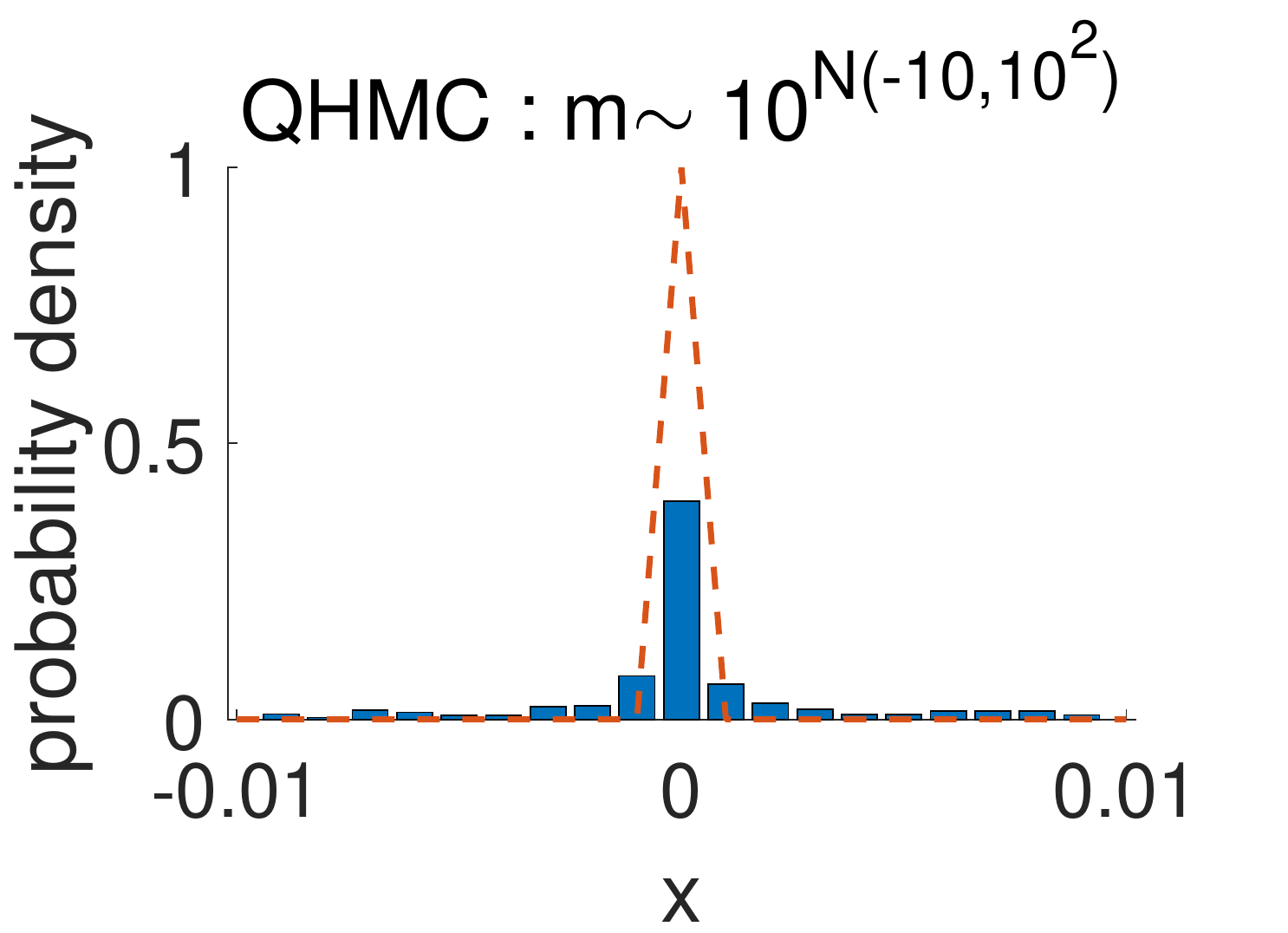}
            \caption[]%
            {{\small}}    
        \end{subfigure}
        \hfill
        \begin{subfigure}[b]{0.32\textwidth}   
            \centering 
            \includegraphics[trim = 35mm 85mm 30mm 85mm, clip,width=\textwidth]{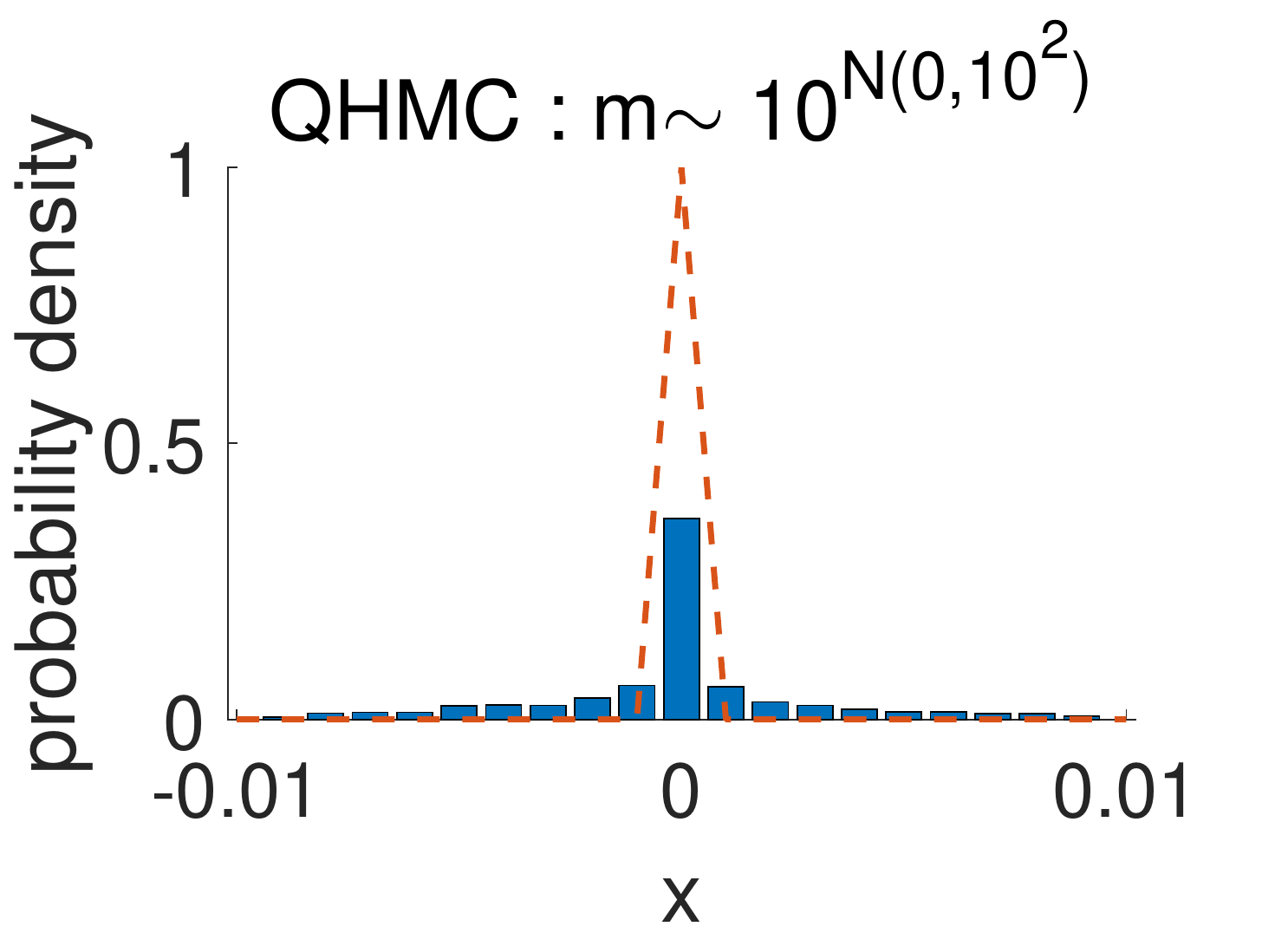}
            \caption[]%
            {{\small}}    
        \end{subfigure}
        \hfill
        \begin{subfigure}[b]{0.32\textwidth}   
            \centering 
            \includegraphics[trim = 35mm 85mm 30mm 85mm, clip,width=\textwidth]{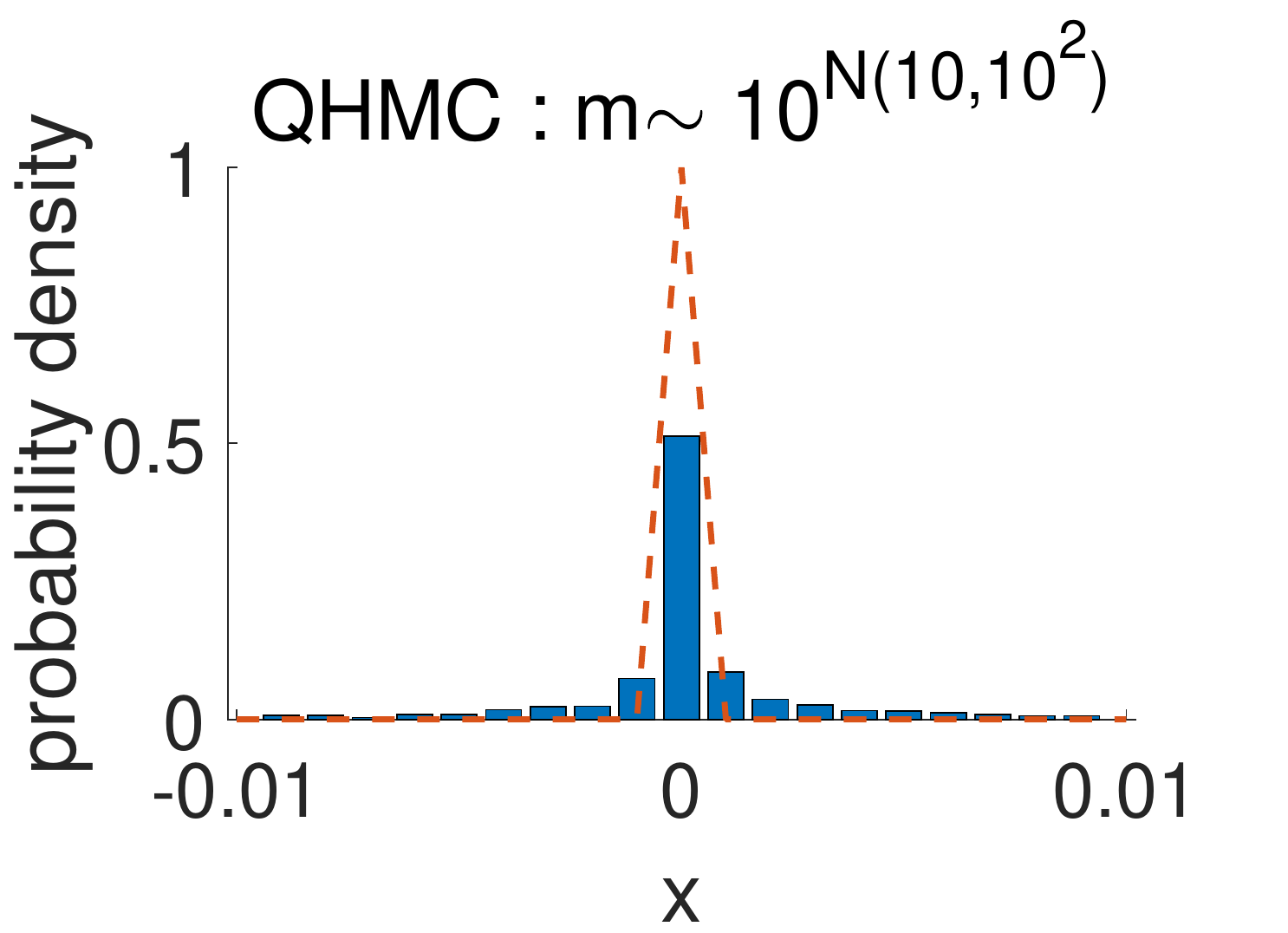}
            \caption[]%
            {{\small}}    
        \end{subfigure}
        \caption[ The average and standard deviation of critical parameters ]
        {\small The results of HMC and QHMC for the  energy function $U(x)=20|x|^{1/10}$. The red dashed line is the true distribution, and the blue histogram is the distribution with 50000 simulation samples. (a)-(c): HMC with different mass values; (d)-(f) QHMC results obtained with different distributions for the mass parameter.}
        \label{fig: lp_0d1}
        \vspace{-15pt}
    \end{figure}

The $\ell_p$ $(0<p<1)$ norm can be used as a regularizer in an optimization problem for model selection and to enforce model sparsity~\citep{zhao2014p,xu20101}. In a Bayesian setting, employing an $\ell_p$ regularization is equivalent to placing a prior density ${\rm exp}(-\frac{\| \mat{x}\|_p^p}  {\tau})$. 

In this experiment, we use QHMC to sample from a 1-D distribution $p(x)\propto {\rm exp}(-|x|^p)$ whose corresponding potential energy function is $U(x)=|x|^p$ for $p=\{1,0.5,0.1\}$. We start from $x_0=0.1$, which is already in the typical set~\footnote{The ``typical set" is a set that contains almost all elements. For instance, $[-3\sigma,3\sigma]$ can be considered as a typical for a 1-D Gaussian distribution with standard deviation $\sigma$,  because this set contains 99.7\% of its elements. In the $\ell_{1/2}$ case, we consider $[-20,20]$ as the typical set because 98.7\% of its elements are in this range.}, therefore the burn-in steps are not needed. The results for $p=1,0.5,0.1$ are presented in Fig.~\ref{fig: lp_1} to Fig.~\ref{fig: lp_0d1} respectively. We can see that as $p$ approaches zero, the advantages of QHMC over HMC becomes increasingly significant.

\begin{figure}[t]
        \centering
        \begin{subfigure}[b]{0.32\textwidth}
            \centering
            \includegraphics[trim = 35mm 85mm 30mm 85mm, clip,width=\textwidth]{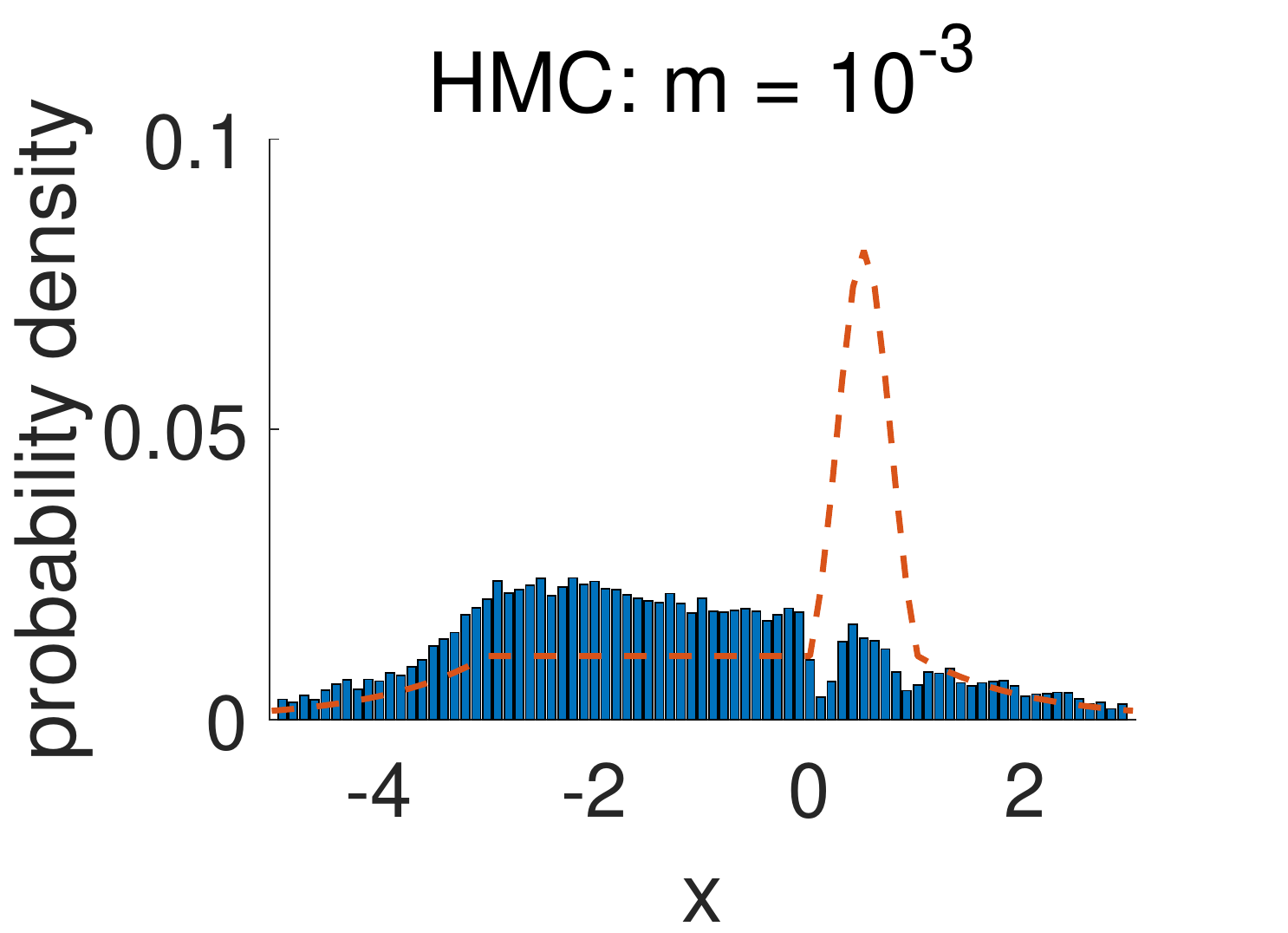}
            \caption[Network2]%
            {{\small}}    
        \end{subfigure}
        \hfill
        \begin{subfigure}[b]{0.32\textwidth}  
            \centering 
            \includegraphics[trim = 35mm 85mm 30mm 85mm, clip,width=\textwidth]{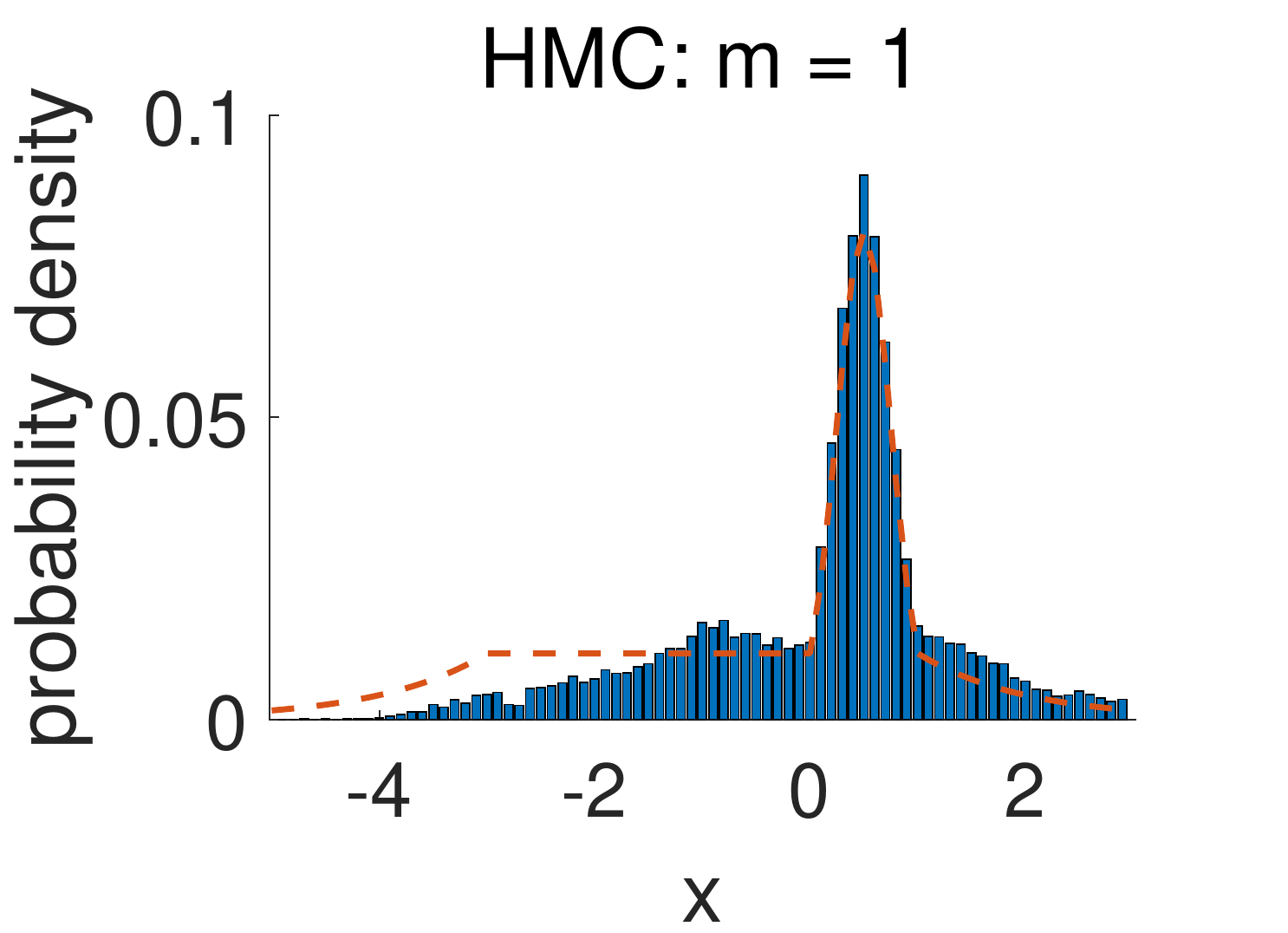}
            \caption[]%
            {{\small}}    
        \end{subfigure}
        \hfill
        \begin{subfigure}[b]{0.32\textwidth}  
            \centering 
            \includegraphics[trim = 35mm 85mm 30mm 85mm, clip,width=\textwidth]{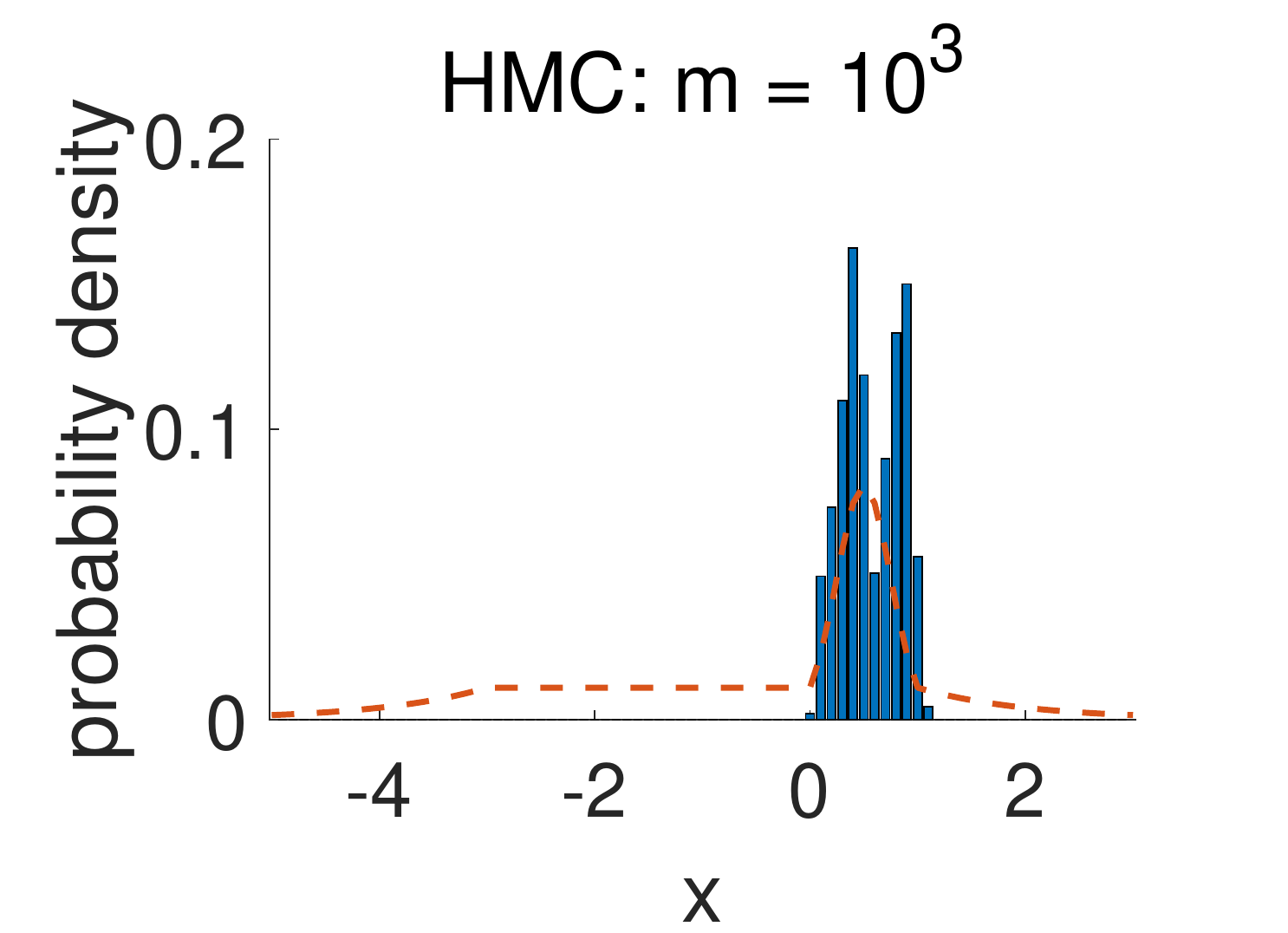}
            \caption[]%
            {{\small}}    
        \end{subfigure}
        \vskip\baselineskip
        \begin{subfigure}[b]{0.32\textwidth}   
            \centering 
            \includegraphics[trim = 35mm 85mm 30mm 85mm, clip,width=\textwidth]{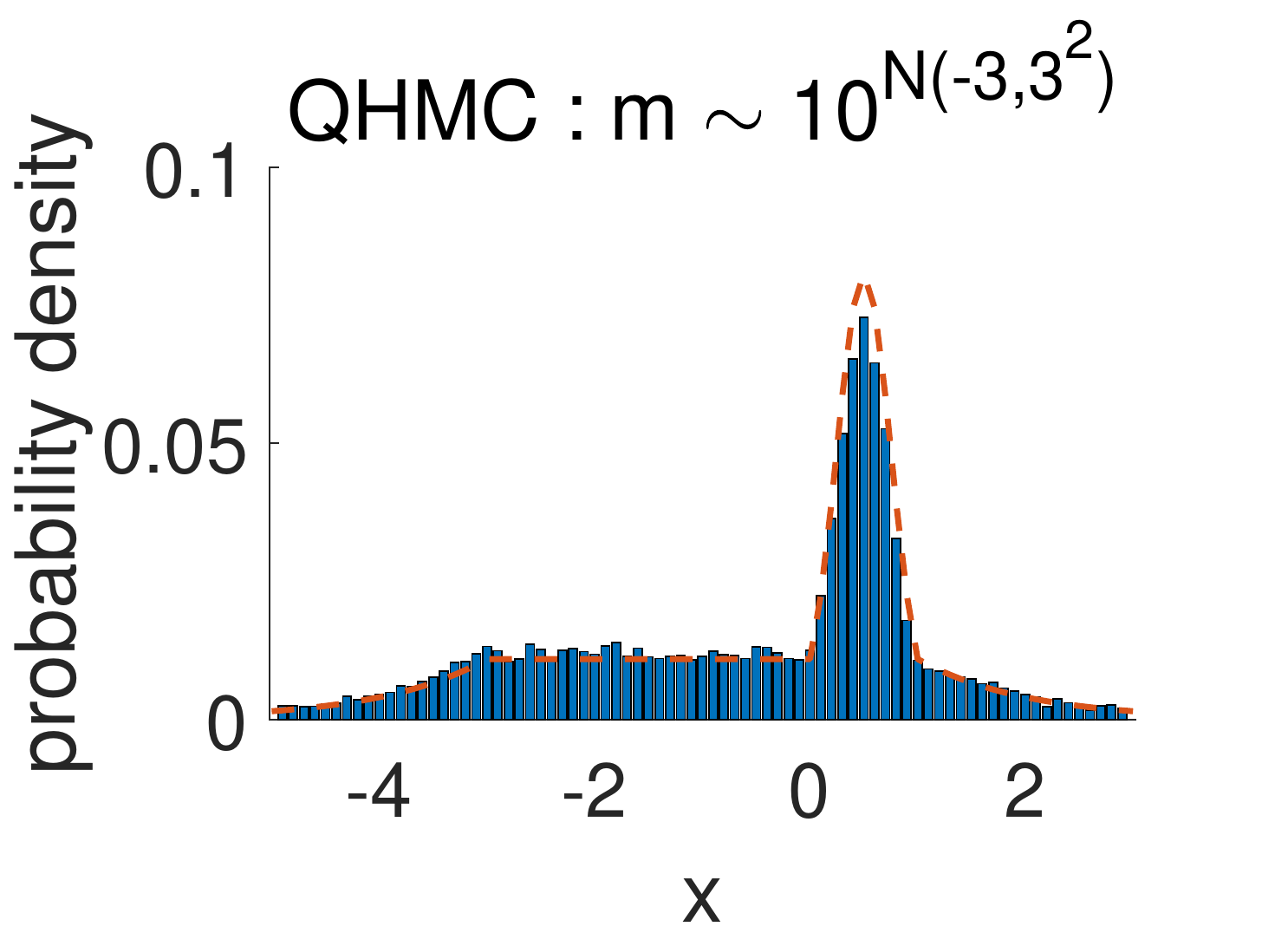}
            \caption[]%
            {{\small}}    
        \end{subfigure}
        \hfill
        \begin{subfigure}[b]{0.32\textwidth}   
            \centering 
            \includegraphics[trim = 35mm 85mm 30mm 85mm, clip,width=\textwidth]{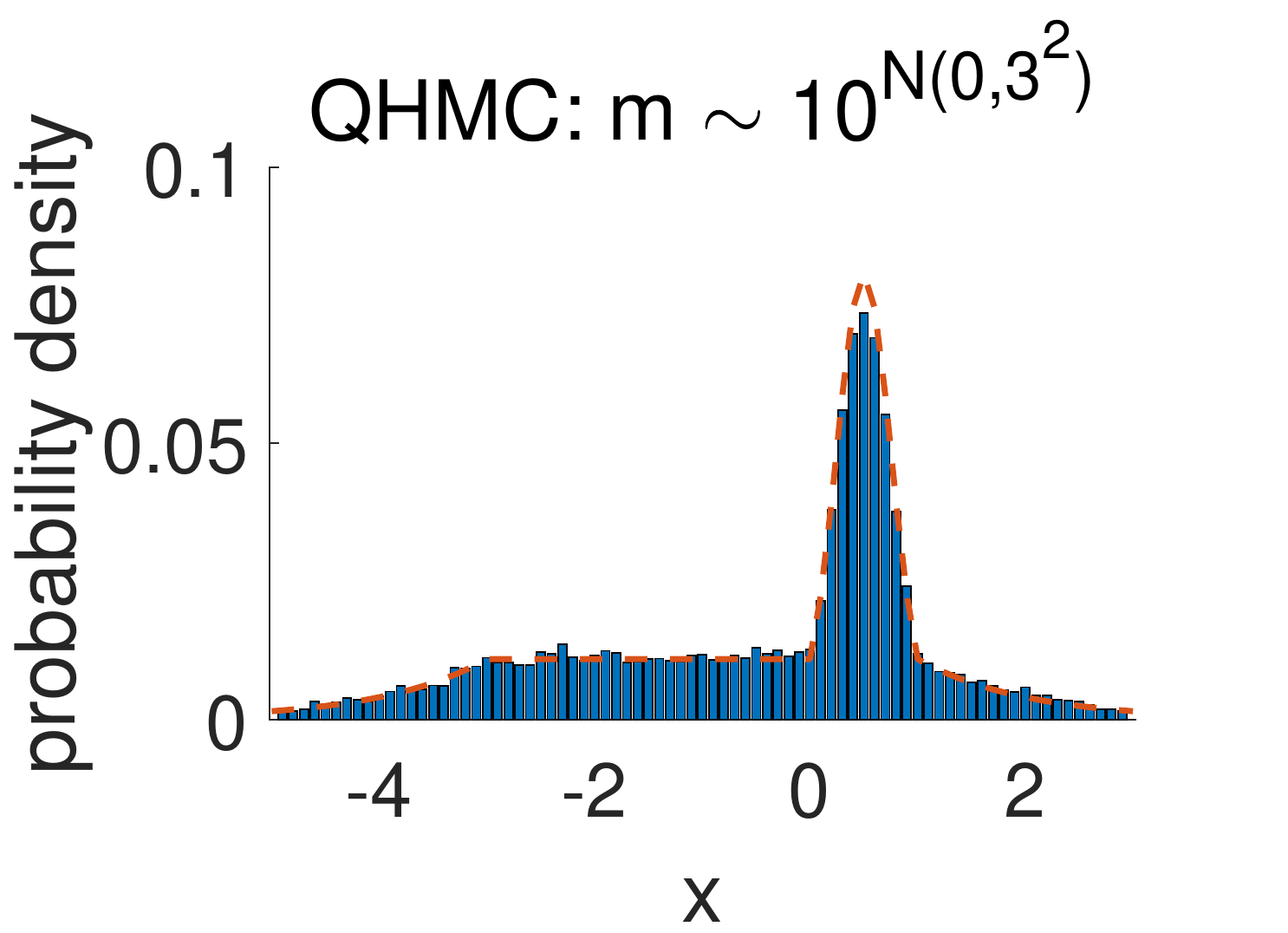}
            \caption[]%
            {{\small}}    
        \end{subfigure}
        \hfill
        \begin{subfigure}[b]{0.32\textwidth}   
            \centering 
            \includegraphics[trim = 35mm 85mm 30mm 85mm, clip,width=\textwidth]{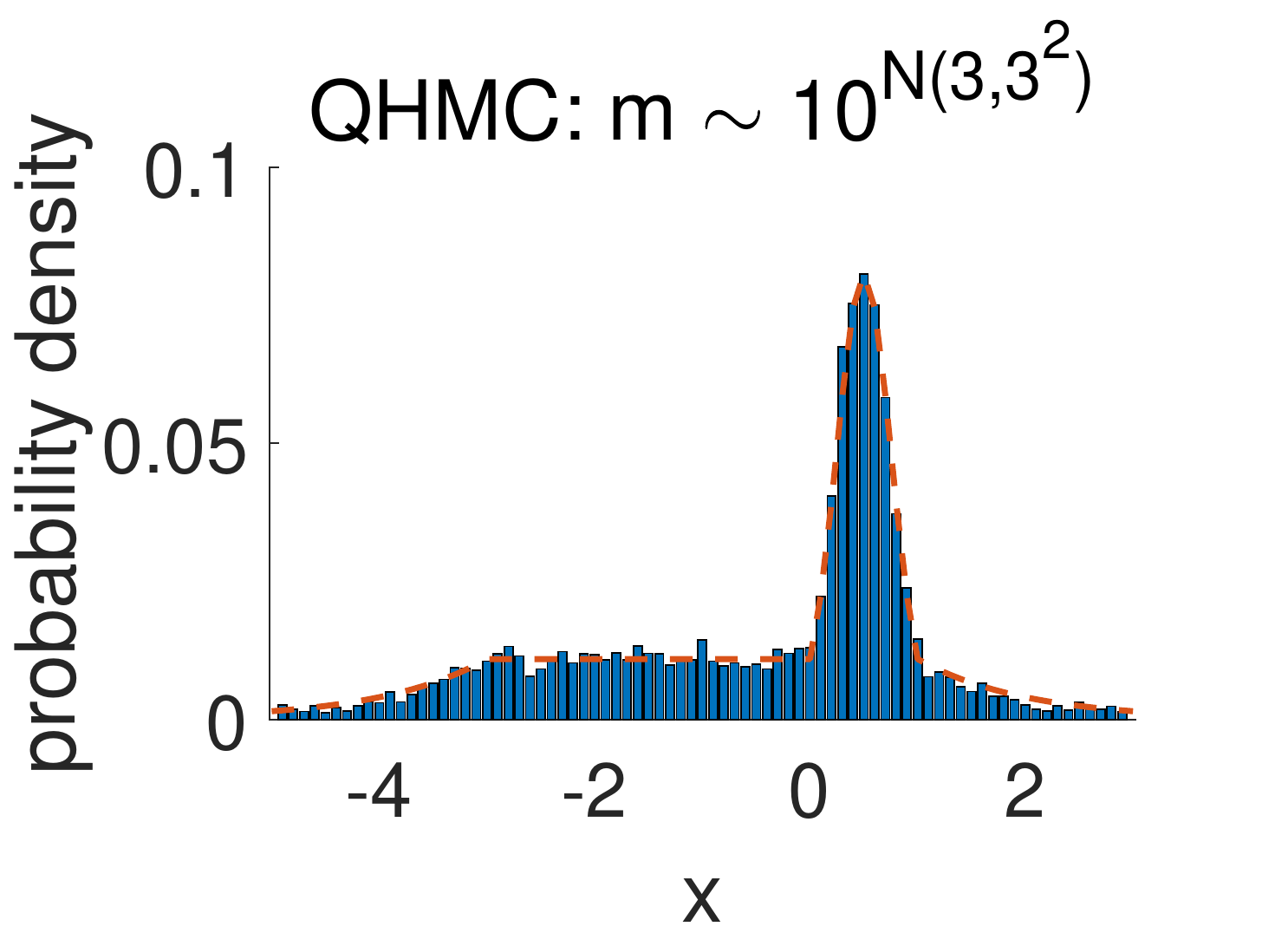}
            \caption[]%
            {{\small}}    
        \end{subfigure}
        \caption[ The average and standard deviation of critical parameters ]
        {\small The spiky+smooth distribution example. We run 50000 paths to get the sample distributions. The ground truth is shown by the dashed red line. It is hard to pick a fixed $m$ in HMC to get accurate results for both the flat smooth region and the spiky region, as shown in (a)--(c). The proposed QHMC can accurately sample the entire target distribution for a wide range of mass distributions, as shown in (d)--(f).}
        \label{fig: well}
        \vspace{-10pt}
    \end{figure}
    
We compare HMC and QHMC implemented with different mass parameters. Take the results of $p=0.5$ as an example (see Fig.~\ref{fig: lp_0d5}). In each row, the mass value in HMC is set as the the median mass value of QHMC. Fig.~\ref{fig: lp_0d5} (a)-(c) show that the distribution can be well sampled by HMC only if the mass value is properly chosen. The region around $x=0$ can be hardly explored when the mass value is too small, as shown in Fig.~\ref{fig: lp_0d5} (a). In Fig.~\ref{fig: lp_0d5} (c), the whole distribution cannot be well explored due to the random walk caused by a large mass. However, the proposed QHMC method does not suffer from this issue. As shown in Fig.~\ref{fig: lp_0d5} (d)-(f), the QHMC can produce accurate sample distributions when the median mass value $10^{\mu_m}$ changes from $10^{-3}$ to $10^3$.

\subsubsection{One-Dimensional Distribution with Both Smooth and Spiky Regions}\label{sec:exp-smooth-spiky}
Section~\ref{sec:mass_adapt} states that different mass magnitudes are preferred for smooth and spiky regions. In this example, we consider the following piecewise (non-normalized) potential energy function:

\begin{equation}\label{eq:well}
U(x)=
\left\{
             \begin{array}{lc}
             -x-3\quad &(x\leq-3)\\
             0\quad &(-3<x\leq 0)\\
             8x(x-1)\quad &(0<x\leq 1)\\
             x-1\quad &(x>1).
             \end{array}
\right.
\end{equation}
This function is very flat and smooth in the range $[-3,0]$ and very spiky in the interval $[0,1]$. 

In order to understand why QHMC outperforms HMC in this synthetic example, we provide a semi-quantitative analysis by introducing the diffusion coefficient $D$ and characteristic length $r_c$. We consider the flat energy function $U(x)=0$ for $-3<x\leq 0$. Suppose that we use a fixed step size $\epsilon$ and $L$ leap-frog steps in each simulation path, and that we run $N$ simulation paths. Because $q_i(i=1,2,\cdots,N)$ are drawn from a Gaussian distribution, their sum also has a Gaussian distribution with a variance that is $N$ times larger. Suppose the particle starts from $x_0=0$, therefore at step $N$ the position of the particle is $x=(\epsilon L/m)\sum_{i=1}^N p_i$, and it is easy to verify the variance of $x$ is $\overline{ x^2}=N(\epsilon L)^2/m$. We can define the diffusion coefficient $D$ as $\overline{x^2}=2DN$ such that $D=(\epsilon L)^2/2m\sim1/m$. Define the characteristic step size as $\epsilon_c=\sqrt{\overline{x^2}/N}=\sqrt{2D}$. Let $r_c$ be the characteristic scale of the smooth region~\footnote{The characteristic length $r_c$ can be seen as the size of typical set, e.g. the standard deviation of a 1-D distribution. Here we do not pursue a rigorous definition and just use it for semi-quantitative analysis.}.  If the particle can explore the smooth region within $N\sim O(1)$ steps, the characteristic step size $\epsilon_c$ and characteristic scale $r_c$ should be at the same order, i.e. $D\sim O(\epsilon_c^2)\sim O(r_c^2)$ or equivalently $m\sim O((\epsilon L)^2/2r_c^2)$.

We can roughly estimate the optimal mass for the piecewise function. On the one hand, Lemma \ref{lemma:3} indicates that the mass should be above $a\epsilon^2/4=3.6\times 10^{-3}$ for the spiky region $[0,1]$. On the other hand, based on the discussion above and the characteristic length $r_c\sim 3$, we know the optimal mass is $m\sim (\epsilon L)^2/2r_c^2= 1.3\times 10^{-3}$ for the smooth region $[-3,0]$. These two requirements cannot be meet simultaneously with a fixed mass as in a standard HMC. As a result, the HMC produces inaccurate distributions for this example, as shown in Fig. \ref{fig: well} (a)-(c).

However, the QHMC can employ a random mass to generate accurate sample distributions in both regions. To demonstrate this, we start from $x_0=0$ and run 50000 paths. As shown in Fig.~\ref{fig: well} (d)-(e), the QHMC produces accurate distributions with $\mu_m$ varying in a wide range.

\subsubsection{One-dimensional Multimodal Distribution}
\begin{figure}[t]
        \centering
        \begin{subfigure}[t]{0.475\textwidth}
            \centering        
            \includegraphics[trim = 35mm 85mm 35mm 85mm, clip,width=\textwidth]{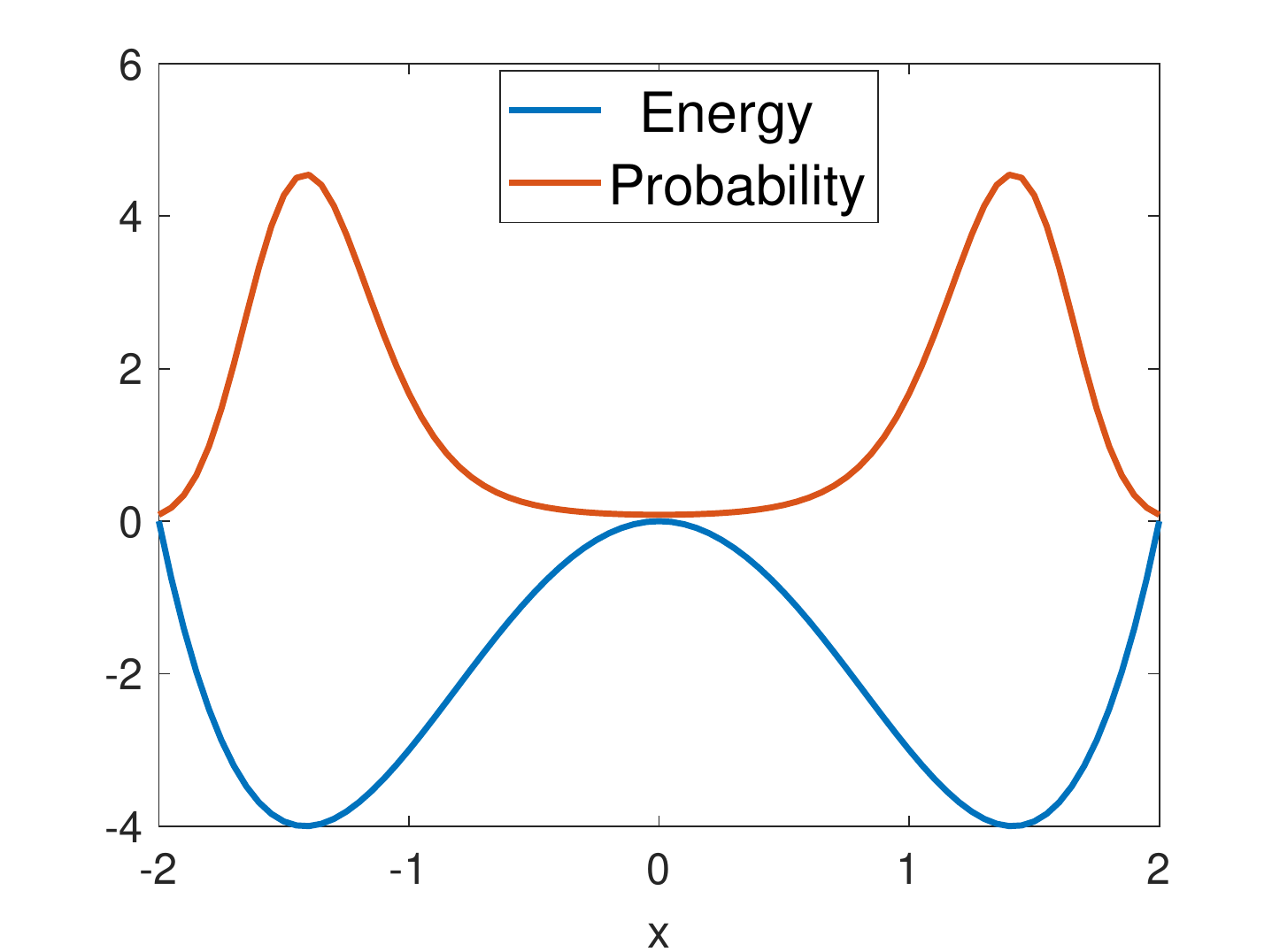}
            \caption[Network2]%
            {{\small}}    
        \end{subfigure}
        \hfill
        \begin{subfigure}[t]{0.475\textwidth}  
            \centering 
            \includegraphics[trim = 35mm 85mm 35mm 85mm, clip,width=\textwidth]{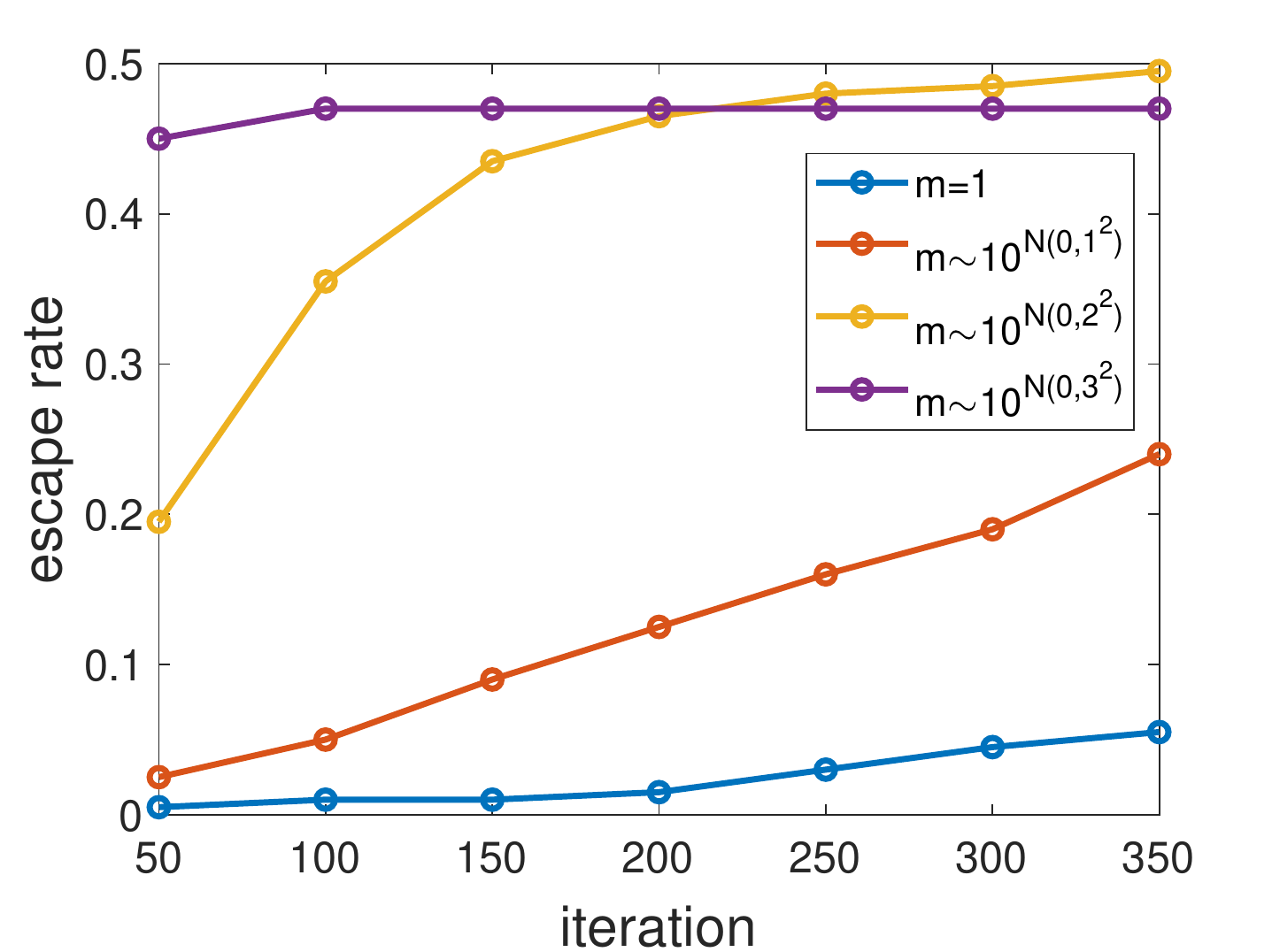}
            \caption[]%
            {{\small}}
        \end{subfigure}
        \begin{subfigure}[t]{0.31\textwidth}
            \centering        
            \includegraphics[trim = 35mm 82mm 35mm 85mm, clip,width=\textwidth]{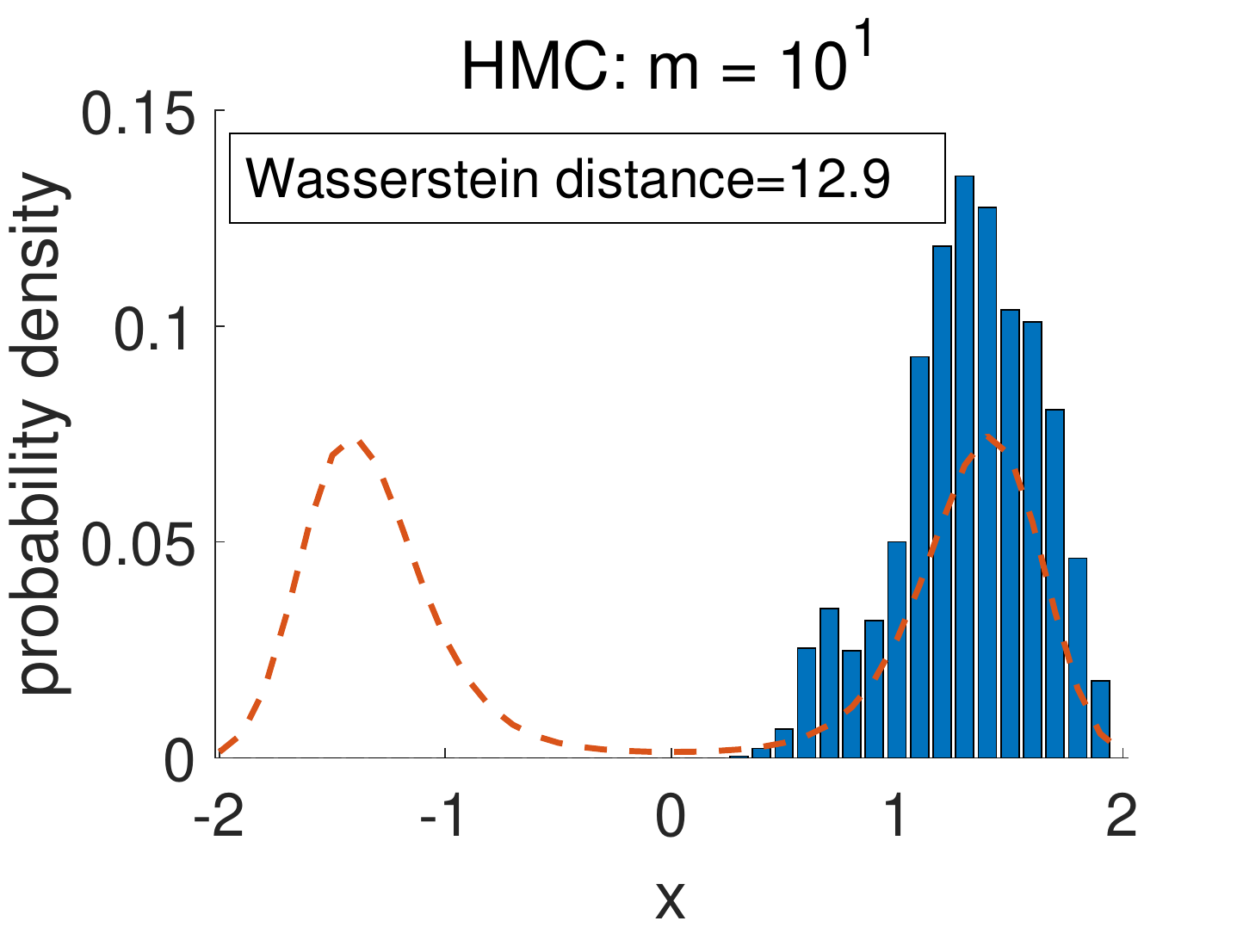}
            \caption[Network2]%
            {{\small}}    
        \end{subfigure}
        \hfill
        \begin{subfigure}[t]{0.31\textwidth}  
            \centering 
            \includegraphics[trim = 35mm 82mm 35mm 85mm, clip,width=\textwidth]{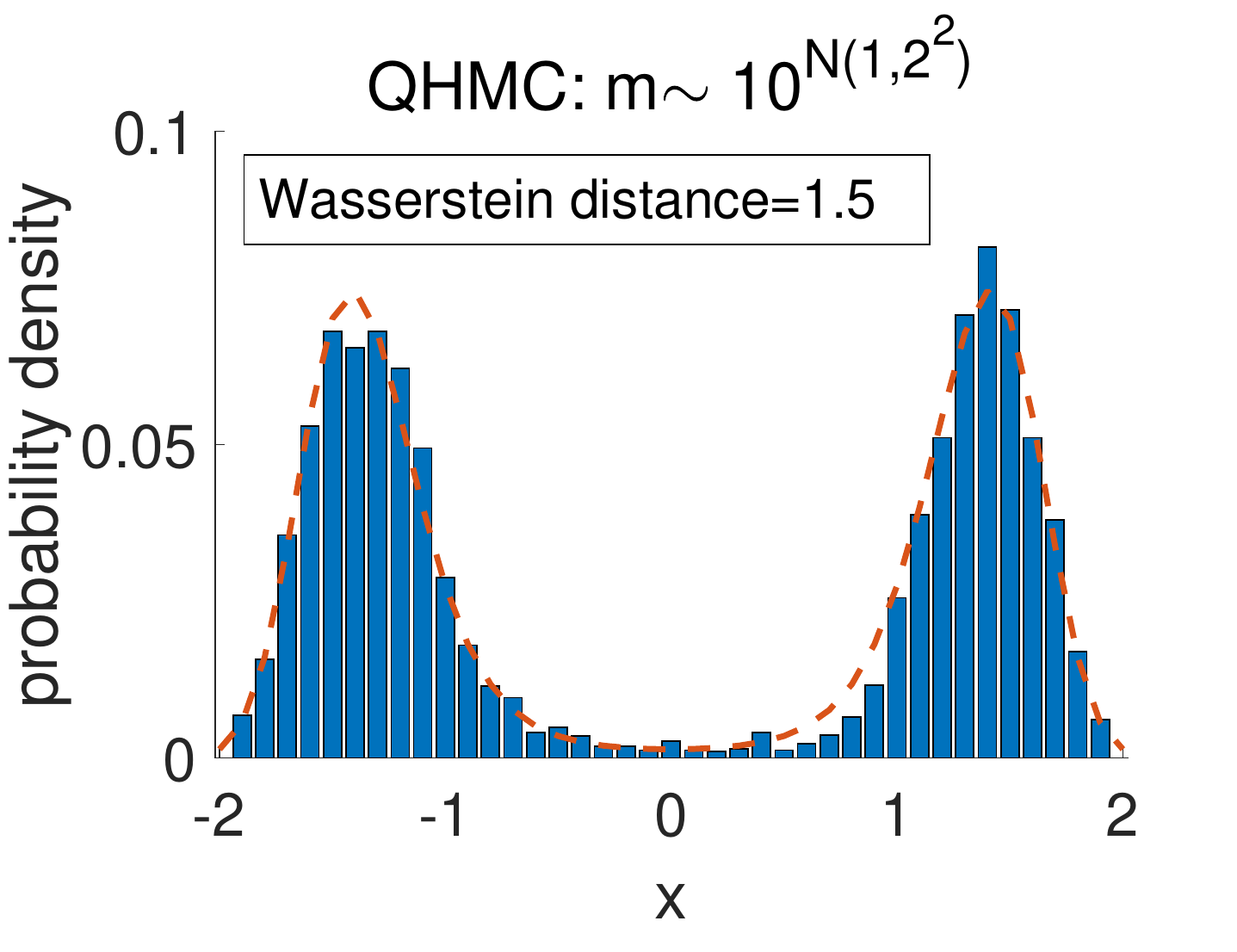}
            \caption[]%
            {{\small}}
        \end{subfigure}
        \hfill
        \begin{subfigure}[t]{0.31\textwidth}  
            \centering 
            \includegraphics[trim = 35mm 82mm 35mm 85mm, clip,width=\textwidth]{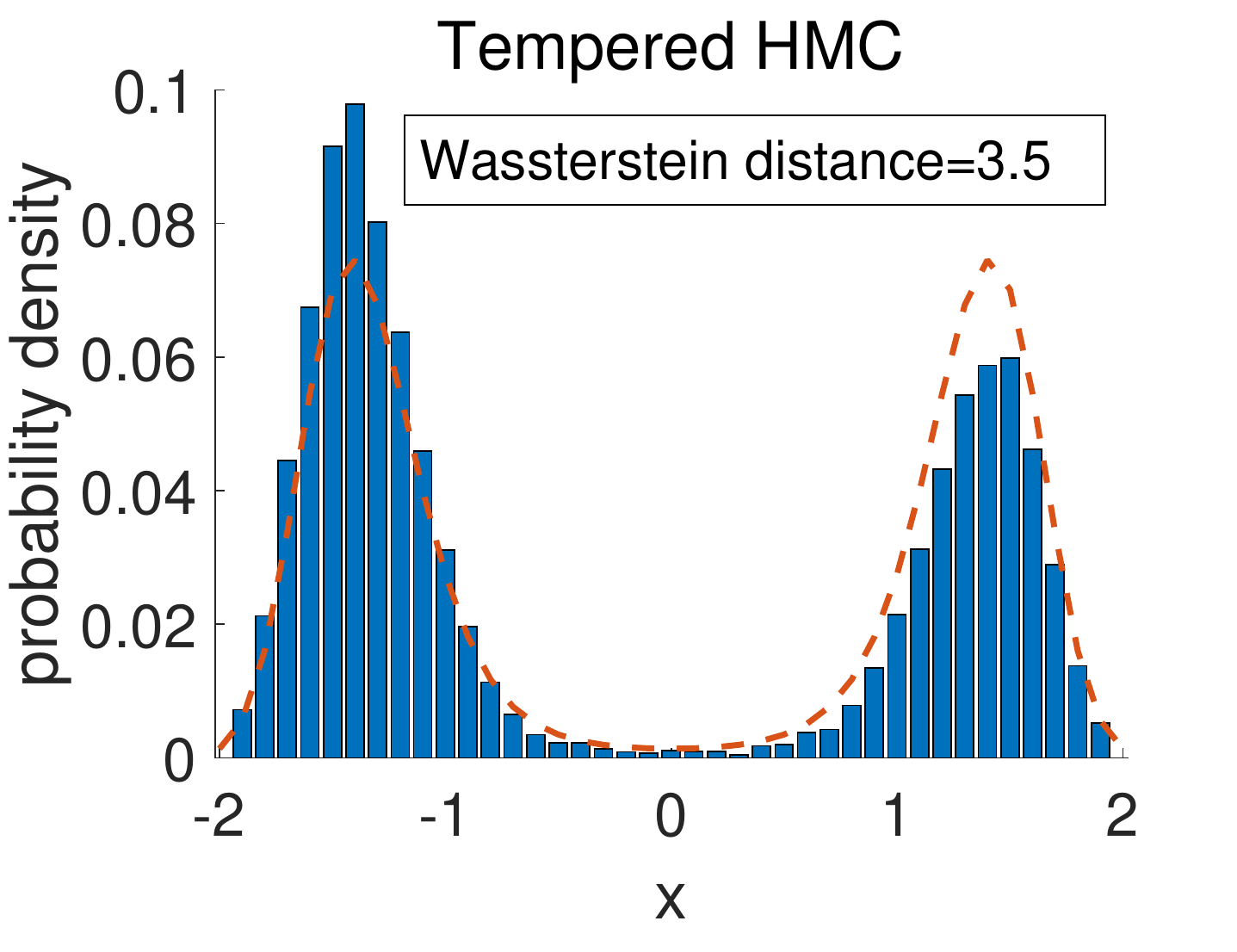}
            \caption[]%
            {{\small}}
        \end{subfigure}
        \caption[]
        {\small A double-well example. (a) The potential energy function $U(x)=x^4-4x^2$ has global minimum at $\pm \sqrt{2}$. We initiate the particle from $x_0=\sqrt{2}$ and run the HMC/QHMC sampling. (b) The escape ratio of HMC and QHMC. (c)-(e): sample distributions from HMC, QHMC and tempered HMC, respectively. We test each algorithm for 100 trials, and plot its sampling result with the 20th smallest Wasserstein distance.}
        \label{fig: multimodal}
        \vspace{-10pt}
    \end{figure}
    
This experiment shows the effectiveness of QHMC in sampling from a multimodal distribution, which is explained in Section~\ref{sec:quantum-tunneling}. 

We consider a double-well posterior distribution $p(x)\propto {\rm exp}(-U(x))$ where the potential energy function $U(x)=x^4-4x^2$, as shown in Fig.~\ref{fig: multimodal} (a). We simulate 200 particles starting from the right minimum point $x_0=\sqrt{2}$, and we check if each particle has crossed the peak at $x=0$ after every 50 iterations. If yes, then the particle has successfully escaped from the right well. We report the ratio of particles that have escaped from the right wells after certain iterations in Fig.~\ref{fig: multimodal} (b). We observe that a moderate variance $\sigma_m\sim 2$ of the log-mass distribution can lead to significantly better performance than simply fixing the mass in standard HMC. We also observe that the particles can escape from the right well more quickly once the log-mass distribution has a larger variance. 

We further compare the proposed QHMC algorithm with standard HMC and tempered HMC~\citep{graham2017continuously}. The HMC with tempering is implemented by setting a low temperature $T_l=1$ and high temperature $T_h=25$ and run a high-temperature step every 30 paths, including 10 paths of burn-in and 20 paths to collect samples. For each method, we run 100 trials of sampling results and sort them by the Wasserstein distance. We show the sample distributions with the 20th smallest Wasserstein distance in Fig.~\ref{fig: multimodal} (c)--(e). The QHMC method produces very accurate results due to its quantum tunneling effect. Tempered HMC can produce a multimodal sample distribution, but its Wasserstein distance is larger than that of QHMC.

\subsubsection{Ill-Conditioned Gaussian Distribution}
In Section \ref{sec:discuss}, we proposed two practical parameterizations of QHMC: S-QHMC and D-QHMC. We show for the ill-conditioned Gaussian distribution, D-QHMC can take into account different scalings along each dimension, hence outperforms S-QHMC and randomized HMC significantly.

The Gaussian distribution $p(\mat{x})\propto {\rm exp}(-\frac{1}{2}\mat{x}^T\mat{\Sigma}^{-1}\mat{x})$ is a well-studied case in sampling problems. The preconditing of mass matrix can help improve sampling performance of ill-conditioned Gaussian distributions and the rule of thumb is $\mat{M}^{-1}\sim\mat{\Sigma}$ as mentioned in e.g. ~\citep{BESKOS20112201}. We consider an ill-conditioned Gaussian with $\mat{\Sigma}=\rm{diag}(100,1)$. The sampler starts from $\mat{x}_0=\mat{0}$ and run 10000 paths to collect samples. 

The results are shown in Fig.~\ref{fig:g2d}. We compare five methods: (a) standard HMC: $m=0.01$; (b) Randomized HMC: $m=0.01$ and $\tau=0.03$; (c) Riemannian HMC: $\mat{M}=0.1\mat{\Sigma}^{-1}$; (d) S-QHMC: $\mu_m=-2,\sigma_m=1$; (e) D-QHMC: $\mu_m^{(1)}=-3, \mu_m^{(2)}=-1, \sigma_m^{(1)}=\sigma_m^{(2)}=1$.  It is clear that D-QHMC outperforms all other four methods on this example.

\begin{figure}[t]
        \centering
        \begin{subfigure}[b]{0.32\textwidth}
            \centering
            \includegraphics[width=\textwidth]{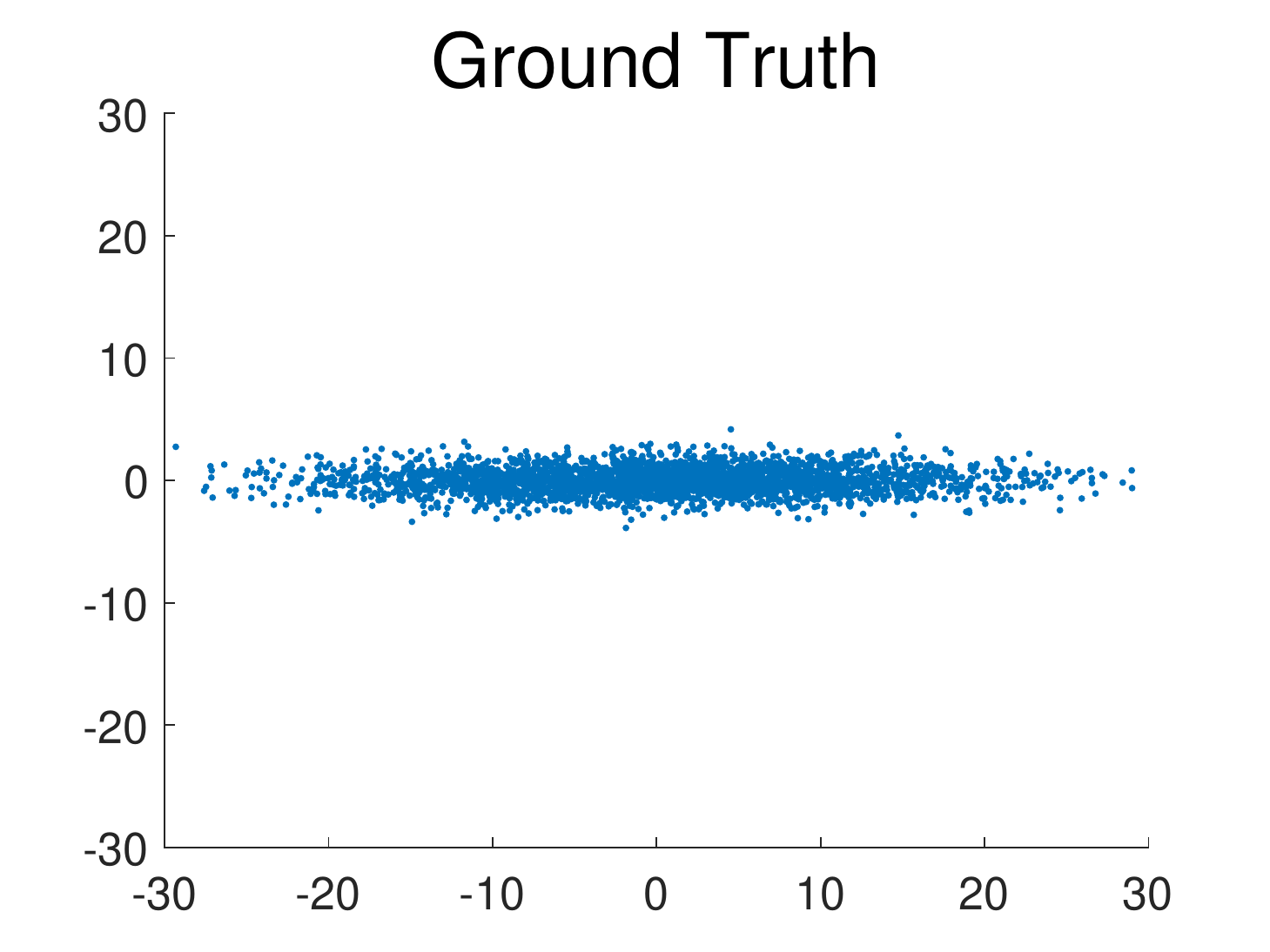}
            \caption[]%
            {{\small}}    
        \end{subfigure}
        \hfill
        \begin{subfigure}[b]{0.32\textwidth}  
            \centering 
            \includegraphics[width=\textwidth]{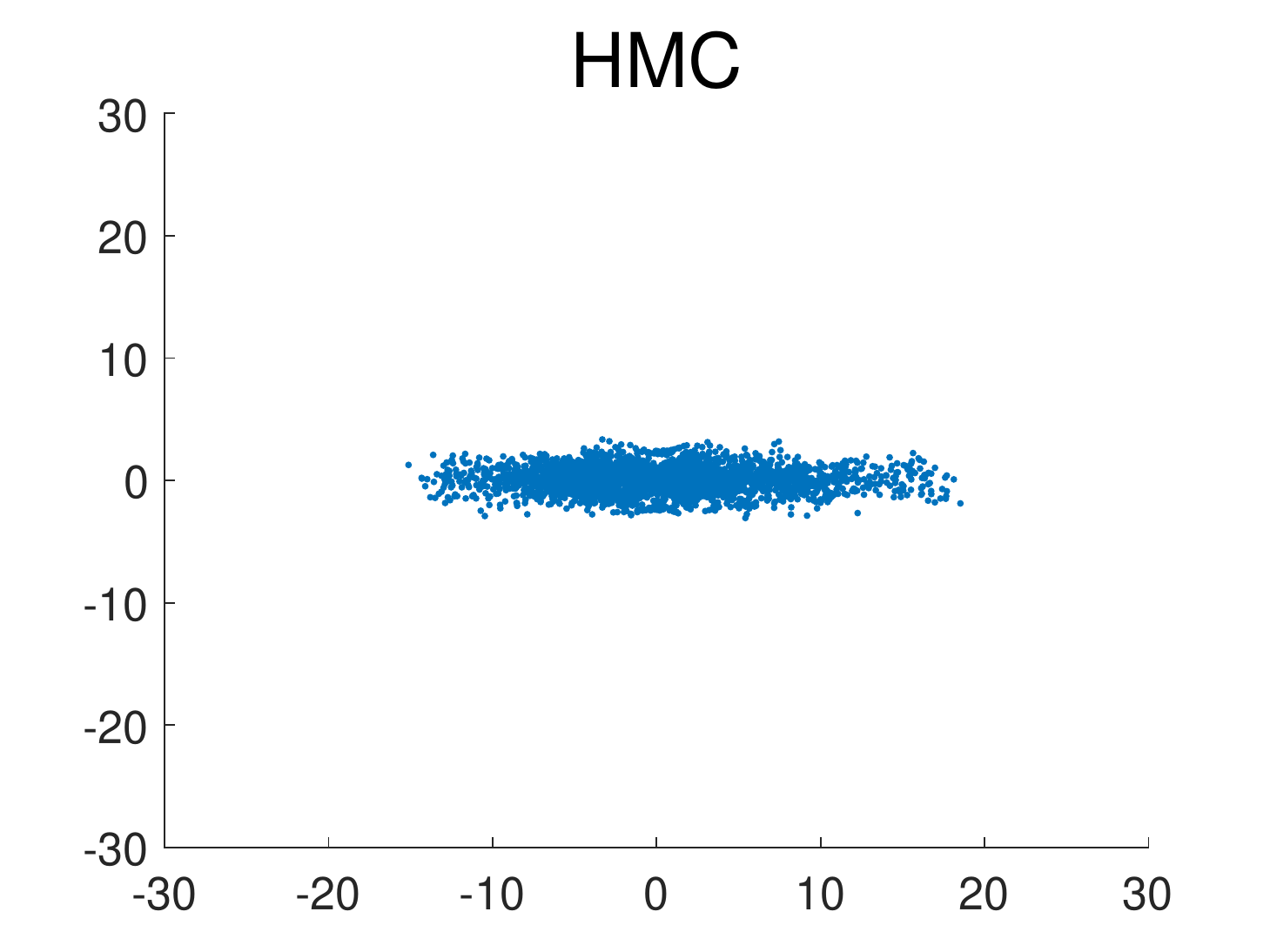}
            \caption[]%
            {{\small}}    
        \end{subfigure}
        \hfill
        \begin{subfigure}[b]{0.32\textwidth}  
            \centering 
            \includegraphics[width=\textwidth]{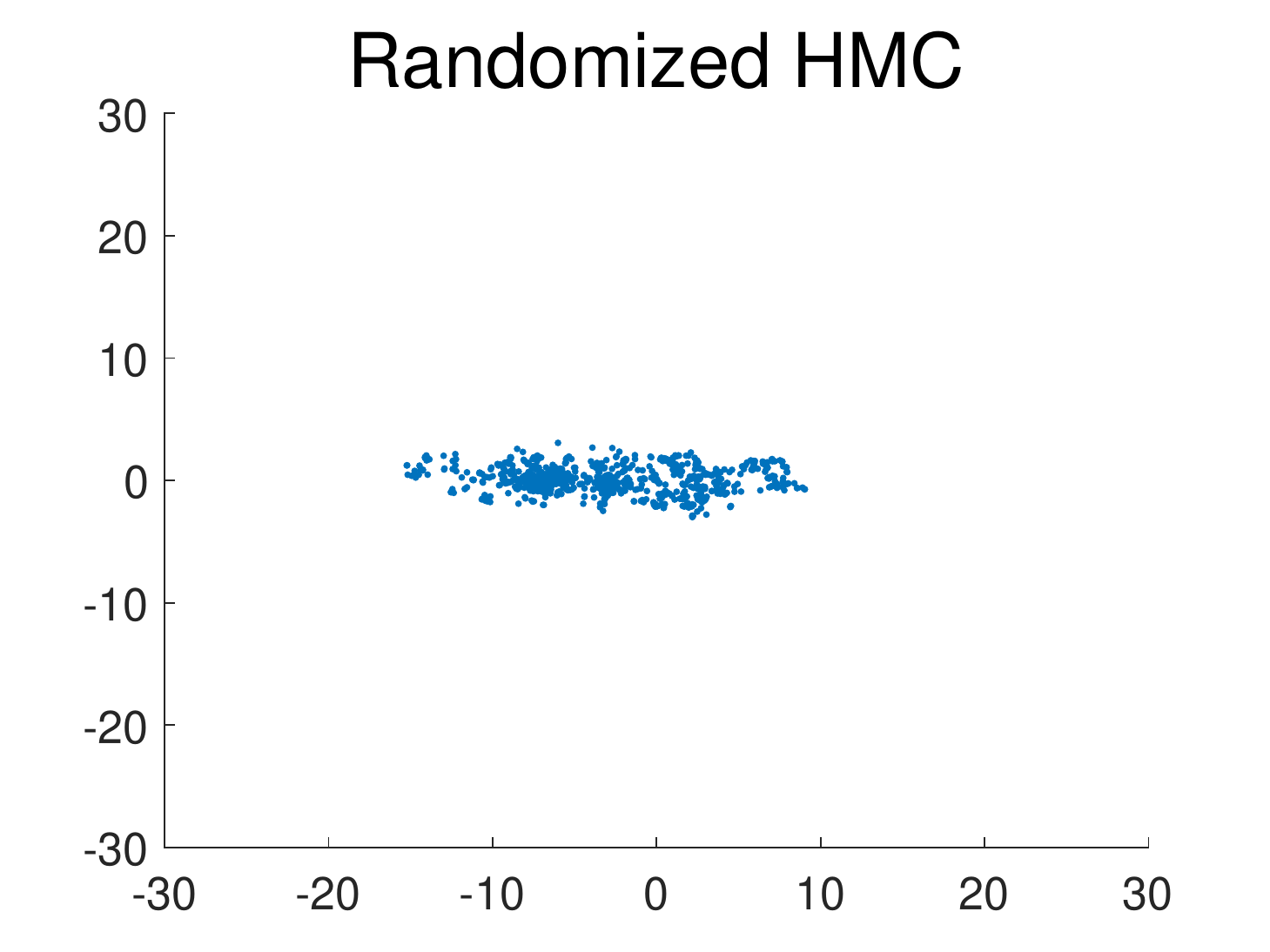}
            \caption[]%
            {{\small}}    
        \end{subfigure}
        \vskip\baselineskip
        \begin{subfigure}[b]{0.32\textwidth}   
            \centering 
            \includegraphics[width=\textwidth]{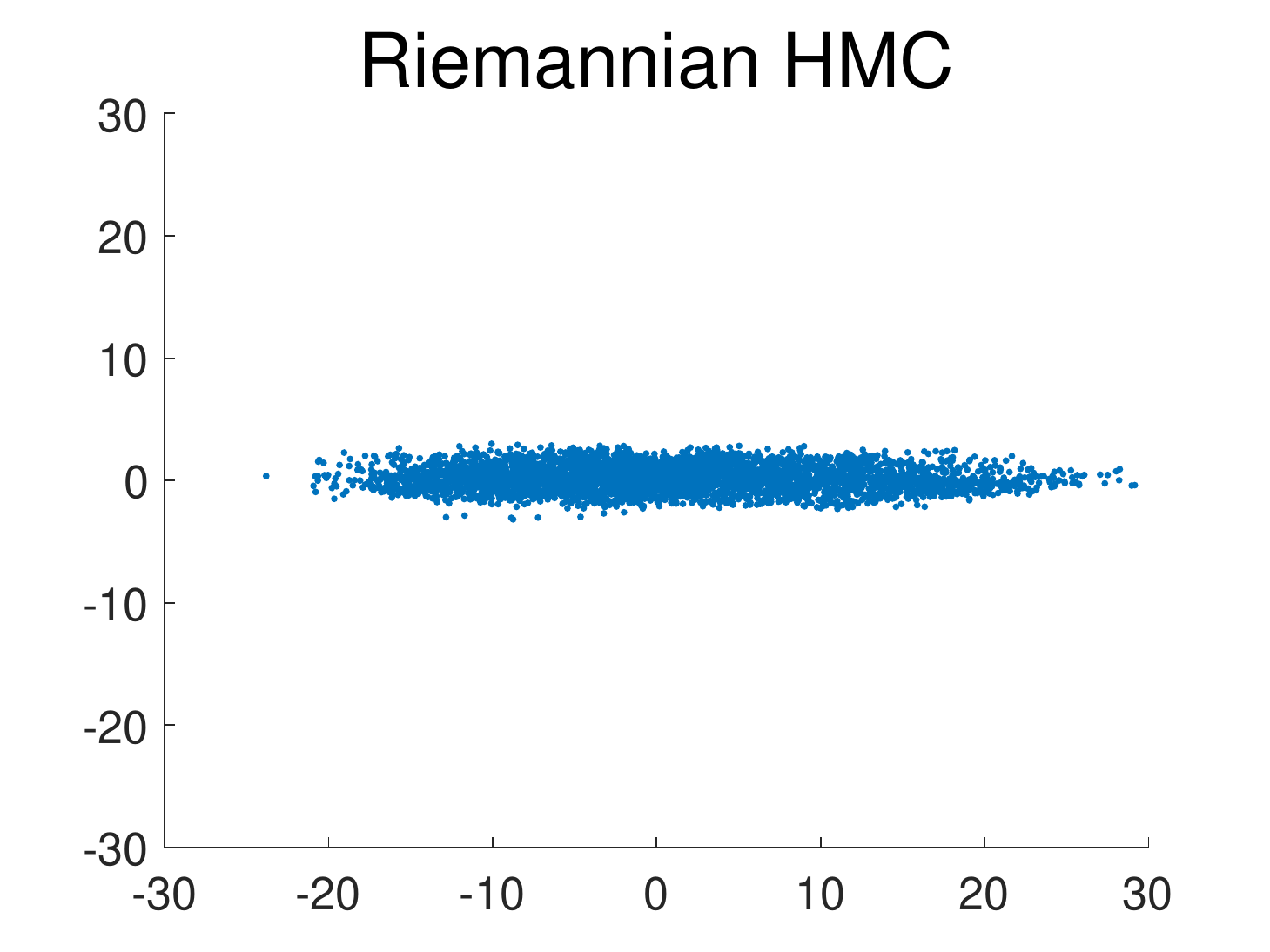}
            \caption[]%
            {{\small}}    
        \end{subfigure}
        \hfill
        \begin{subfigure}[b]{0.32\textwidth}   
            \centering 
            \includegraphics[width=\textwidth]{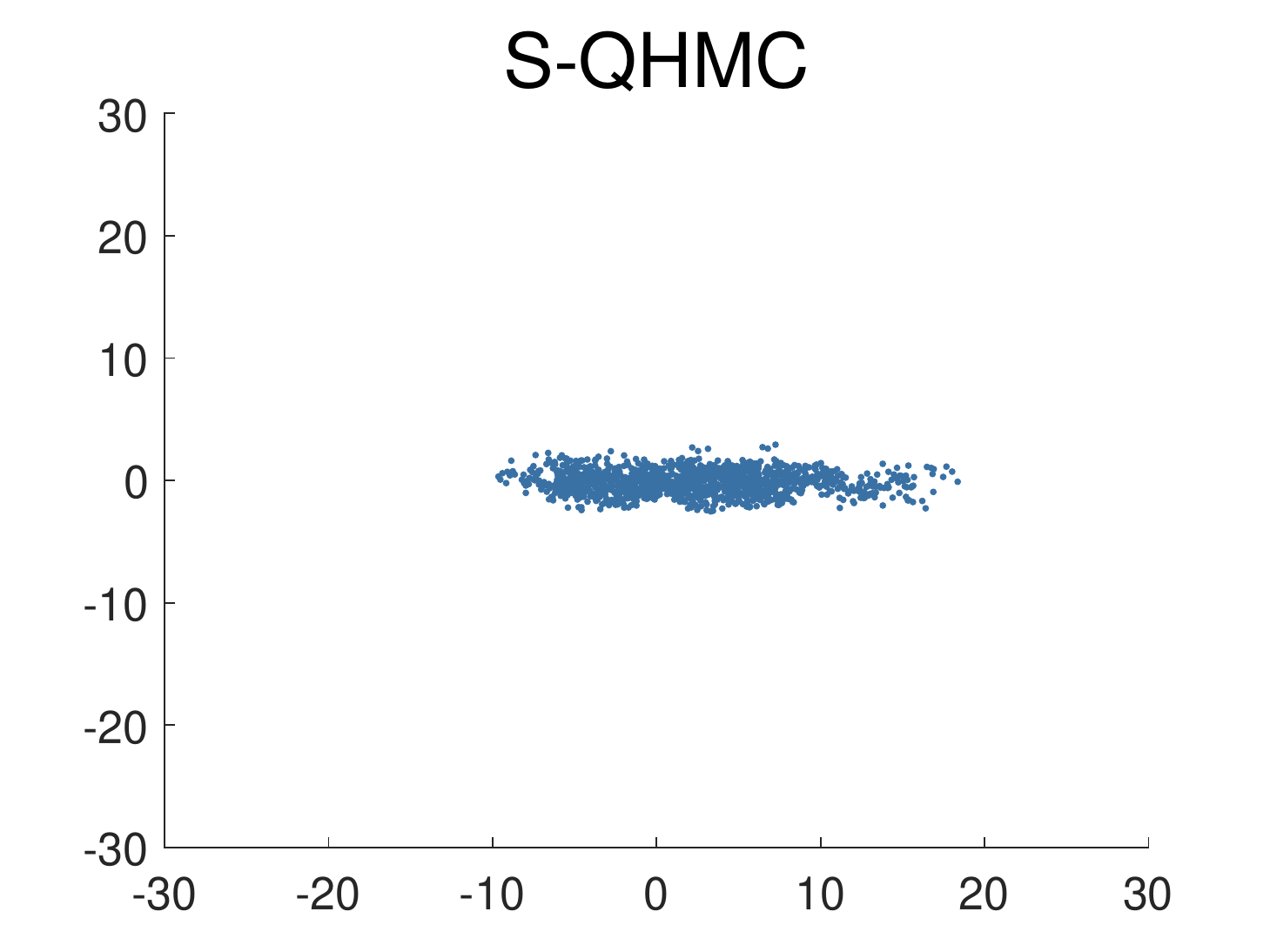}
            \caption[]%
            {{\small}}    
        \end{subfigure}
        \hfill
        \begin{subfigure}[b]{0.32\textwidth}   
            \centering 
            \includegraphics[width=\textwidth]{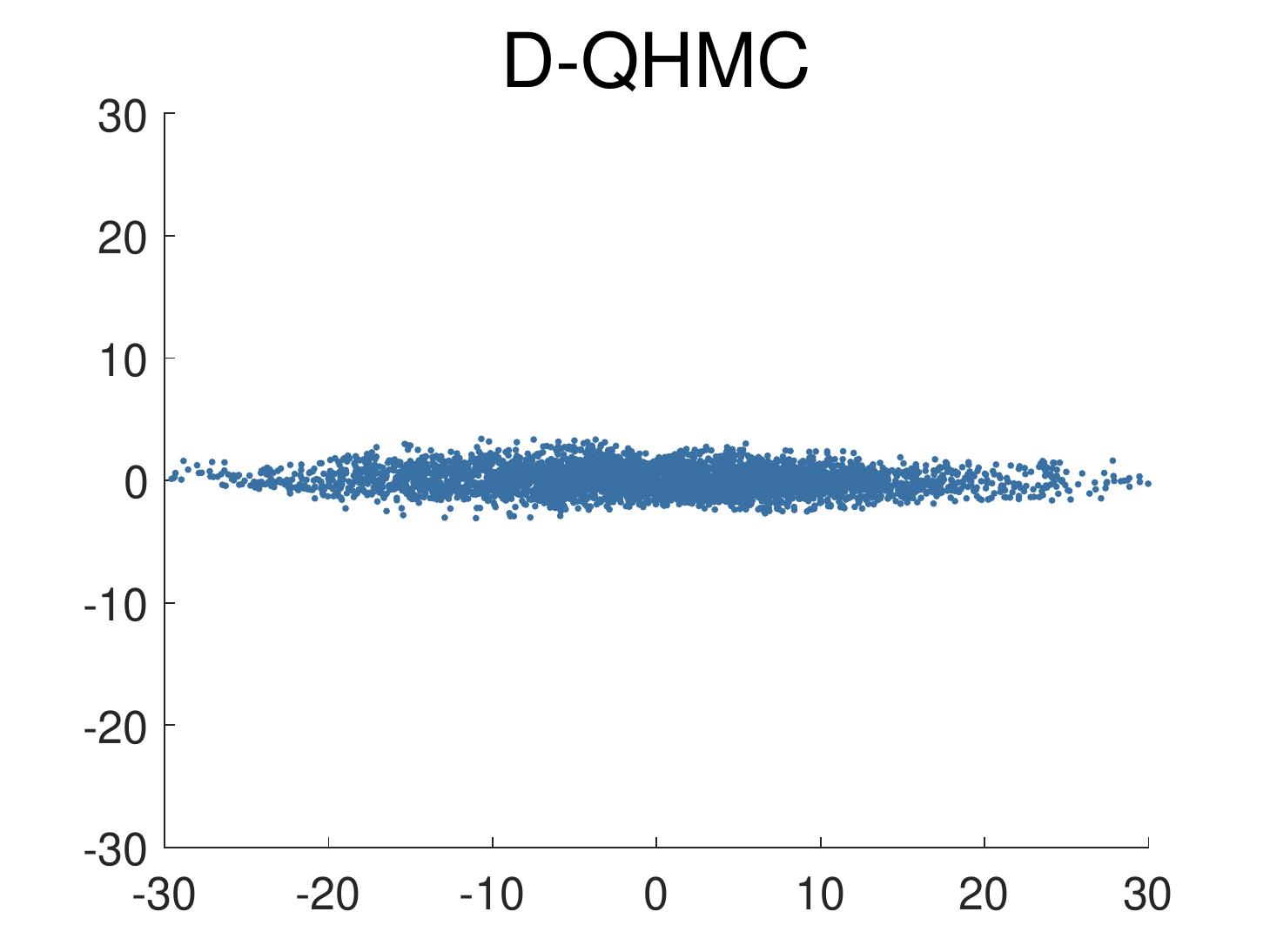}
            \caption[]%
            {{\small}}    
        \end{subfigure}
        \caption{An ill-conditioned Gaussian distribution. (a) shows the samples drawn from the ground truth. (b)-(f) five methods used to sample the distribution: HMC, randomized HMC, riemannian HMC, S-QHMC and D-QHMC. D-QHMC outperforms S-QHMC and other baselines.}
        \label{fig:g2d}
        \vspace{-10pt}
    \end{figure}

\subsubsection{Two-Dimensional Gaussian Mixture Model}\label{exp:2mm}


In this example, we show how the multi-mode QHMC (i.e., M-QHMC) described Section~\ref{sec:discuss} can help sample from a Gaussian-mixture distribution. 

We consider the target two-modal distribution:
\begin{equation}
    p(\mat{x})= p_1 {\rm exp}(-\frac{1}{2}\mat{x}^T\mat{\Sigma}_1^{-1}\mat{x})+p_2 {\rm exp}(-\frac{1}{2}\mat{x}^T\mat{\Sigma}_2^{-1}\mat{x})
\end{equation}
with $p_1=p_2=\frac{1}{2} $, $\mat{\Sigma}_1={\rm diag}(1,100)$ and $\mat{\Sigma}_2={\rm diag}(100,1)$. We intend to compare five sampling methods: (1) baseline HMC with $\mat{M}=m\mat{I}$ ($m=0.02$); (2) randomized HMC with $m=0.02$ and $\tau=0.03$; (3) Riemannian HMC with $\mat{M}=0.1\mat{\Sigma}_1^{-1}$; (4) D-QHMC with $\mu_m^{(1)}=\mu_m^{(2)}=-2, \sigma_m^{(1)}=\sigma_m^{(2)}=1$; (5) M-QHMC with $P_\mat{M}(\mat{M})=\frac{1}{2}\delta(\mat{M}-0.1\mat{\Sigma}_1^{-1})+\frac{1}{2}\delta(\mat{M}-0.1\mat{\Sigma}_2^{-1})$.

\begin{figure}[t]
        \centering
        \begin{subfigure}[b]{0.32\textwidth}
            \centering
            \includegraphics[width=\textwidth]{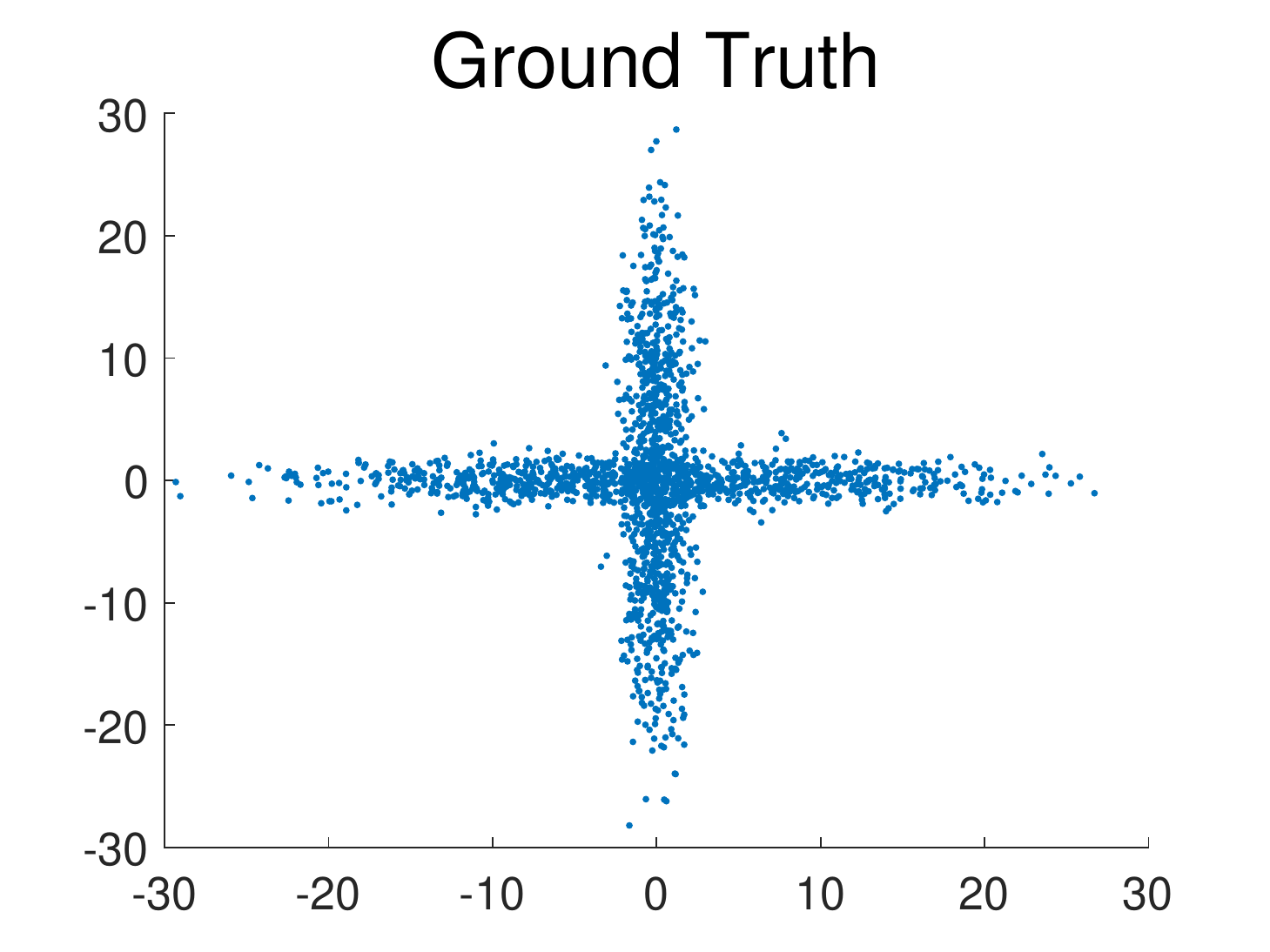}
            \caption[]%
            {{\small}}    
        \end{subfigure}
        \hfill
        \begin{subfigure}[b]{0.32\textwidth}  
            \centering 
            \includegraphics[width=\textwidth]{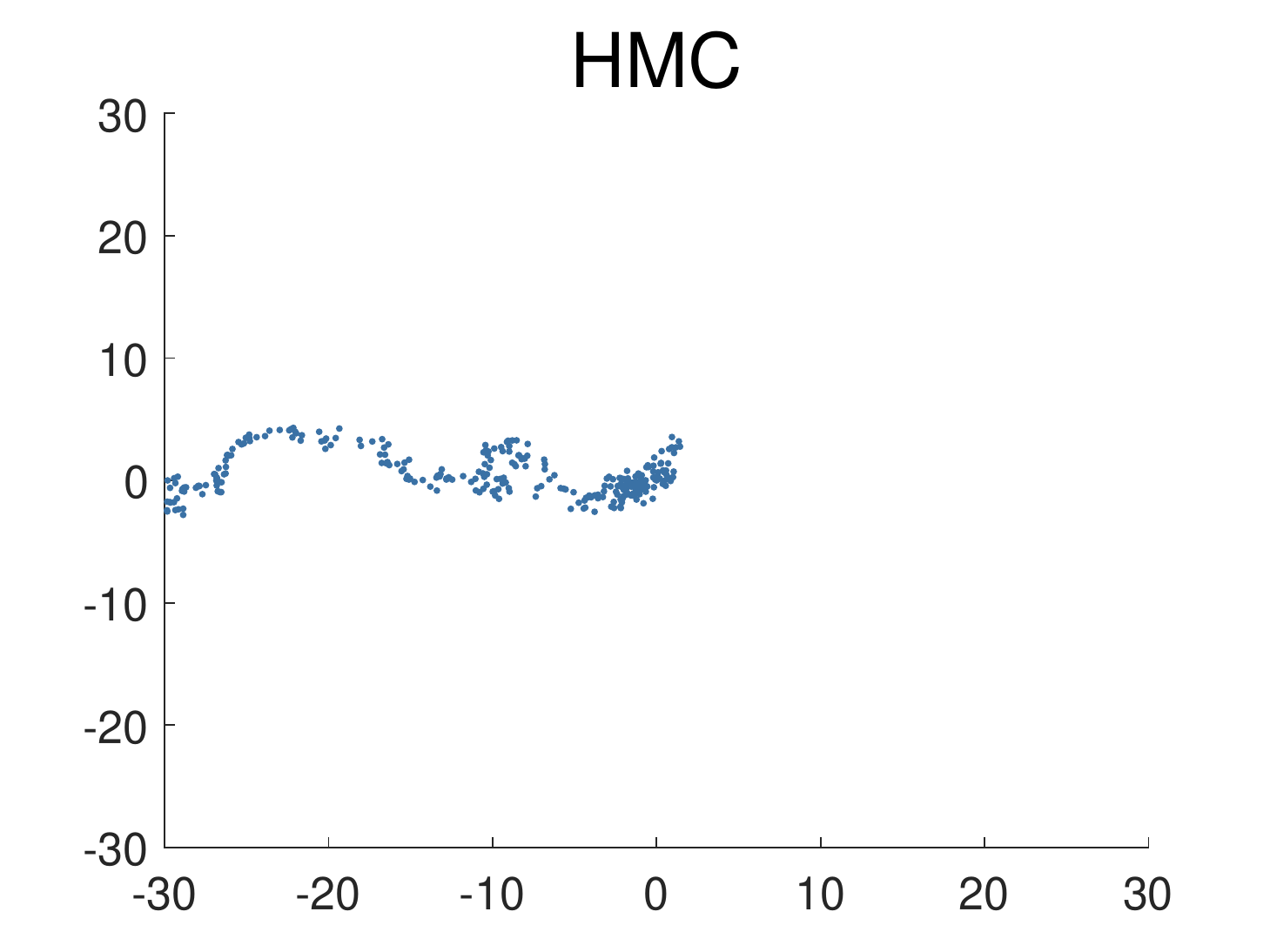}
            \caption[]%
            {{\small}}    
        \end{subfigure}
        \hfill
        \begin{subfigure}[b]{0.32\textwidth}  
            \centering 
            \includegraphics[width=\textwidth]{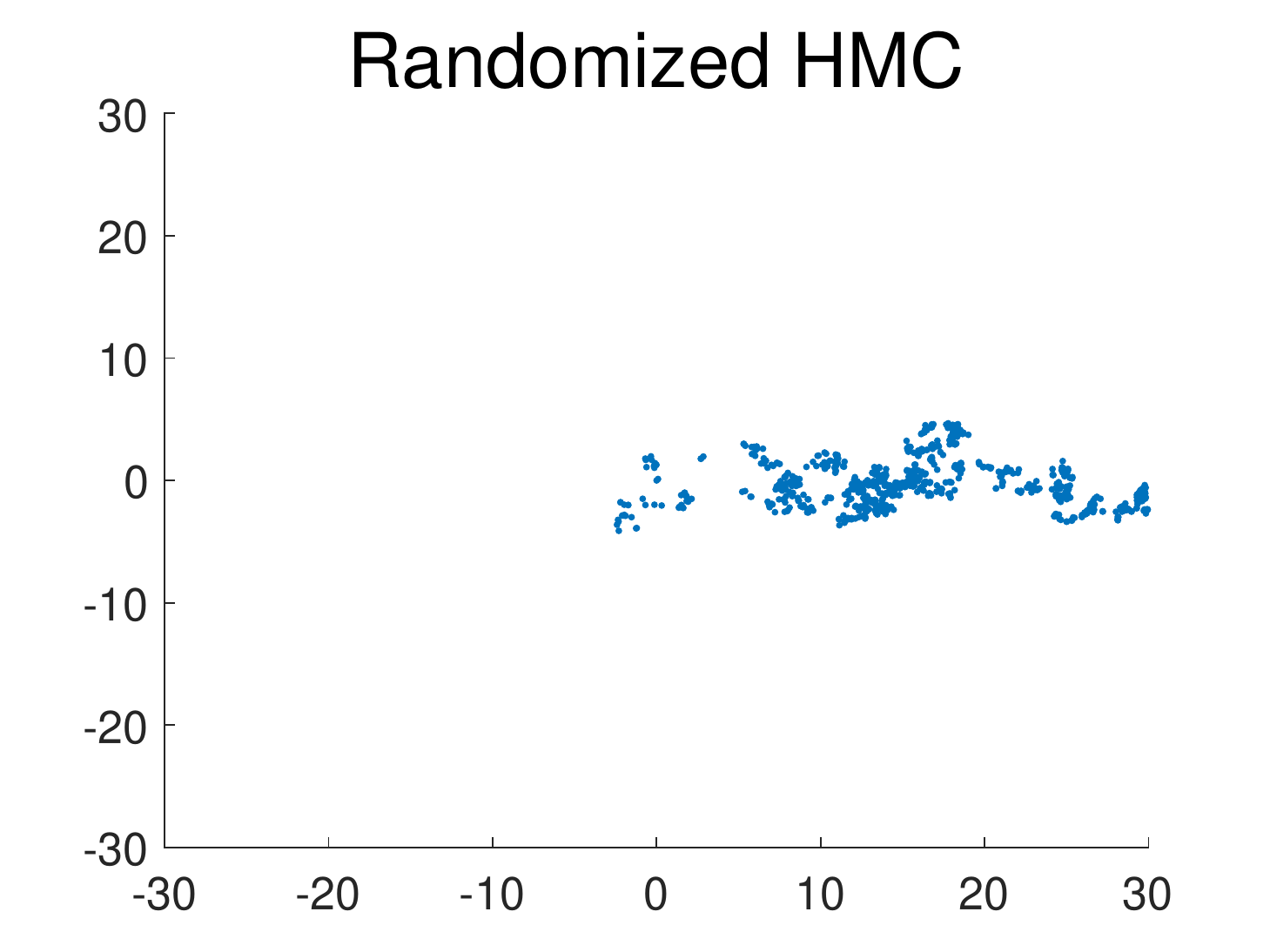}
            \caption[]%
            {{\small}}    
        \end{subfigure}
        \vskip\baselineskip
        \begin{subfigure}[b]{0.32\textwidth}   
            \centering 
            \includegraphics[width=\textwidth]{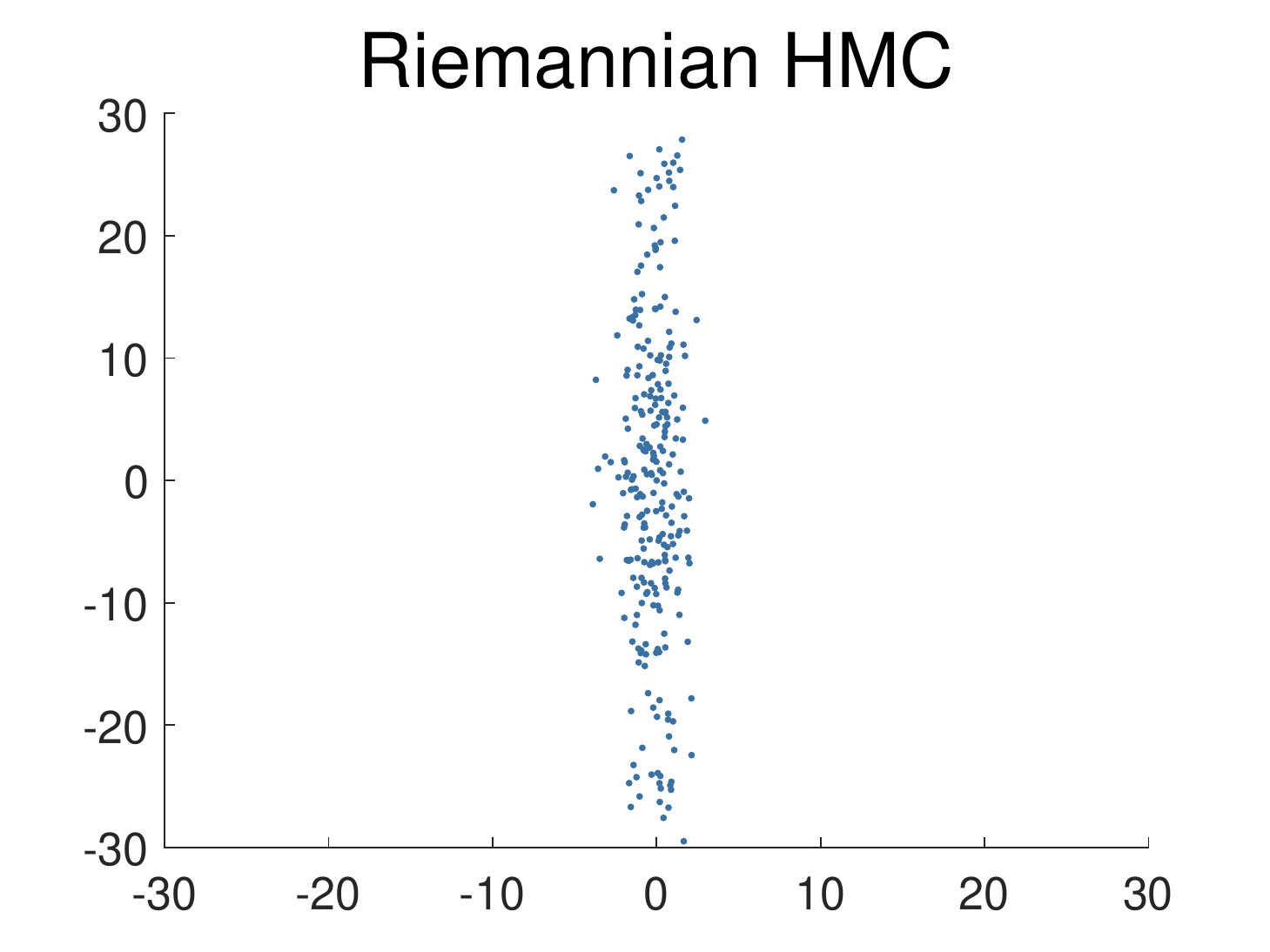}
            \caption[]%
            {{\small}}    
        \end{subfigure}
        \hfill
        \begin{subfigure}[b]{0.32\textwidth}   
            \centering 
            \includegraphics[width=\textwidth]{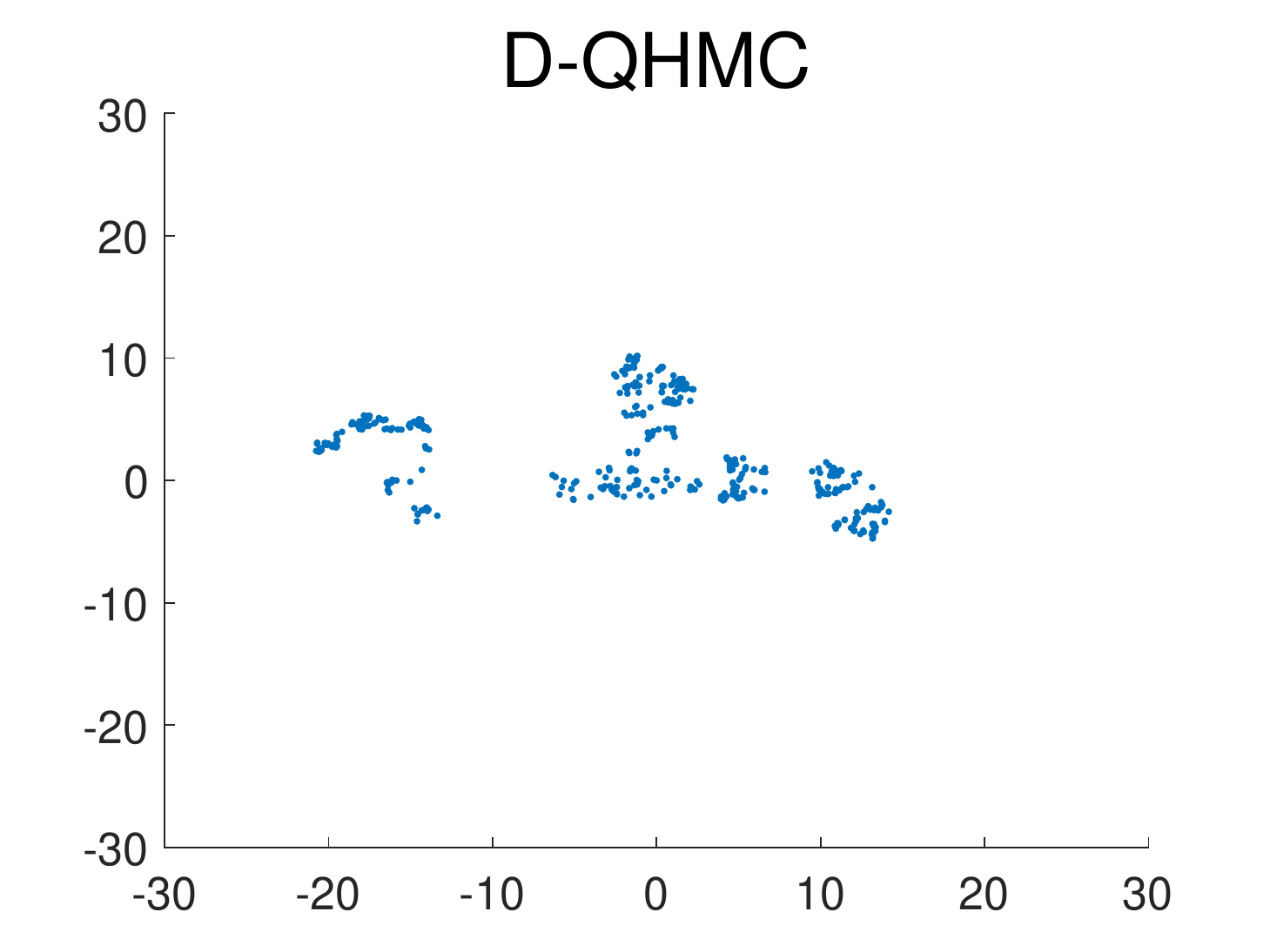}
            \caption[]%
            {{\small }}    
        \end{subfigure}
        \hfill
        \begin{subfigure}[b]{0.32\textwidth}   
            \centering 
            \includegraphics[width=\textwidth]{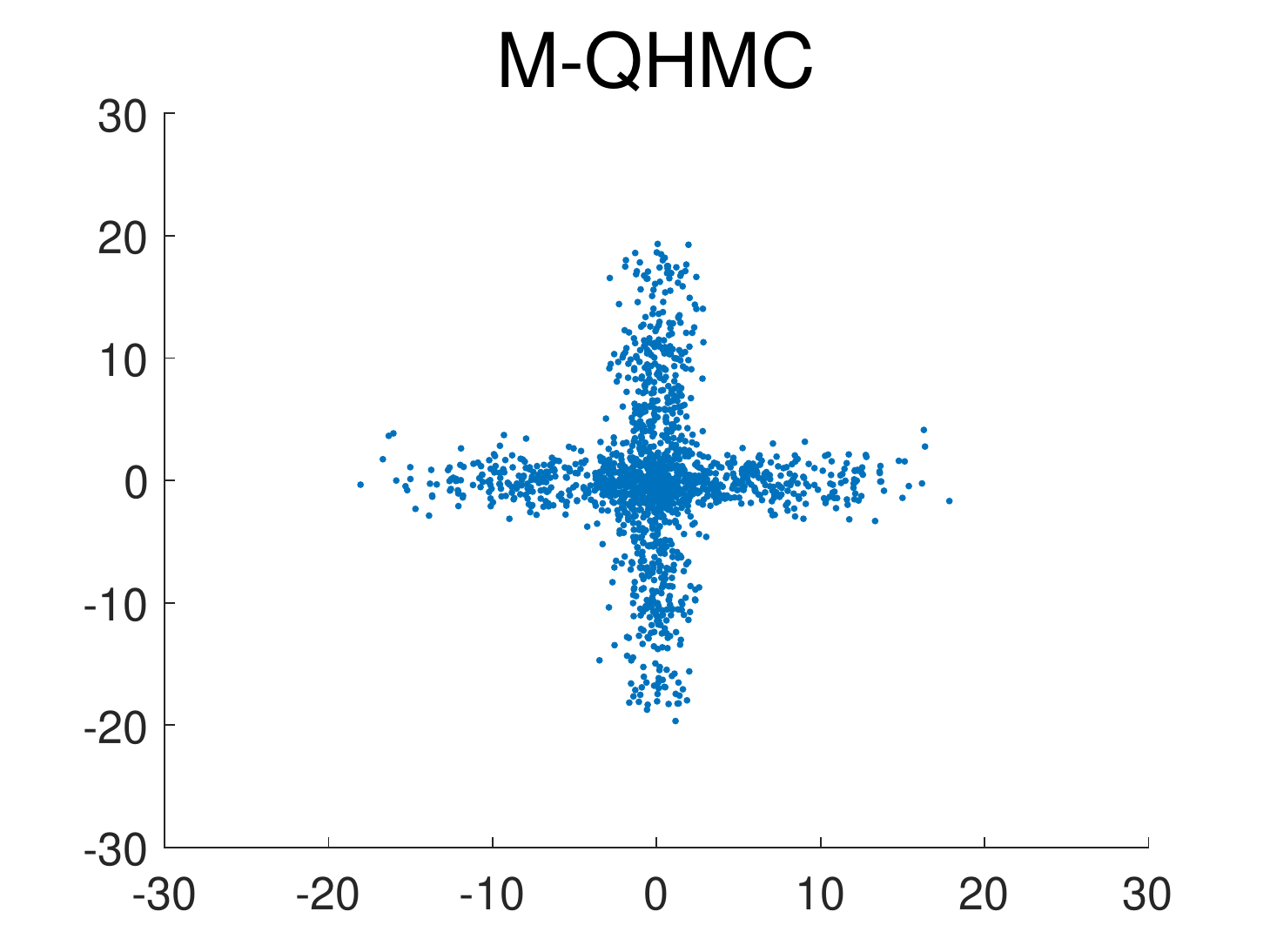}
            \caption[]%
            {{\small}}    
        \end{subfigure}
        \caption{A two-dimensional Gaussian mixture example. (a) samples drawn from the ground truth distribution. (b)-(f) five methods used for sampling: HMC, randomized HMC, riemannian HMC, D-QHMC and M-QHMC. Only M-QHMC explores both Gaussians efficiently.}
        \label{fig:mm2d}
        \vspace{-10pt}
    \end{figure}
    
The results are shown in Fig.~\ref{fig:mm2d}. Among the five methods, only M-QHMC can generate accurate sample distributions. 

\subsection{Application: Sparse Modeling via Bridge Regression}

In data mining and machine learning, the following loss function is often minimized
\begin{equation}
    L(\boldsymbol{\beta})=\frac{\mu}{2n}||\mat{y}-\mat{X}\boldsymbol{\beta}||_F^2+\lambda ||\boldsymbol{\beta}||_p^p=L_{ls}(\mat{\beta})+L_{re}(\mat{\beta})
\end{equation}
in order to learn a model. In Bayesian learning, this is equivalent to a linear model with Gaussian noise and $\ell_p$ prior. The likelihood function and prior distribution are $p({\cal D}|\mat{\beta})\propto {\rm exp}(-L_{ls}(\beta))$ and $p(\mat{\beta})\propto {\rm exp}(-L_{re}(\mat{\beta}))$ respectively, thus the posterior distribution is $p(\mat{\beta}|{\cal D})\propto p(\mat{\beta})p({\cal D}|\mat{\beta})\propto {\rm exp}(-L(\beta))$. In ``bridge regression"~\citep{polson2014bayesian,armagan2009variational}, the parameter $p$ is chosen in the range $(0,1)$ in order to select proper features and to enforce model sparsity. In a Bayesian stetting, this is equivalent to placing a prior ${\rm exp} (-\| \boldsymbol{\beta} \|_p^p/{\tau})$ over the unknown model paramters $\boldsymbol{\beta}$.

In this experiment, we consider the case $p=1/2$ and perform bridge regression using the Stanford diabetes dataset~\footnote{https://web.stanford.edu/~hastie/Papers/LARS/diabetes.data}. This dataset includes $n=442$ people, 10 attributes (AGE, SEX, BMI, BP, S1-S6) and 1 health indicator (Y). The goal is to select as few attributes as possible but they still accurately predict the target Y. We split the dataset into a training set with 300 people and a testing set with 142 people. 
The hyper-parameters $\mu$ and $\lambda$ can be automatically determined if by further introducing some hyper-priors over $\lambda$ and $\mu$~\citep{polson2014bayesian,armagan2009variational}. However, the non-existence of analytical conjugate priors~\citep{armagan2009variational} makes updates (Gibbs sampling) of $\mu$ and $\lambda$ only approximate, but not exact. For the sake of simplicity, we utilize a grid search method in the $(\lambda,\mu)$ plane with the testing MSE (mean-squared error) as a criteria to choose the hyper-parameters. We start from $\boldsymbol{\beta}=\mat{0}$ and use 1000 paths of gradient descent as a burn-in process. After that, we run another 1000 paths of HMC/QHMC and collect the samples. The results are shown in Table \ref{tab:sp-test-mse} for both HMC and QHMC. When the regularization $\lambda$ is small ($\lambda=0.1,1,10$), the difference between HMC and QHMC is insignificant; however, when large regularization is required for very sparse models (e.g. $\lambda=100$ or $1000$), QHMC can produce models with higher accuracy.

\begin{table}[t]
\caption{Testing mean-squared error (MSE) for the bridge regression task}
\begin{center}
\begin{tabular}{|c|c|c|c|c|c|}
    \hline
    \multicolumn{2}{|c|}{\multirow{2}{*}{HMC/QHMC}}&\multicolumn{4}{c|}{$\mu$}\\\cline{3-6}
    \multicolumn{2}{|c|}{}&0.1&1&10&100\\
    \hline
    \multirow{5}{*}{$\lambda$}&0.1&4.20/5.14&0.68/0.69&0.29/0.29&0.25/0.25\\
    \cline{2-6}
    &1&0.71/1.05&0.52/0.52&0.30/0.30&0.26/0.25\\
    \cline{2-6}
    &10&0.51/0.51&0.51/0.46&0.51/0.49&{\bf 0.31/0.28}\\
    \cline{2-6}
    &100& {\bf 0.64/0.48}&{\bf 0.88/0.58}&{\bf 0.72/0.50}&{\bf 0.69/0.52}\\
    \cline{2-6}
    &1000& 3.9/2.7&30.4/23.6&15.9/9.3& {\bf 2.4/0.8}\\
    \hline
\end{tabular}
\end{center}
\label{tab:sp-test-mse}
\vspace{-15pt}
\end{table}

\begin{table}[t]
\caption{CPU time and accuracy comparison of QHMC with some baseline methods.}
\begin{center}
\begin{tabular}{|c|c|c|c|c|}\hline
     & S-QHMC/D-QHMC & HMC & RMHMC & NUTS \\\hline
    CPU time(s) & {\bf 3.9/4.1} & 3.8 & 9.8 & 98.3 \\\hline
    test MSE & {\bf 0.28/0.29} & 0.31 & 0.98 & 0.28\\\hline
\end{tabular}
\end{center}
\label{tab:sp-baseline}
\vspace{-5pt}
\end{table}

\begin{figure}[t]
    \centering
        \begin{subfigure}[t]{0.8\textwidth}
            \centering
            \includegraphics[trim = 40mm 0mm 30mm 0mm,width=\textwidth]{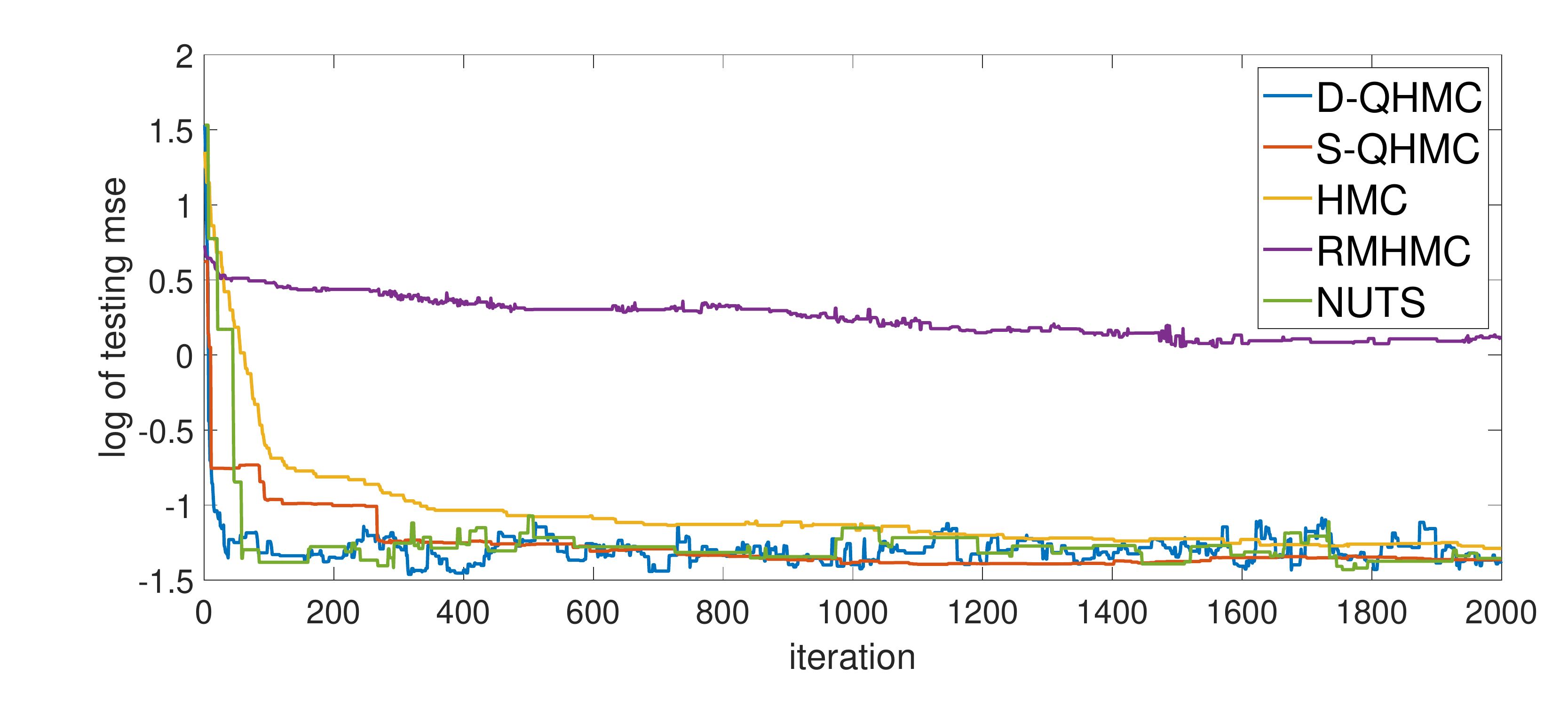}
            \caption[Network2]%
            {{\small}}    
        \end{subfigure}
        \hfill
        \begin{subfigure}[t]{0.9\textwidth}  
            \centering 
            \includegraphics[trim = 30mm 0mm 60mm 0mm, clip,width=\textwidth]{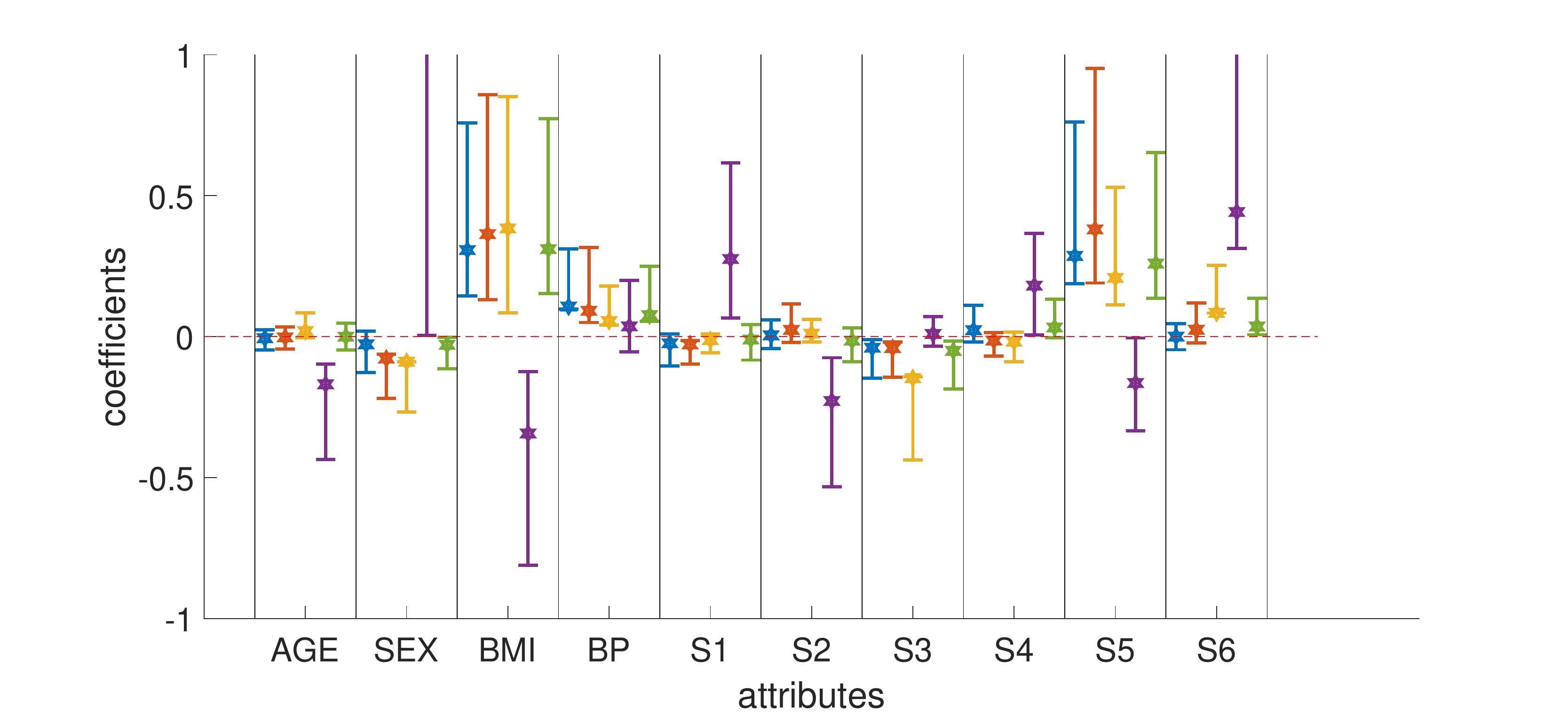}
            \caption[]%
            {{\small}}    
        \end{subfigure}
    \caption{(a) The log testing error as a function of the iteration step. The number of steps required by burn-in process in QHMC ($\sim 500$ steps) is much smaller than that in HMC ($\sim 2000$ steps). (b) Distribution of all attributes in the diabetes dataset using QHMC, HMC, NUTS and RMHMC with $\mu=100$ and $\lambda=10$. QHMC is better at sampling sparser models than standard HMC (QHMC produces smaller variance in coefficients) and has smaller MSE on the testing data.}
    \label{fig:sp_reg}
    \vspace{-10pt}
\end{figure}
We consider the case with $\lambda=10$ and $\mu=100$, and compare S-QHMC and D-QHMC against the standard HMC, NUTS~\citep{hoffman2014no} and RMHMC~\citep{girolami2011riemann}. The resulting test MSE and CPU time are reported in Table \ref{tab:sp-baseline}. RMHMC is implemented with $(\mat{X}^T\mat{X})$ (with $\mat{X}$ being the data matrix) being the metric of the Riemannian manifold in~\citep{girolami2011riemann}, and NUTS is implemented based on Alg.~2 in~\citep{hoffman2014no}. Although RMHMC is an adaptive HMC, it tends to degrade for spiky energy functions. Because the $\ell_p$ penalty is isotropic (i.e., the regularization coefficients for all attributes are the same), imposing an anisotropic metric renders the sampling performance even worse than HMC. NUTS can achieve comparable test MSE with QHMC, nontheless it consumes $25\times$ more CPU time than QHMC due to the expensive recursions of building balanced binary trees. Compared with HMC, QHMC has three advantages. Firstly, QHMC is more accurate as shown by the lower test MSE. Secondly, QHMC has a much shorter burn-in phase than HMC, as shown in Fig.~\ref{fig:sp_reg} (a). This is expected because the possible small masses in QHMC can speed up the burn-in phase. Thirdly, QHMC can produce sparser models than HMC as expected. As shown in Fig.~\ref{fig:sp_reg} (b), seven attributes have nearly zero mean (except BMI, S3 and S5). The samples from QHMC have smaller variance for these seven attributes, implying a more spiky posterior sample distribution of $\boldsymbol{\beta}$. Besides, D-QHMC achieves similar CPU time and test MSE with S-QHMC, but can produce sparser samples (see attributes SEX, S2 and S6 in Fig.~\ref{fig:sp_reg} (b)).

\subsection{Application: Image Denoising}\label{sec:image-denoising}

\begin{figure}[t]
    \centering
        \begin{subfigure}[b]{0.49\textwidth}
            \centering
            \includegraphics[trim = 45mm 120mm 70mm 108mm, clip,width=\textwidth]{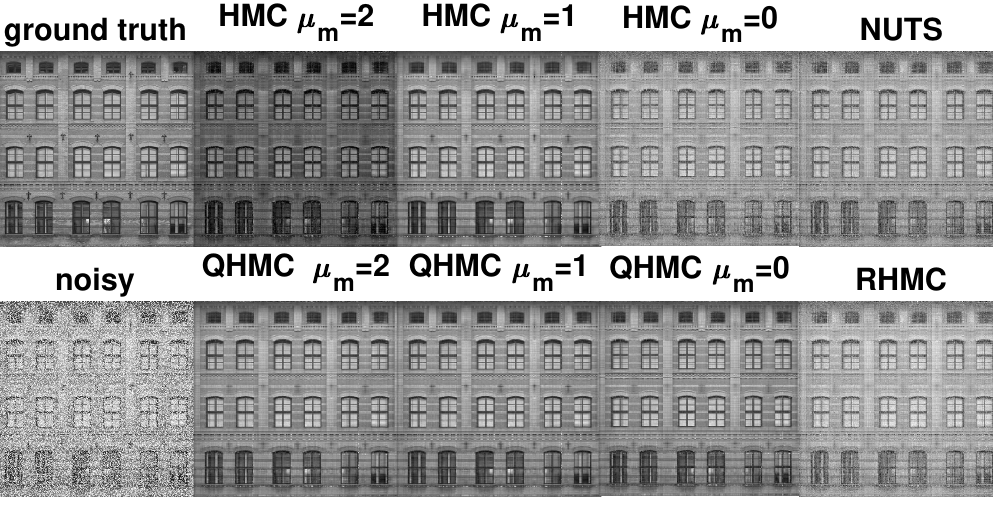}
            \caption[Network2]%
            {{\small}}    
        \end{subfigure}
        \hfill
        \begin{subfigure}[b]{0.46\textwidth}
            \centering
            \includegraphics[trim = 45mm 125mm 70mm 93mm, clip,width=\textwidth]{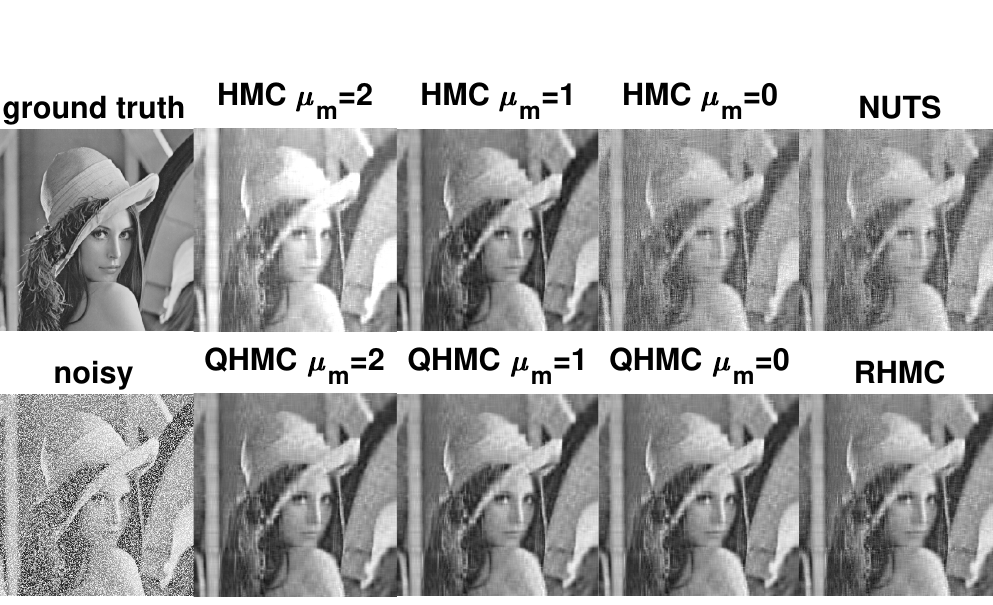}
            \caption[Network2]%
            {{\small}}    
        \end{subfigure}
        
        \begin{subfigure}[b]{0.32\textwidth}
            \centering
          \includegraphics[trim = 35mm 80mm 35mm 90mm, clip,width=1\textwidth]{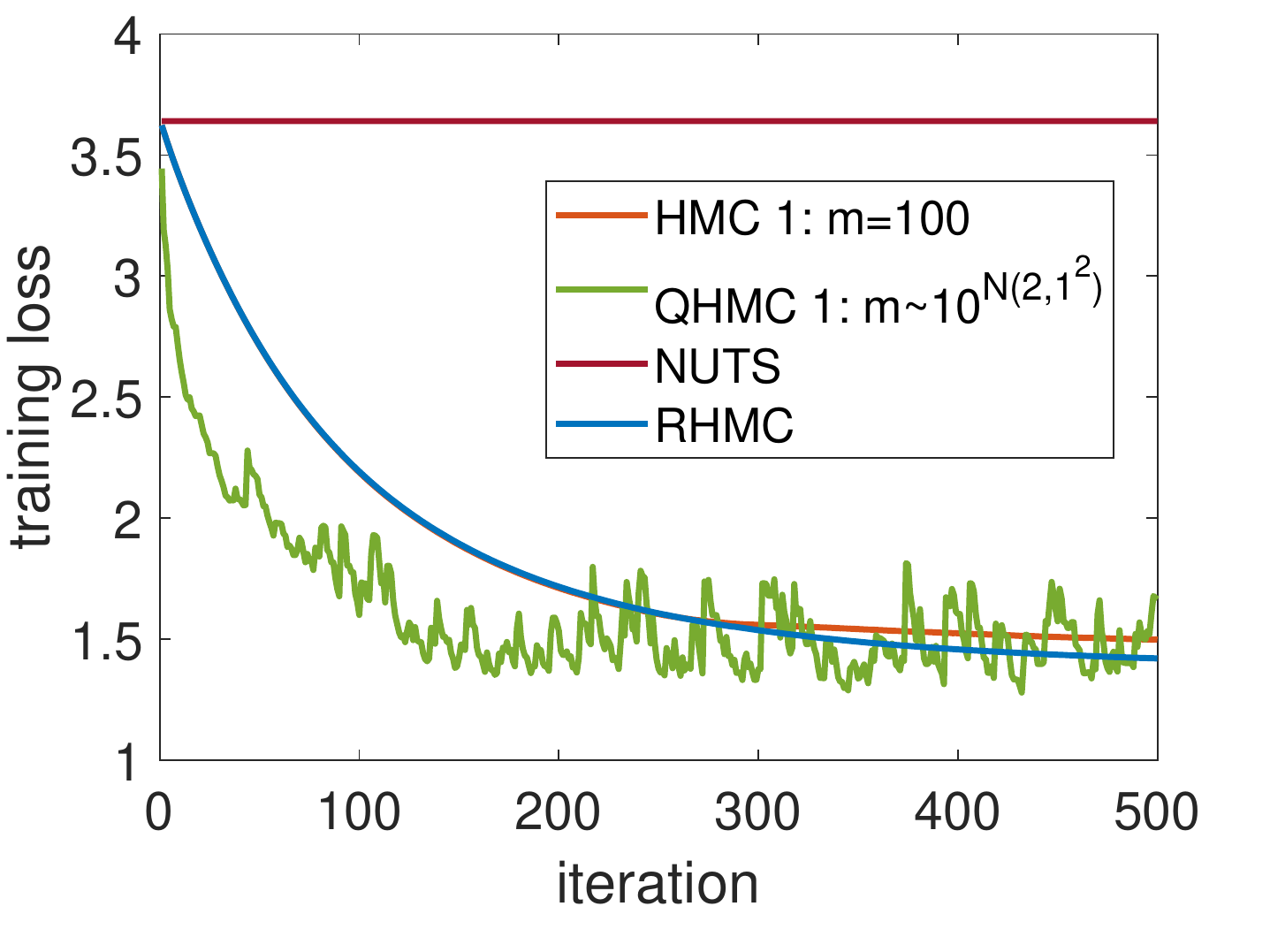}
          \caption[Network2]%
            {{\small}}    
        \end{subfigure}
        \hfill
        \begin{subfigure}[b]{0.32\textwidth}
            \centering
          \includegraphics[trim = 35mm 80mm 35mm 90mm, clip,width=1\textwidth]{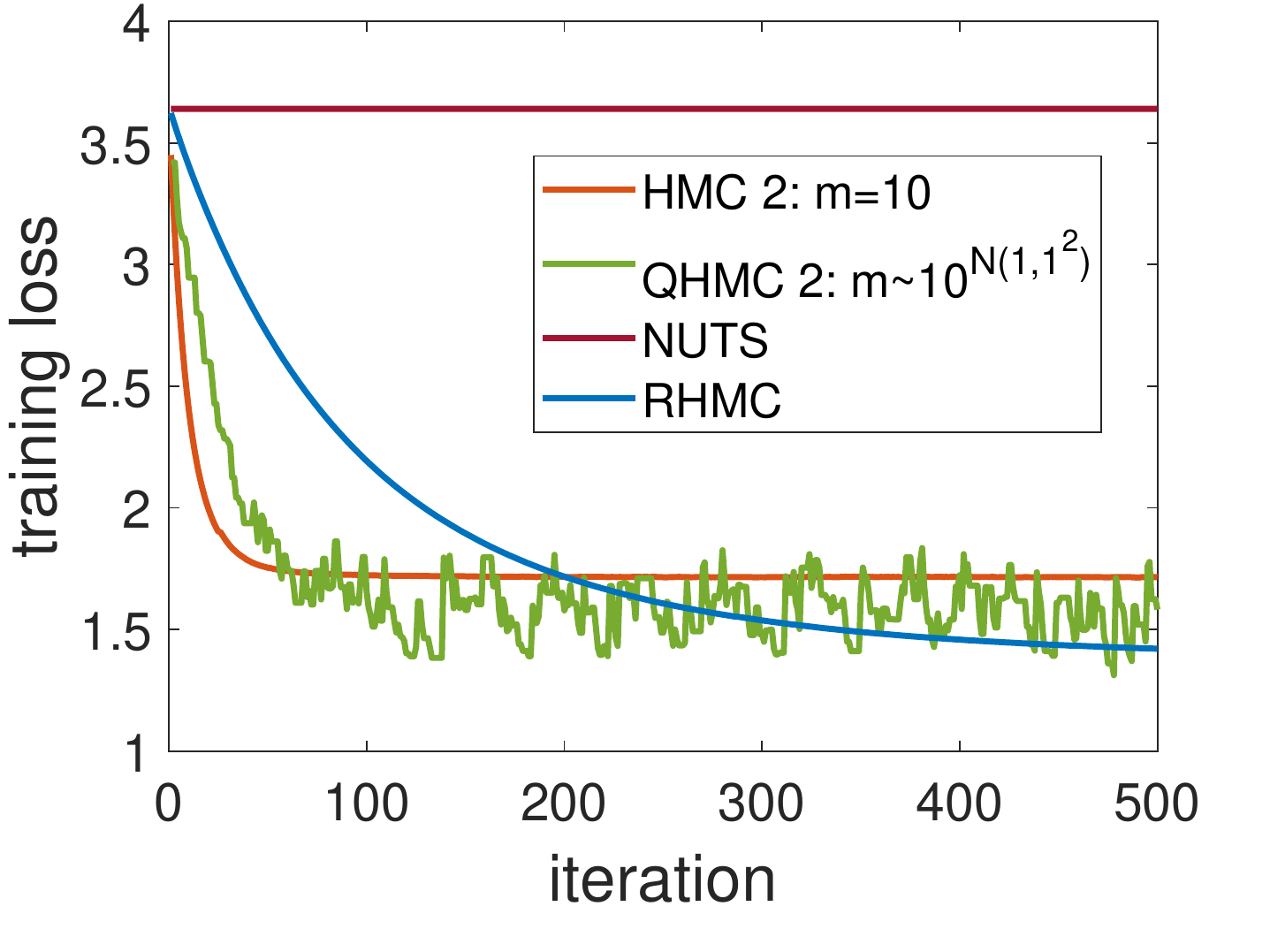}
          \caption[Network2]%
            {{\small}}    
        \end{subfigure}
        \hfill
        \begin{subfigure}[b]{0.32\textwidth}
            \centering
          \includegraphics[trim = 35mm 80mm 35mm 90mm, clip,width=1\textwidth]{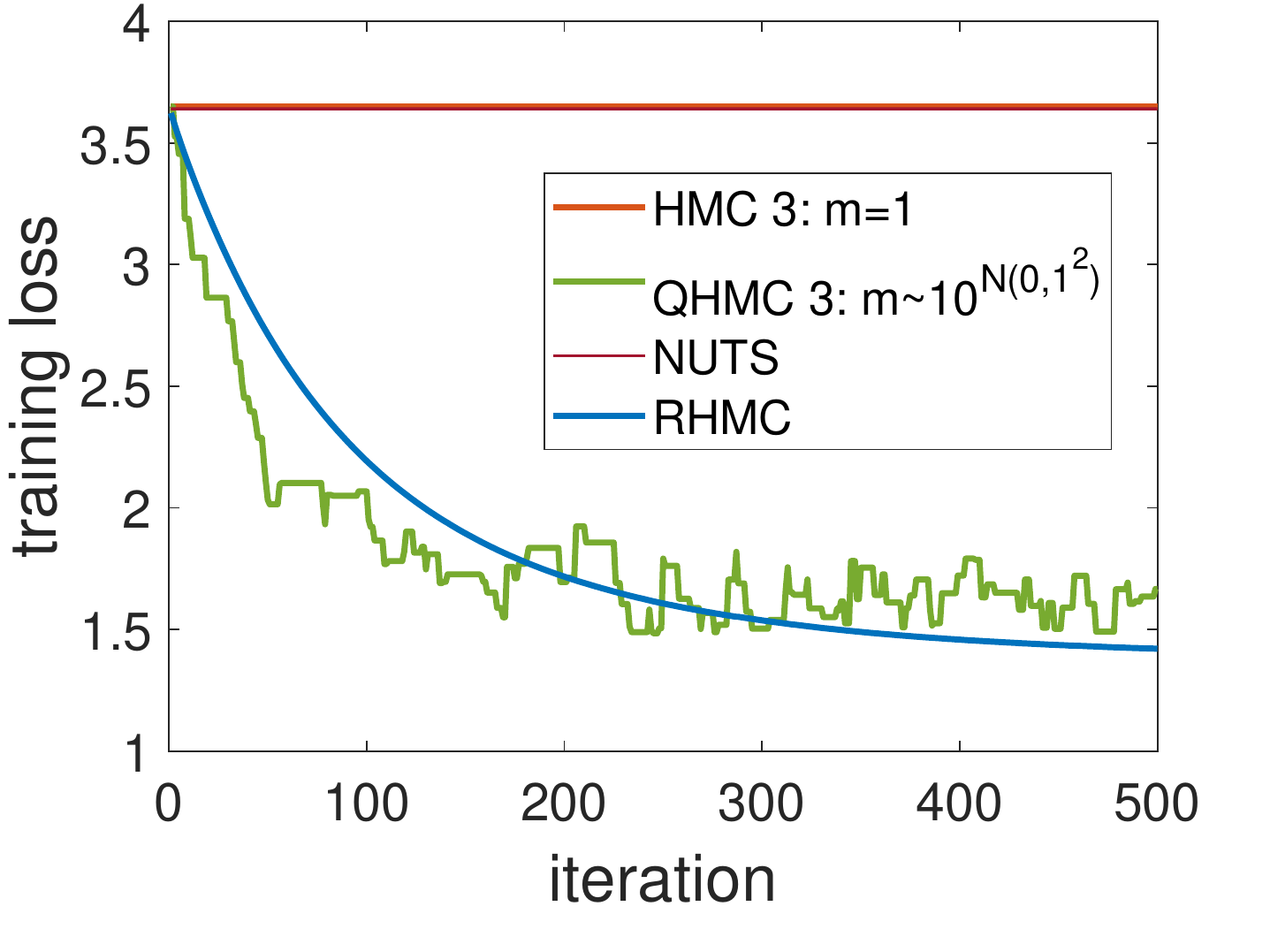}
          \caption[Network2]%
            {{\small}}    
        \end{subfigure}
    \caption[ The average and standard deviation of critical parameters ]
        {\small Image denoising example. Here we choose S-QHMC to implement QHMC. (a)(b) HMC can achieve good reconstruction only for the second parameters, while QHMC has good performance for all settings and its performance is much less sensitive to the choice of mass parameters. Also, Neither NUTS or RMHMC can achieve comparable results with QHMC. (c) For small step size (large mass), QHMC can speed up convergence in the large gradients phase; (d) For proper step size, HMC and QHMC show no significant difference. (e) For large step size (small mass), QHMC can still make progress in terms of decreasing the loss function, while HMC gets stuck at high loss.}
    \label{fig:mnist}
    \label{-15pt}
\end{figure}

In this experiment, we apply QHMC to solve an image denoising problem.

Specifically, we consider a two-dimensional gray-level image, but the extension to high-order tensors is available~\citep{sobral2015online}. We employ the robust low-rank matrix factorization model~\citep{candes2011robust} which models a corrupted image $\mat{Y}$ as the sum of a low-rank matrix $\mat{L}=\mat{A} \mat{B}$, a sparse matrix $\mat{S}$ describing outliers and some i.i.d. Gaussian noise. Standard robust PCA uses $\ell_1$-norm to enforce the sparsity of $\mat{S}$. Instead, we employ $\ell_p$ norm with $p=1/2$ because it is closer to the exact sparsity measurement $\ell_0$-norm. As a result, we can define the potential energy function as follows:

\begin{equation}
   U(\mat{A},\mat{B},\mat{S})=L_{g}+L_{l}+L_{n}=\frac{1}{2}\mu||\mat{Y}-\mat{AB}-\mat{S}||_F^2+\frac{1}{2}\lambda_1 (||\mat{A}||_F^2+||\mat{B}||_F^2)+\lambda_2||\mat{S}||_p^p
\end{equation}

In the Bayesian setting, ${\rm exp}(-L_{g})$ refers to the likelihood function; ${\rm exp}(-L_{l})$ refers to the prior density of the low rank part; ${\rm exp}(-L_{n})$ refers to the prior of the salt-and-pepper noise (which is sparse in nature). The gradient of the above loss function with respect to $\mat{A}$, $\mat{B}$ and $\mat{S}$ are:
\begin{equation}
\left\{
\begin{aligned}
\frac{\partial U}{\partial \mat{A}}&=\mu (\mat{AB}+\mat{S}-\mat{Y})\mat{B}^T+\lambda_1\mat{A} \\
\frac{\partial U}{\partial \mat{B}}&=\mu \mat{A}^T(\mat{AB}+\mat{S}-\mat{Y})+\lambda_1\mat{B}\\
\frac{\partial U}{\partial \mat{S}}&=\mu(\mat{AB}+\mat{S}-\mat{Y})+\lambda_2\frac{p_0}{\mat{S}^{1-p_0}+\epsilon_0}\mathrm{sign}(\mat{S}).
\end{aligned}
\right.
\end{equation}
In the second term of the gradient with respect to $\mat{S}$, all operations are element-wise, and $\mathrm{sign}(x)$ is defined as 
\begin{equation}
    \mathrm{sign}(x)=
    \left\{
\begin{aligned}
1,&\quad x\geq 0\\
-1,&\quad x<0.
\end{aligned}
\right.
\end{equation}

The parameters are fine tuned and set as $\mu=100$, $\lambda_1=1$, $\lambda_2=10$, $p=1/2$ and the rank $r=20$. It is possible to automatically determine the rank by introducing some hyper-priors (e.g.~\citep{zhao2015bayesian}), but here we focus on the sampling of a given model with a spiky loss function.

We use the proposed QHMC method (both D-QHMC and S-QHMC) to draw samples from the model, and we compare it with HMC, RMHMC~\citep{girolami2011riemann} and NUTS~\citep{hoffman2014no}. In all experiments we initialize the model with Singular Value Decomposition (SVD) of the corrupted image and run 300 iterations as burn-in and another 200 iterations to collect the sample images. In HMC and QHMC, we use three groups of parameter choices for the mass matrix. The metric of the Riemannian manifold~\citep{girolami2011riemann} in RMHMC is chosen as a diagonal matrix whose diagonal elements are the singular values of the corrupted image. The NUTS is implemented based on Alg.~2 in~\citep{hoffman2014no}.

We show the image denoising results and convergence behaviors of QHMC, HMC, NUTS and RMHMC in Fig.~\ref{fig:mnist}(a)--(e). We use the peak signal to noise ratio (PSNR) as the measure of denoising effects and report the results in Table~\ref{tab:psnr}. The standard HMC method is very sensitive to parameter choices, because it can only provide good results with the 2nd group of parameter setting. In contrast, our QHMC method shows good performance for all parameters choices. RMHMC achieves similar performance with HMC ($\mu_m=1$), only decreasing the training loss faster for larger iteration steps. However, RMHMC still fails to produce comparable denoising effects with QHMC. NUTS converges very slowly for this example because the ``U-Turn" criteria is easily satisfied in the spiky regions. Consequently, the recursions of balanced binary tree terminate quickly in NUTS, leading to low mixing rates. The difference between D-QHMC and S-QHMC in this example is nearly negligible as shown in Table \ref{tab:psnr}, and in Fig.~\ref{fig:mnist} we only showed the results for S-QHMC.

\begin{table}[t]
    \centering
    \caption{Image denoising example: PSNR for different algorithms}
 \begin{tabular}{|c|c|c|c|c|}
\hline
\multirow{2}{*}{Algorithm} & \multicolumn{3}{c|}{S-QHMC/D-QHMC with $m=10^{{\cal N}(\mu_m, 1)}$}             &  \multirow{1}{*}{NUTS} \\ \cline{2-4}
                           & $\mu_m=2$ & $\mu_m=1$ & $\mu_m=0$ &         ~\citep{hoffman2014no}    \\ \hline
PSNR                       & {\bf 23.72/23.80}      & {\bf 23.98/23.94}      & {\bf 23.96/24.06}     & 17.01                       \\ \hline
\multirow{2}{*}{Algorithm} & \multicolumn{3}{c|}{HMC with $m=10^{\mu_m}$} & \multirow{1}{*}{RMHMC} \\ \cline{2-4}
                            & $\mu_m=2$  & $\mu_m=1$   & $\mu_m=0$    &         ~\citep{girolami2011riemann}                     \\ \hline
PSNR                       & 14.94  & 23.61  & 17.03                   & 22.40                 \\ \hline

\end{tabular}
\normalsize
    \label{tab:psnr}
\end{table}


\subsection{Application: Bayesian Neural Network Pruning}

Finally we apply the stochastic gradient implementation of our QHMC method to the Bayesian pruning of neural networks. Neural networks~\citep{neal2012bayesian,schmidhuber2015deep} have achieved great success in wide engineering applications, but they are often over-parameterized. In order to reduce the memory and computational cost of deploying neural networks, pruning techniques~\citep{karnin1990simple} have been developed to remove redundant model parameters (e.g. weight parameters with tiny absolute values).

In this experiment we consider training the following two-layer neural network
\begin{equation}
    \mat{\hat{y}}=\mathrm{softmax}(\mathrm{relu}(\mat{W}_2\mathrm{relu}(\mat{W}_1\mat{x}+\mat{b}_1)+\mat{b}_2))
\end{equation}
to classify the MNIST dataset~\citep{lecun-mnisthandwrittendigit-2010}. Here ``relu" means a ReLU activation function;  $\mat{W}_1 \in \mathbb{R}^{784 \times 200}$ and $\mat{W}_2 \in \mathbb{R}^{200 \times 10}$ are the weight matrices for the 1st and 2nd fully connected layers; a softmax layer is used before the output layer. The log-likelihood of this neural network is a cross-entropy function. We introduce the $\ell_p (p=1/2)$ priors, ${\rm exp}(-\lambda|| {\rm Vec} (\mat{W}_1)||_p^p)$ and ${\rm exp}(-\lambda|| {\rm Vec} (\mat{W}_2)||_p^p)$ with $\lambda=0.0001$, for the weight matrices $\mat{W}_1$ and $\mat{W}_2$ to enforce their sparsity. The resulting potential energy function is thus 
\begin{equation}
    U(\mat{W_1},\mat{W_2},\mat{b_1},\mat{b_2})=\sum_{i=1}^N\sum_{j=1}^c -\mat{y}_{i,c}{\rm log}(\hat{\mat{y}}_{i,c})+\lambda||{\rm Vec(\mat{W_1})}||_p^p+\lambda(||{\rm Vec(\mat{W_2})}||_p^p
\end{equation}
where $c$ is the number of classes (10 for MNIST), $\mat{y}$ is the ground truth label vector, and ${\rm Vec}$ is an operator that vectorizes the matrix. For simplicity we employ uniform prior distributions for  $\mat{b_1}$ or $\mat{b_2}$. Due to the large training data set (with 60000 images), we use the stochastic gradient implementation with a batch size of 64 in the stochastic gradient Nos{\'e}-Hoover thermostat (SGNHT)~\citep{ding2014bayesian} and in our proposed quantum stochastic gradient Nos{\'e}-Hoover thermostat (QSGNHT). In~\citep{chen2014stochastic} it is shown that a naive implementation of stochastic gradient HMC can produce wrong steady distributions, therefore we utilize the thermostat method~\citep{ding2014bayesian} to correct the steady distributions, and adopt the QSGNHT algorithm proposed in Section \ref{sec:SQHMC}. Our proposed QSGNHT is slightly different from SGHMC~\citep{chen2014stochastic} in the practical implementation. Firstly, we implement the MH step where a Hamiltonian is estimated based on a batch of samples, although no MH steps are used in~\citep{chen2014stochastic} due to efficiency considerations. Secondly, we set a lower bound $m_0$ for $m$: if a sampled mass $m$ is smaller than $m_0$, then we set $m=m_0$. Both tricks help avoid very large model updates because $m$ can be arbitrarily close to 0 if there is no lower bound for $m$. 

\begin{table}[t]
\caption{Bayesian pruning results of a 2-layer neural network classifier for the MNIST dataset. Comp rate: compression rate; Test acc: test accuracy.}
\begin{center}
\begin{tabular}{|c|c|c|c|c|}
    \hline
    &  SGNHT  & SGHMC & SGLD &
    QSGNHT\\
    & \citep{ding2014bayesian} & \citep{chen2014stochastic} & \citep{welling2011bayesian} & (proposed) \\\hline
    Comp rate & $47.6\pm 1.3$  & $32.2\pm 1.5$ & $30.8\pm 1.4$ & {\bf 54.2} $\pm$ {\bf 2.1} \\\hline
    Test acc (\%) & $94.4\pm 0.2$  & $95.2\pm 0.2$ & $93.2\pm 0.2$ & {\bf 96.6} $\pm$ {\bf 0.2} \\\hline
\end{tabular}
\end{center}
\label{tab:mnist_nn_pruning}
\end{table}

\begin{figure}[H]
    \centering
    \includegraphics[width=1\linewidth]{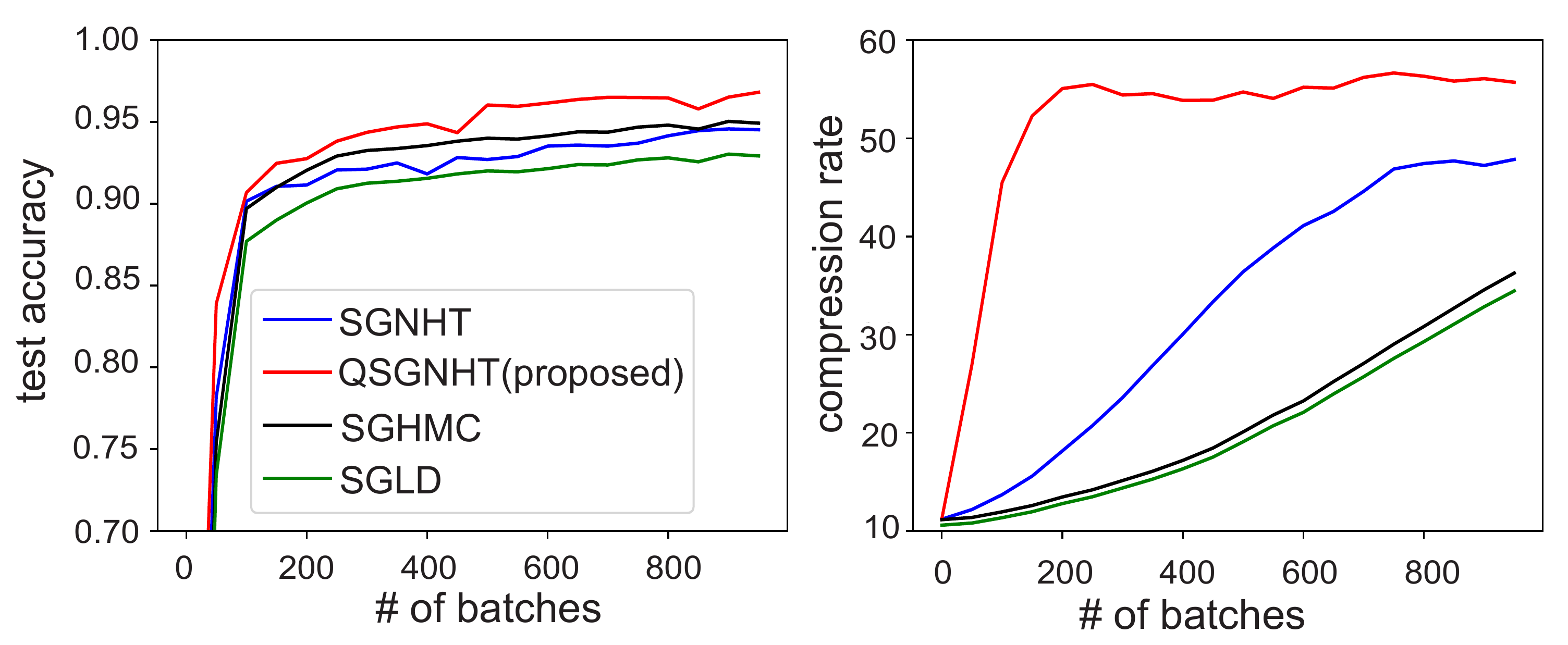}
    \caption{Bayesian pruning of a neural network with QSGNHT (proposed), SGNHT, SGHMC and SGLD. Left: testing accuracy; Right: compression rate.}
    \label{fig:nn-pruning}
    \vspace{-10pt}
\end{figure}

We choose $m=10^1$ for SGNHT and $m\sim10^{{\cal N}(1,0.5^2)}, m_0=1$ for QSGNHT, and run 1000 steps to collect random samples based on Hamiltonian simulations. We set a weight parameter to 0 if its absolute value is below $0.01$. The resulting compression rate is measured as:
\begin{equation}
    \mathrm{Compression\ rate}=\frac{\mathrm{\#\ of\ all \ weight\ parameters}}{\mathrm{\#\ of \ paramters\ after\ pruning}}.
\end{equation}

We plot the test accuracy estimated on 10000 test images and the compression rate as a function of the total number of training batches accessed in gradient evaluations, shown in Fig.~\ref{fig:nn-pruning}. Our method is compared with SGNHT~\citep{ding2014bayesian}, SGHMC~\citep{chen2014stochastic} and SGLD~\citep{welling2011bayesian}. SGHMC is implemented in the framework of SGNHT with the paramter setting $\xi=1$ and $m_\mu=\infty$. SGLD is implemented by excluding the momentum in SGHMC and directly updating the network with gradients. We consider the first 800 batches as the burn-in phase and use the simulation samples in last 200 batches. The test accuracy and compression rate for the neural network samples are reported in Table~\ref{tab:mnist_nn_pruning}. For test accuracy, our proposed QSGNHT achieves $2\%$ accuracy improvement compared with SGNHT~\citep{ding2014bayesian}. SGHMC and SGLD also underperform QSGNHT in terms of test accuracy. Regarding the compression rate, the third-order Langevin dynamics (QSGNHT and SGNHT) outperform the second-order Langevin dynamics (SGHMC) and first-order (SGLD) Langevin dynamics. Besides, the proposed QSGHNT achieves 14\% more compression rate than SGNHT.   

\section{Conclusions}
Leveraging the energy-time uncertainty relation in quantum mechanics, we have proposed a quantum-inspired Hamiltonian Monte Carlo method (QHMC) for Bayesian sampling. Different from the standard HMC, our method sets the particle mass matrix as a random variable associated with a probability distribution. We have proved the convergence of its steady state (both in space and in time sequence) to the true posterior density, and have theoretically justified its advantage over standard HMC in sampling from spiky distributions and multimodal distributions. In order to improve the efficiency of QHMC in massive-data scenarios, we have proposed a stochastic-gradient implementation with Nos{\'e}-Hoover thermostat terms and proved its theoretical properties. We have discussed S-QHMC, D-QHMC and M-QHMC as three special yet useful parametrizations of QHMC and have demonstrated their advantages over HMC and its variants by several synthetic examples. Finally, we have applied our methods to solve several classical Bayesian learning problems including sparse regression, image denoising and neural network pruning. Our methods have outperformed HMC, NUTS, RMHMC and several stochastic-gradient implementations on these realistic examples. In the future, we plan to develop a deeper theoretical understanding of QHMC and more robust and efficient implementation for large-scale learning problems.

\acks{This work was partly supported by NSF CCF-1817037 and NSF CAREER Award No. 1846476. The authors would like to thank Kaiqi Zhang, Cole Hawkins and Dr. Chunfeng Cui for their valuable technical suggestions.}

\bibliography{ref}

\begin{thebibliography}{46}
\providecommand{\natexlab}[1]{#1}
\providecommand{\url}[1]{\texttt{#1}}
\expandafter\ifx\csname urlstyle\endcsname\relax
  \providecommand{\doi}[1]{doi: #1}\else
  \providecommand{\doi}{doi: \begingroup \urlstyle{rm}\Url}\fi

\bibitem[Armagan(2009)]{armagan2009variational}
Artin Armagan.
\newblock Variational bridge regression.
\newblock In \emph{Artificial Intelligence and Statistics}, pages 17--24, 2009.

\bibitem[Bakry et~al.(2008)Bakry, Cattiaux, and Guillin]{bakry2008rate}
Dominique Bakry, Patrick Cattiaux, and Arnaud Guillin.
\newblock Rate of convergence for ergodic continuous markov processes: Lyapunov
  versus poincar{\'e}.
\newblock \emph{Journal of Functional Analysis}, 254\penalty0 (3):\penalty0
  727--759, 2008.

\bibitem[Beskos et~al.(2011)Beskos, Pinski, Sanz-Serna, and
  Stuart]{BESKOS20112201}
A.~Beskos, F.J. Pinski, J.M. Sanz-Serna, and A.M. Stuart.
\newblock Hybrid monte carlo on hilbert spaces.
\newblock \emph{Stochastic Processes and their Applications}, 121\penalty0
  (10):\penalty0 2201 -- 2230, 2011.
\newblock ISSN 0304-4149.
\newblock \doi{https://doi.org/10.1016/j.spa.2011.06.003}.
\newblock URL
  \url{http://www.sciencedirect.com/science/article/pii/S0304414911001396}.

\bibitem[Borovkov(1998)]{borovkov1998ergodicity}
Aleksandr~Alekseevich Borovkov.
\newblock \emph{Ergodicity and Stability of Stochastic Processes}.
\newblock J. Wiley, 1998.

\bibitem[Bou-Rabee and Sanz-Serna(2017)]{Bou_Rabee_2017}
Nawaf Bou-Rabee and Jesús~María Sanz-Serna.
\newblock Randomized hamiltonian monte carlo.
\newblock \emph{The Annals of Applied Probability}, 27\penalty0 (4):\penalty0
  2159–2194, Aug 2017.
\newblock ISSN 1050-5164.
\newblock \doi{10.1214/16-aap1255}.
\newblock URL \url{http://dx.doi.org/10.1214/16-AAP1255}.

\bibitem[Cand{\`e}s et~al.(2011)Cand{\`e}s, Li, Ma, and
  Wright]{candes2011robust}
Emmanuel~J Cand{\`e}s, Xiaodong Li, Yi~Ma, and John Wright.
\newblock Robust principal component analysis?
\newblock \emph{Journal of the ACM (JACM)}, 58\penalty0 (3):\penalty0 11, 2011.

\bibitem[Chaari et~al.(2016)Chaari, Tourneret, Chaux, and
  Batatia]{chaari2016hamiltonian}
Lotfi Chaari, Jean-Yves Tourneret, Caroline Chaux, and Hadj Batatia.
\newblock A {Hamiltonian} {Monte Carlo Method} for {Non-Smooth Energy
  Sampling}.
\newblock \emph{IEEE Transactions on Signal Processing}, 64\penalty0
  (21):\penalty0 5585--5594, 2016.

\bibitem[Chaari et~al.(2017)Chaari, Tourneret, and Batatia]{chaari17nshmc}
Lotfi Chaari, Jean-Yves Tourneret, and Hadj Batatia.
\newblock {A General Non-Smooth} {Hamiltonian Monte Carlo} {Scheme Using}
  {Bayesian Proximity Operator Calculation}.
\newblock In \emph{European Signal Processing Conference (EUSIPCO)}, pages
  1220--1224, 2017.

\bibitem[Chen et~al.(2014)Chen, Fox, and Guestrin]{chen2014stochastic}
Tianqi Chen, Emily Fox, and Carlos Guestrin.
\newblock Stochastic gradient {Hamiltonian} {Monte Carlo}.
\newblock In \emph{International conference on machine learning}, pages
  1683--1691, 2014.

\bibitem[Ding et~al.(2014)Ding, Fang, Babbush, Chen, Skeel, and
  Neven]{ding2014bayesian}
Nan Ding, Youhan Fang, Ryan Babbush, Changyou Chen, Robert~D Skeel, and Hartmut
  Neven.
\newblock Bayesian sampling using stochastic gradient thermostats.
\newblock In \emph{Advances in neural information processing systems}, pages
  3203--3211, 2014.

\bibitem[Donoho et~al.(2006)]{donoho2006compressed}
David~L Donoho et~al.
\newblock Compressed sensing.
\newblock \emph{IEEE Transactions on information theory}, 52\penalty0
  (4):\penalty0 1289--1306, 2006.

\bibitem[Duane et~al.(1987)Duane, Kennedy, Pendleton, and Roweth]{Duane:1987de}
S.~Duane, A.~D. Kennedy, B.~J. Pendleton, and D.~Roweth.
\newblock {Hybrid Monte Carlo}.
\newblock \emph{Phys. Lett.}, B195:\penalty0 216--222, 1987.
\newblock \doi{10.1016/0370-2693(87)91197-X}.

\bibitem[Eckart(1953)]{PhysRev.91.784}
Carl Eckart.
\newblock Relation between time averages and ensemble averages in the
  statistical dynamics of continuous media.
\newblock \emph{Phys. Rev.}, 91:\penalty0 784--790, Aug 1953.

\bibitem[Eldar and Kutyniok(2012)]{eldar2012compressed}
Yonina~C Eldar and Gitta Kutyniok.
\newblock \emph{Compressed Sensing: Theory and Applications}.
\newblock Cambridge university press, 2012.

\bibitem[Girolami and Calderhead(2011)]{girolami2011riemann}
Mark Girolami and Ben Calderhead.
\newblock Riemann manifold langevin and {Hamiltonian} {Monte Carlo} methods.
\newblock \emph{Journal of the Royal Statistical Society: Series B (Statistical
  Methodology)}, 73\penalty0 (2):\penalty0 123--214, 2011.

\bibitem[Graham and Storkey(2017)]{graham2017continuously}
Matthew~M Graham and Amos~J Storkey.
\newblock Continuously tempered {Hamiltonian Monte Carlo}.
\newblock \emph{arXiv preprint arXiv:1704.03338}, 2017.

\bibitem[Gray and Gray(1988)]{gray1988probability}
Robert~M Gray and RM~Gray.
\newblock \emph{Probability, Random Processes, and Ergodic Properties}.
\newblock Springer, 1988.

\bibitem[Hastings(1970)]{hastings1970monte}
W~Keith Hastings.
\newblock {Monte Carlo} sampling methods using {Markov} chains and their
  applications.
\newblock 1970.

\bibitem[Hoffman and Gelman(2014)]{hoffman2014no}
Matthew~D Hoffman and Andrew Gelman.
\newblock The {No-U-Turn} sampler: Adaptively setting path lengths in
  {Hamiltonian Monte Carlo}.
\newblock \emph{Journal of Machine Learning Research}, 15\penalty0
  (1):\penalty0 1593--1623, 2014.

\bibitem[Hou et~al.(2018)Hou, Xia, and Zhou]{hou2018structural}
Rongrong Hou, Yong Xia, and Xiaoqing Zhou.
\newblock Structural damage detection based on l1 regularization using natural
  frequencies and mode shapes.
\newblock \emph{Structural Control and Health Monitoring}, 25\penalty0
  (3):\penalty0 e2107, 2018.

\bibitem[Huang et~al.(2008)Huang, Horowitz, Ma, et~al.]{huang2008asymptotic}
Jian Huang, Joel~L Horowitz, Shuangge Ma, et~al.
\newblock Asymptotic properties of bridge estimators in sparse high-dimensional
  regression models.
\newblock \emph{The Annals of Statistics}, 36\penalty0 (2):\penalty0 587--613,
  2008.

\bibitem[Karnin(1990)]{karnin1990simple}
Ehud~D Karnin.
\newblock A simple procedure for pruning back-propagation trained neural
  networks.
\newblock \emph{IEEE Trans. Neural Networks}, 1\penalty0 (2):\penalty0
  239--242, 1990.

\bibitem[Lan et~al.(2014)Lan, Streets, and Shahbaba]{lan2014wormhole}
Shiwei Lan, Jeffrey Streets, and Babak Shahbaba.
\newblock Wormhole {Hamiltonian Monte Carlo}.
\newblock In \emph{Twenty-Eighth AAAI Conference on Artificial Intelligence},
  2014.

\bibitem[LeCun and Cortes(2010)]{lecun-mnisthandwrittendigit-2010}
Yann LeCun and Corinna Cortes.
\newblock {MNIST} handwritten digit database.
\newblock 2010.
\newblock URL \url{http://yann.lecun.com/exdb/mnist/}.

\bibitem[Leimkuhler and Reich(2009)]{Leimkuhler2009AMA}
Benedict~J. Leimkuhler and Sebastian Reich.
\newblock A metropolis adjusted {Nos{\'e}-Hoover} thermostat.
\newblock 2009.

\bibitem[Lu et~al.(2016)Lu, Perrone, Hasenclever, Teh, and
  Vollmer]{lu2016relativistic}
Xiaoyu Lu, Valerio Perrone, Leonard Hasenclever, Yee~Whye Teh, and Sebastian~J
  Vollmer.
\newblock Relativistic {Monte Carlo}.
\newblock \emph{arXiv preprint arXiv:1609.04388}, 2016.

\bibitem[Ma et~al.(2015)Ma, Chen, and Fox]{ma2015complete}
Yi-An Ma, Tianqi Chen, and Emily Fox.
\newblock A complete recipe for stochastic gradient {MCMC}.
\newblock In \emph{Advances in Neural Information Processing Systems}, pages
  2917--2925, 2015.

\bibitem[Mohasel~Afshar and Domke(2015)]{NIPS2015_5801}
Hadi Mohasel~Afshar and Justin Domke.
\newblock {Reflection, Refraction,} and {Hamiltonian Monte Carlo}.
\newblock In \emph{Advances in Neural Information Processing Systems}, pages
  3007--3015. 2015.

\bibitem[M{\"u}ller-Kirsten(2013)]{ma14ller2013basics}
Harald~JW M{\"u}ller-Kirsten.
\newblock \emph{Basics of Statistical Physics}.
\newblock World Scientific Publishing Company, second edition, 2013.

\bibitem[Myshkov and Julier(2016)]{myshkov2016posterior}
Pavel Myshkov and Simon Julier.
\newblock {Posterior Distribution Analysis} for {Bayesian Inference in Neural
  Networks}.
\newblock \emph{Advances in Neural Information Processing Systems (NIPS)},
  2016.

\bibitem[Neal(2012)]{neal2012bayesian}
Radford~M Neal.
\newblock \emph{Bayesian Learning for Neural Networks}, volume 118.
\newblock Springer Science \& Business Media, 2012.

\bibitem[Neal et~al.(2011)]{neal2011mcmc}
Radford~M Neal et~al.
\newblock {MCMC} using {Hamiltonian} dynamics.
\newblock \emph{Handbook of Markov Chain Monte Carlo}, 2\penalty0
  (11):\penalty0 2, 2011.

\bibitem[Nishimura et~al.(2017)Nishimura, Dunson, and
  Lu]{nishimura2017discontinuous}
Akihiko Nishimura, David Dunson, and Jianfeng Lu.
\newblock Discontinuous {Hamiltonian} {Monte Carlo} for {Discrete Parameters
  and Discontinuous Likelihoods}.
\newblock \emph{arXiv preprint arXiv:1705.08510}, 2017.

\bibitem[Oliveira and Werlang(2007)]{oliveira2007ergodic}
C{\'e}sar R~de Oliveira and Thiago Werlang.
\newblock Ergodic hypothesis in classical statistical mechanics.
\newblock \emph{Revista Brasileira de Ensino de F{\'\i}sica}, 29\penalty0
  (2):\penalty0 189--201, 2007.

\bibitem[Polson and Sun(2019)]{Polson2017BayesianL0}
Nicholas~G Polson and Lei Sun.
\newblock Bayesian $l_0$-regularized least squares.
\newblock \emph{Applied Stochastic Models in Business and Industry},
  35\penalty0 (3):\penalty0 717--731, 2019.

\bibitem[Polson et~al.(2014)Polson, Scott, and Windle]{polson2014bayesian}
Nicholas~G Polson, James~G Scott, and Jesse Windle.
\newblock The bayesian bridge.
\newblock \emph{Journal of the Royal Statistical Society: Series B (Statistical
  Methodology)}, 76\penalty0 (4):\penalty0 713--733, 2014.

\bibitem[Rosenthal(1995)]{rosenthal1995convergence}
Jeffrey~S Rosenthal.
\newblock Convergence rates for {Markov} chains.
\newblock \emph{SIAM Review}, 37\penalty0 (3):\penalty0 387--405, 1995.

\bibitem[Schmidhuber(2015)]{schmidhuber2015deep}
J{\"u}rgen Schmidhuber.
\newblock Deep learning in neural networks: An overview.
\newblock \emph{Neural networks}, 61:\penalty0 85--117, 2015.

\bibitem[Sobral et~al.(2015)Sobral, Javed, Ki~Jung, Bouwmans, and
  Zahzah]{sobral2015online}
Andrews Sobral, Sajid Javed, Soon Ki~Jung, Thierry Bouwmans, and El-hadi
  Zahzah.
\newblock Online stochastic tensor decomposition for background subtraction in
  multispectral video sequences.
\newblock In \emph{Proceedings of the IEEE International Conference on Computer
  Vision Workshops}, pages 106--113, 2015.

\bibitem[Springenberg et~al.(2016)Springenberg, Klein, Falkner, and
  Hutter]{NIPS2016_6117}
Jost~Tobias Springenberg, Aaron Klein, Stefan Falkner, and Frank Hutter.
\newblock {Bayesian Optimization} with {Robust Bayesian Neural Networks}.
\newblock In \emph{Advances in Neural Information Processing Systems}, pages
  4134--4142, 2016.

\bibitem[Tripuraneni et~al.(2017)Tripuraneni, Rowland, Ghahramani, and
  Turner]{tripuraneni2017magnetic}
Nilesh Tripuraneni, Mark Rowland, Zoubin Ghahramani, and Richard Turner.
\newblock Magnetic {Hamiltonian Monte Carlo}.
\newblock In \emph{Proceedings of the 34th International Conference on Machine
  Learning-Volume 70}, pages 3453--3461. JMLR. org, 2017.

\bibitem[Welling and Teh(2011)]{welling2011bayesian}
Max Welling and Yee~W Teh.
\newblock Bayesian learning via stochastic gradient langevin dynamics.
\newblock In \emph{Proceedings of the 28th international conference on machine
  learning (ICML-11)}, pages 681--688, 2011.

\bibitem[Xu et~al.(2010)Xu, Zhang, Wang, Chang, and Liang]{xu20101}
Zongben Xu, Hai Zhang, Yao Wang, XiangYu Chang, and Yong Liang.
\newblock ${L}_{1/2}$ regularization.
\newblock \emph{Science China Information Sciences}, 53\penalty0 (6):\penalty0
  1159--1169, 2010.

\bibitem[Ye and Zhu(2018)]{ye2018stochastic}
Nanyang Ye and Zhanxing Zhu.
\newblock Stochastic fractional hamiltonian monte carlo.
\newblock In \emph{Proceedings of the 27th International Joint Conference on
  Artificial Intelligence}, pages 3019--3025. AAAI Press, 2018.

\bibitem[Zhao et~al.(2014)Zhao, Yang, Cong, Wang, and Intes]{zhao2014p}
Lingling Zhao, He~Yang, Wenxiang Cong, Ge~Wang, and Xavier Intes.
\newblock ${L}_p$ regularization for early gate fluorescence molecular
  tomography.
\newblock \emph{Optics letters}, 39\penalty0 (14):\penalty0 4156--4159, 2014.

\bibitem[Zhao et~al.(2015)Zhao, Zhang, and Cichocki]{zhao2015bayesian}
Qibin Zhao, Liqing Zhang, and Andrzej Cichocki.
\newblock Bayesian {CP} factorization of incomplete tensors with automatic rank
  determination.
\newblock \emph{IEEE transactions on pattern analysis and machine
  intelligence}, 37\penalty0 (9):\penalty0 1751--1763, 2015.

\end{thebibliography}

\end{document}